\documentclass{article} 

\usepackage{xcolor}  
\usepackage[breaklinks=true,colorlinks]{hyperref}
\hypersetup{
  linkcolor={red!50!black},
  citecolor={blue!50!black},
}   
\usepackage{url}

\usepackage[utf8]{inputenc} 
\usepackage[truedimen,margin=21mm]{geometry} 

\usepackage[T1]{fontenc}    
\usepackage{url}            
\usepackage{booktabs}       
\usepackage{amsfonts}       
\usepackage{nicefrac}       
\usepackage{microtype}      
\usepackage[pdftex]{graphicx}
\usepackage{natbib} 
\usepackage{algorithm} 
\usepackage{algorithmic} 
\usepackage{multirow}
\usepackage{amsmath}
\usepackage{amssymb}
\usepackage{amsthm}
\usepackage{here}
\usepackage{wrapfig}
\usepackage{centernot}
\usepackage{enumitem}
\usepackage{comment}
\usepackage{subfig}
\usepackage{bbm}
\usepackage{titlesec}
\usepackage{abstract}

\allowdisplaybreaks[1]
\titleformat*{\section}{\large\bfseries}

\newtheorem{theorem}{Theorem}
\newtheorem{lemma}{Lemma}
\newtheorem{proposition}{Proposition}
\newtheorem{definition}{Definition}
\newtheorem{corollary}{Corollary}
\newtheorem{assumption}{Assumption}
\newtheorem{example}{Example}

\usepackage{bm}

\usepackage{color}

\newcommand{\R}{\mathbb{R}}
\newcommand{\E}{\mathbb{E}}
\newcommand{\bH}{\mathbb{H}}
\newcommand{\dd}{\mathrm{d}}
\newcommand{\KL}{\mathrm{KL}}
\newcommand{\FI}{\mathrm{FI}}
\newcommand{\Var}{\mathrm{Var}}

\newcommand{\Law}{\mathrm{Law}}

\newcommand{\ip}[2]{\left\langle #1,#2\right\rangle}

\newcommand{\cA}{\mathcal{A}}

\newcommand{\cL}{\mathcal{L}}

\newcommand{\cR}{\mathcal{R}}
\newcommand{\cN}{\mathcal{N}}

\newcommand{\cPI}[1]{C_{#1}^{\mathrm{PI}}}
\newcommand{\cLSI}[1]{C_{#1}^{\mathrm{LSI}}}

\newcommand{\two}{I\hspace{-1.2pt}I}

\title{\Large Slowly Annealed Langevin Dynamics: \\ Theory and Applications to Training-Free Guided Generation}
\author{}

\author{Atsushi Nitanda$^{1,2,\dag}$, Dake Bu$^{3,1}$, Yueming Lyu$^1$, Tanya Veeravalli$^1$
\vspace{2mm}\\
\normalsize{\textit{$^1$Agency for Science, Technology and Research (A$\star$STAR)}}, \\
\normalsize{\textit{$^2$Nanyang Technological University}}, \\
\normalsize{\textit{$^3$City University of Hong Kong}} \\
\small{Email: $^\dag$atsushi\_nitanda@cfar.a-star.edu.sg} 
}

\date{}

\begin{document}

\maketitle
\begin{abstract}
We study Slowly Annealed Langevin Dynamics (SALD), a sampler for tracking a path of moving target distributions and approximating the terminal target through time slowdown. We establish non-asymptotic convergence guarantees via a KL differential inequality, showing that slowdown improves tracking through contraction of intermediate targets and the complexity of the path. Motivated by training-free guided generation with pretrained score-based generative models, we further introduce Velocity-Aware SALD (VA-SALD), which explicitly incorporates the underlying marginal distributions of the pretrained model and uses slowdown to correct the additional deviation induced by guidance. This yields a principled framework for training-free guided generation for diffusion-based and related generative model families, together with convergence guarantees that clarify the roles of intermediate functional inequalities and guide-induced residual. Code is available at \href{https://github.com/anitan0925/sald}{\nolinkurl{https://github.com/anitan0925/sald}}.
\end{abstract}
\section{Introduction}
Sampling from complex high-dimensional distributions is a central problem in machine learning, statistics, and generative modeling. A successful strategy is to construct a path of intermediate distributions that gradually deforms a simple reference law into a target distribution, and to follow this path using Langevin-type dynamics. This idea underlies annealed Langevin methods and score-based generative modeling, where a time-dependent score function guides the sampler through a sequence of increasingly structured distributions \cite{song2019generative,song2021score,lee2022convergence,xunposterior,chehab2025provable,cordero2025non,cattiaux2025diffusion}. In particular, diffusion-based generative models can be viewed as defining a family of time-varying marginals whose terminal distribution coincides with the data distribution \cite{ho2020denoising,song2021score}. This viewpoint naturally leads to the problem of sampling while tracking a \emph{moving target path}.

A recent line of work has highlighted the benefit of \emph{slowdown} in annealed Langevin dynamics. In particular, \cite{GuoTaoChen2024} studied Slowly Annealed Langevin Dynamics (SALD), where the target path is traversed more slowly in algorithmic time, and established the convergence guarantee via a path-space comparison argument based on Girsanov's theorem. Their analysis reveals a complexity term governed by the path geometry and explains why slow schedules are beneficial for non-log-concave sampling. However, path-space arguments are less directly suited to capturing \emph{marginal contraction} toward moving targets from mismatched initialization. In many applications, one would like to understand how slowdown interacts with functional inequalities (e.g., log-Sobolev inequality) of intermediate targets, how it reduces tracking error at the distribution level, and how it stabilizes dynamics even when the initial law does not match the starting point of the target path.

At the same time, \emph{guided generation} has emerged as a major use case for time-dependent generative sampling. Given a pretrained score-based generative model, one would like to steer sampling toward outputs preferred by a reward, constraint, or guidance function \cite{ho2022classifier,ye2024tfg,jiao2025towards,liu2026alignment,chen2024overview}. A theoretically principled approach is provided by Doob's \(h\)-transform, which yields the correct time-dependent modification of the reverse process targeting a tilted terminal distribution \cite{kawatadirect,tang2024stochastic,chang2026inference,zhu2026training, kim2025testtime}. However, this correction generally requires learning or approximating a time-dependent guiding function, or performing costly inference-time Monte Carlo estimation. More broadly, directly tracking the full guided target path may be substantially more complex than following the pretrained marginal path itself. 

In this work, we revisit SALD from a marginal-distribution viewpoint and extend it to training-free guided generation. First, we develop a non-asymptotic convergence analysis of SALD based on a reverse KL differential inequality. Our analysis shows that slowdown improves tracking through two complementary mechanisms: contraction induced by functional inequalities of intermediate targets, and a path-complexity term captured by an \emph{action/energy} term. This yields a transparent explanation of how slowdown can reduce both tracking error and initialization mismatch. Second, motivated by pretrained score-based generative models, we introduce \emph{Velocity-Aware SALD} (VA-SALD). The key idea is to explicitly incorporate the underlying transport velocity of the pretrained marginal path and to use slowdown primarily to compensate for the additional deviation induced by the guide. As a result, VA-SALD avoids paying for the full complexity of the guided path and instead focuses on the guide-induced correction.

Our contributions are summarized as follows:
\begin{itemize}[itemsep=0mm,leftmargin=5mm,topsep=0mm] 
    \item We establish a reverse-KL convergence analysis of SALD for general moving target distributions. The resulting bounds reveal how slowdown combines contraction of intermediate targets with the path complexity of the moving target path (Theorem~\ref{thm:forward-KL}). We further extend the analysis to discrete-time SALD and obtain an iteration complexity of \(O(\varepsilon^{-6})\) for achieving an \(\varepsilon^2\)-accurate solution (Theorem~\ref{thm:forward-KL-discrete}).

    \item We propose Velocity-Aware SALD (VA-SALD), a training-free and easy-to-implement guided sampling method that explicitly exploits the pretrained marginal evolution of score-based generative models. We prove convergence guarantees for VA-SALD in both continuous- and discrete-time settings, and show that its complexity is governed by the guide-induced correction rather than the full guided path (Theorems~\ref{thm:general-moving-target-SALD}). This can lead to substantially improved convergence compared with a direct application of SALD to guided targets.

    \item We empirically verify the benefits of SALD and VA-SALD on synthetic guided sampling tasks and guided image generation. Across these settings, we observe stable performance improvements consistent with our theoretical predictions (Section \ref{sec:experiments} and Section \ref{app:experiment_details} in Appendix).
\end{itemize}

\section{Problem Setup and Methods}\label{sec:problem_setup}

\subsection{Slowly Annealed Langevin Dynamics \cite{GuoTaoChen2024}}
We give an overview of annealed Langevin dynamics with slowdown \cite{GuoTaoChen2024}. Let $(\pi_t)_{t\in[0,T]}$ be an evolution of probability measures on $\R^d$ with finite second moments.
Our objective is to approximately sample from the terminal target distribution $\pi_T$, assuming oracle access to the time-dependent score function $\nabla \log \pi_t(x)$.
A straightforward approach is to consider an annealed Langevin dynamics governed by this score: $\dd X_t = \nabla \log \pi_t(X_t)\dd t + \sqrt{2}\dd W_t$, where $W_t$ is the standard Brownian motion in $\R^d$. To enhance tracking performance for rapidly evolving target distributions, we follow \cite{GuoTaoChen2024} and slow down the annealing schedule; with a slowdown parameter $r\ge 1$, we reparametrize the time scale as $t=t(s)~(s \in [0,S])$ where $t(s)$ is a smooth monotonic increasing sequence. An example is $t(s) = s/r~(s\in[0,rT])$ where $r > 1$ is a hyperparameter that defines the time-scale.

The corresponding target at time $s$ is $\tilde{\pi}_s = \pi_{t(s)}$. Then, we arrive at a \textit{Slowly Annealed Langevin Dynamics} (SALD) \cite{GuoTaoChen2024} that uses the slowly changing score $\nabla \log \tilde{\pi}_{s}$ in $s$.

\paragraph{Continuous-time SALD.}
Define the SALD diffusion $(X_s)_{s\in[0,rT]}$ by the SDE
\begin{equation}\label{eq:SALD}
    \dd X_s = \nabla \log \tilde{\pi}_{s}(X_s)\,\dd s + \sqrt{2}\,\dd W_s,~~~~X_0\sim\rho_0,
\end{equation}
where $X_0\sim\rho_0$ is an initial distribution. Note that we do not necessarily assume $\rho_0=\pi_0$. The probability law $\rho_s := \textrm{Law}(X_s)$ of SALD satisfies the following Fokker--Planck equation:
\begin{equation}\label{eq:FP-eq}
   \partial_s\rho_s = \nabla\cdot\left(\rho_s \nabla \log \frac{\rho_s}{\tilde{\pi}_s}\right).
\end{equation}

\paragraph{Discrete-time implementation.}
We also consider a simple Euler--Maruyama discretization with stepsize $\eta$:
\begin{equation}\label{eq:sald_em_terminal_disc_prop_additive_final}
    X_{k+1}^\eta
    =
    X_k^\eta
    +\eta\,\nabla\log\pi_{t_k}(X_k^\eta)
    +\sqrt{2\eta} \xi_k,
    \qquad
    \xi_k\stackrel{\mathrm{i.i.d.}}{\sim}\mathcal N(0,I_d),
    \qquad
    X_0^\eta\sim\rho_0,
\end{equation}
where $t_k = t(k\eta)$ is a discrete-time index\footnote{In \cite{GuoTaoChen2024}, this discrete-time SALD is referred to as ALMC.}.

\paragraph{Relation to SMLD.} 
Score-based sampling methods such as SMLD \cite{song2019generative} use a two-loop structure: an outer loop advances the noise level (or intermediate target), while an inner loop performs multiple Langevin updates to better equilibrate at the current target. SALD shares a related intuition, namely to exploit the closeness of neighboring intermediate targets. The key difference is that SALD uses a single dynamics with a slowly varying score, rather than alternating between frozen targets and repeated inner-loop corrections. In addition, while SMLD is typically formulated in the context of diffusion-based generative sampling, our analysis of SALD is developed for general moving targets $(\pi_t)_{t\in[0,T]}$. 

SALD can be used to sample from a target Gibbs distribution proportional to \(\exp(-V)\). For example, \cite{GuoTaoChen2024} considers a Gaussian annealing path: $\pi_t\propto\exp\bigl(-\omega(t)V-\frac{\lambda(t)}{2}\|\cdot\|^2\bigr)$ where \(\omega(\cdot)\) and \(\lambda(\cdot)\) are differentiable monotone schedules satisfying \(\omega_0 = \omega(0) \nearrow \omega(1) = 1\) and \(\lambda_0 = \lambda(0) \searrow \lambda(1) = 0,~(\omega_0 \in [0,1],~\lambda_0 \in [1, \infty))\). For this construction, \cite{GuoTaoChen2024} established an iteration complexity of \(O(\varepsilon^{-6})\) for obtaining an \(\varepsilon\)-accurate solution in forward KL divergence, under the matched-initialization assumption \(\rho_0=\pi_0\), using a Girsanov-based path-space argument. In contrast, our analysis provides a comparable complexity guarantee for the reverse KL divergence while allowing mismatched initialization. We also verify the assumptions for reverse VP diffusion paths initialized from the target Gibbs distribution under smooth dissipativity conditions; see Appendix~\ref{app:vp-dissipative-verification}.

\subsection{Velocity-Aware SALD for Guided Generation}\label{subsec:guided-generation}
In this section, we extend SALD \eqref{eq:SALD} to enable efficient guided generation with pretrained score-based generative models following It\^{o} diffusion processes such as diffusion models, flow matching models, and Schr\"{o}dinger bridges. Given a score-based generative model trained on the data distribution $p_\mathrm{data}$, we consider sampling from a tilted distribution proportional to $p_\mathrm{data}\exp(-f)$, where $f: \R^d \to \R$ is a smooth guide function. Our objective is to achieve this by using only the score of pre-trained generative model and the gradient of $f$. Let $(p_t)_{t \in [0,T]}$ denote the marginal distributions of the reverse process of the generative model, with $p_T = p_\mathrm{data}$. While the score-based models provide only an approximate estimate of the score $\nabla \log p_t$, we assume oracle access to it for simplicity. Under this assumption, we propose an extension of SALD that exploits the underlying velocity field of $(p_t)_{t \in [0,T]}$ for sampling from the target distribution $\pi_T \propto p_\mathrm{data}\exp(-f)$.

We first provide the notion of \textit{transport velocity field}. Consider the deterministic particle dynamics $\dd X_t = u_t(X_t)\dd t$ induced by a time-dependent vector field $u_t: \R^d \to \R^d~(t\in [0,T])$. The corresponding probability law $\mu_t := \Law(X_t)$ evolves according to the continuity equation $\partial_t \mu_t = - \nabla \cdot(\mu_t u_t)$ in the weak sense. Thus, $u_t$ describes how mass is transported along the path $(\mu_t)_{t \in [0,T]}$. We therefore call $u_t$ a transport velocity field of $(\mu_t)_{t \in [0,T]}$, or simply say that $u_t$ generates  $(\mu_t)_{t \in [0,T]}$. 

\paragraph{It\^o diffusion.}  We now describe our setting. Let \((q_\tau)_{\tau\in[0,T]}\) be the marginal laws of the forward It\^o diffusion with $q_0 = p_\mathrm{data}$:
\begin{equation}\label{eq:ito_forward}
\dd Y_\tau = \bar{B}_\tau(Y_\tau)\,\dd\tau + \bar{\sigma}_\tau \dd W_\tau,
\qquad \tau\in[0,T],\quad Y_0\sim q_0=p_\mathrm{data},
\end{equation}
where \(\bar{\sigma}_\tau>0\) is a real-valued scalar function of time and \(\bar{B}_\tau:\R^d\to\R^d\) is a drift term (vector field).
Assume that each \(q_\tau\) admits a smooth positive density. Define the reverse marginal family by $p_t:=q_{T-t},B_{t}:=\bar{B}_{T-t},\ 
\ 
\sigma_t:=\bar{\sigma}_{T-t},
\  t\in[0,T]$. Then the velocity fields that generate forward and reverse marginal distributions are
\begin{equation}\label{eq:continuity_Ito}
    \begin{aligned}
    &\text{(forward)} \quad \bar{u}_\tau(y)
    =
    \bar{B}_\tau(y)-\frac{\bar{\sigma}_\tau^2}{2}\nabla\log q_\tau(y),  \\
    &\text{(reverse)} \quad u_t(x)
    =
    -\bar{u}_{T-t}(x)
    =
    - B_{t}(x)+\frac{\sigma_t^2}{2}\nabla\log p_t(x). 
\end{aligned}
\end{equation}
That is,  $\partial_\tau q_\tau + \nabla\cdot (\bar{u}_\tau q_\tau)=0,~\partial_t p_t + \nabla\cdot (u_t p_t)=0$ hold.

Let $f_t$ be a (possibly time-varying) guide function satisfying $f_T=f$, and $\pi_t (x)\propto p_t(x)\exp\!\left(- f_t(x)\right)$ is the guided moving target. Our \textit{Velocity-Aware SALD (VA-SALD)} is then defined as: $X_0\sim \rho_0$
\begin{equation}\label{eq:SALD_Ito}
\dd X_s
=
\left(
\dot t(s)c_{t(s)}(X_s)
+
\frac{\sigma_{t(s)}^2}{2}\nabla\log \pi_{t(s)}(X_s)
\right)\dd s
+
\sigma_{t(s)} \dd W_s. 
\end{equation}
Here $c_t$ is a reference vector field that captures a tractable part of the guided evolution; it need not itself generate a probability path $\pi_t$. We consider two choices. Option I uses the transport velocity $c_t^I := u_t$ of the pretrained marginal path. Option II additionally incorporates the explicit guide-dependent field: $c_t^{\two}:=u_t -\frac{\sigma_t^2}{2}\nabla f_t$. Both choices use only the pretrained velocity, the pretrained score, and the gradient of the prescribed guide. 

Equivalently, by Eq.~\eqref{eq:continuity_Ito}, VA-SALD with Option I and Option II can be described as follows: 
\begin{align*}
&\dd X_s
=
\left(
-\dot t(s)B_{t(s)}(X_s)
+
\left(\dot t(s)+1\right)\frac{\sigma_{t(s)}^2}{2} 
\nabla\log p_{t(s)}(X_s)
-
\frac{\sigma_{t(s)}^2}{2} \nabla f_{t(s)}(X_s)
\right)\dd s
+
\sigma_{t(s)} \dd W_s,\\
&\dd X_s
=
\left(
-\dot t(s)B_{t(s)}(X_s)
+
\left(\dot t(s)+1\right)\frac{\sigma_{t(s)}^2}{2} 
\nabla\log p_{t(s)}(X_s)
-
\left(\dot t(s)+1\right)\frac{\sigma_{t(s)}^2}{2} \nabla f_{t(s)}(X_s)
\right)\dd s
+
\sigma_{t(s)} \dd W_s.
\end{align*}
Compared to Eq.~\eqref{eq:SALD}, VA-SALD requires the knowledge of velocity, namely $u_{t(s)}(X_s)$ and the diffusion scale $\sigma_{t(s)}$ from the forward process of It\^o diffusion. 

\begin{example}
\label{em:unified-subcases-explicit}
It\^o diffusion contains the Variance Preserving diffusion, flow-matching and Schr\"odinger bridge as special cases:
\begin{itemize}[itemsep=0mm,leftmargin=5mm,topsep=0mm] 
    \item \textbf{\emph{Variance Preserving (VP) diffusion}} \citep{ho2020denoising}. Set $\bar{B}_\tau(y)=-\frac{1}{2}\beta(\tau)y,
    \ 
    \bar{\sigma}_\tau=\sqrt{\beta(\tau)}$, where $\beta(\tau)$ is a prescribed schedule.
    Then the reverse velocity is 
    \begin{align*}
        & u_t(x)=\frac12\beta(T-t)\bigl(x+\nabla\log p_t(x)\bigr).
    \end{align*}
    \item \textbf{\emph{Flow matching}} \citep{lipman2022flow}.
    Set $\bar{B}_\tau(y)
    =\beta(\tau)\left(v_\tau(y)+\frac{\sigma(\tau)^2}{2}\nabla\log q_\tau(y)\right),~\bar{\sigma}_\tau=\sqrt{\beta(\tau)}\,\sigma(\tau),$ where $\beta(\tau)$ is a prescribed schedule and \(v_\tau:\mathbb{R}^d\to\mathbb{R}^d\) is a prescribed flow-matching velocity field generating the target marginal path (e.g., a compatible \(v\) in \citep{lipman2022flow}, \(v^{\mathrm{RF}}\) in rectified flow \citep{liu2022flow}, or \(v^{\mathrm{MF}}\) in MeanFlow \citep{geng2025mean}). Then the reverse velocity is 
    \begin{align*}
        & u_t(x)=-\beta(T-t)v_{T-t}(x).
    \end{align*}
\end{itemize}

For a guide function class $f_t$ with $f_T=f$, the Velocity-Aware SALD of the three cases could be directly obtained by substituting their own $u_t$ and $\sigma_t$ into Eq.~\eqref{eq:SALD_Ito}.
\end{example}

Since $u_t$ generates the marginal distributions $(p_t)_{t \in [0,T]}$ through the continuity equation $\partial_t p_t = - \nabla \cdot (p_t u_t)$, slowing down $t=t(s)$ yields
\[
    \partial_s p_{t(s)}
    = \dot{t}(s) \partial_t p_t|_{t=t(s)} 
    = - \dot{t}(s) \nabla\cdot(p_{t(s)}u_{t(s)})
    = - \nabla\cdot (p_{t(s)} \dot{t}(s)u_{t(s)}).
\]
Thus, $\dot{t}(s)u_{t(s)}$ is precisely the transport velocity associated with the slowed marginal distributions $(p_{t(s)})_{s \in [0,S]}$. In the unguided case $f_t\equiv 0$, the Fokker--Planck equation of VA-SALD \eqref{eq:SALD_Ito} admits the above slowed marginal path as a solution. Indeed, at $\rho_s=p_{t(s)}$, the score drift and diffusion terms cancel, and the equation reduces to $\partial_s p_{t(s)}= - \nabla\cdot\bigl(p_{t(s)}\dot t(s)u_{t(s)}\bigr)$. Thus, when initialized from $p_0$, the unguided VA-SALD follows $p_{t(s)}$ regardless of the choice of time-rescaling. In particular, for $t(S)=T$, the terminal distribution remains $p_T = p_{\mathrm{data}}$. This shows that VA-SALD with the slowdown itself does not introduce bias in the unguided dynamics. 
Under guidance, the discrepancy from the guided path $\pi_t\propto p_t\exp(-f_t)$ is determined by the part of its transport that is not captured by the selected reference field $c_t$. Although the original SALD (1) can in principle be applied directly to the guided path, it must track the entire moving target $(\pi_t)_{t\in[0,T]}$. In contrast, the general VA-SALD (6) incorporates a structured reference field and uses slowdown primarily to compensate for the remaining guide-induced residual. Under Option I, this residual is the additional transport not captured by the pretrained marginal velocity $u_t$. Under Option II, the reference field additionally incorporates $-\frac{\sigma_t^2}{2}\nabla f_t$, and the resulting residual is characterized by the Doob-consistency defect of the prescribed guide. These two residual complexities are analyzed in Section \ref{sec:convergence_velocityAwareSALD}

\paragraph{Remark.} Let $h_t(x):=\E \left[e^{-f(X_T)} \mid X_t=x \right]$ and $F_t^h:= -\log h_t$, where the conditional expectation is taken under the pretrained reverse process. If this time-dependent $h$-function is available, then the correction $-\nabla F_t^h = \nabla\log h_t$ yields the exact Doob-transformed reverse process targeting the tilted terminal distribution proportional to $p_Te^{-f}$. Constructing or approximating $h_t$, however, typically requires either additional training of a time-dependent guide model \cite{kawatadirect,tang2024stochastic,chang2026inference} or inference-time Monte Carlo estimation \cite{zhu2026training}.
Option II does not require access to $h_t$. It can be run with any prescribed guide $f_t$ using only $\nabla f_t$. Nevertheless, it is consistent with the exact Doob transform in the following sense: if $f_t=F_t^h$, then, without slowdown, its drift coincides with the exact Doob-transformed drift, and the residual complexity introduced in Section 4 vanishes. More generally, the residual of Option II quantifies the dynamic inconsistency of the prescribed guide with the exact Doob potential.


\section{Convergence Analysis of SALD}\label{sec:convergence_sald}
In this section, we provide further theoretical justification for SALD through convergence analysis.
We introduce several key concepts for convergence analysis. 
The first is the log-Sobolev inequality (LSI), which is imposed on the evolution of the moving target. 

\begin{definition}[Log-Sobolev Inequality (LSI)]
    We say a probability distribution $\pi$ satisfies LSI with constant $\cLSI~>0$ if for all smooth $\varphi$ with $\int \varphi^2\,\dd\pi=1$,
    \[
        \int \varphi^2 \log \varphi^2 \dd\pi \le \frac{2}{\cLSI~} \int \|\nabla \varphi\|^2 \dd\pi.
    \]
\end{definition}
In particular, for all densities $\rho\ll \pi$, LSI with $\varphi=\sqrt{\rho/\pi}$ implies the following inequality.
\begin{equation}\label{eq:LSI-KL-FI}
\KL(\rho\|\pi) \le \frac{1}{2\cLSI~} \FI(\rho\|\pi),
\end{equation}
where $\KL$ and $\FI$ are the Kullback-Leibler divergence and Fisher information defined below. For probability distributions $\rho,\pi$ with $\rho\ll\pi$, define
\[
    \KL(\rho\|\pi) := \int \log \frac{\rho}{\pi}(x) \rho(\dd x),~~~~
    \FI(\rho\|\pi) := \int \left\|\nabla \log \frac{\rho}{\pi}(x)\right\|^2 \rho(\dd x).
\]

Next, we introduce a quantity that measures the complexity of vector fields along \(\pi=(\pi_t)_{t\in[0,T]}\). Let \(v=(v_t)_{t\in[0,T]}\) be any family of measurable vector fields \(v_t:\mathbb{R}^d\to\mathbb{R}^d\). For \(\alpha>0\), define
\[
\mathfrak{E}_\alpha(\pi_t,v_t)
:=
\frac{1}{\alpha}\log \E_{\pi_t}\!\left[\exp\!\big(\alpha \|v_t\|^2\big)\right],
\qquad
\mathcal A_\alpha(\pi,v)
:=
\int_0^T \mathfrak{E}_\alpha(\pi_t,v_t)\,dt.
\]
We refer to \(\mathcal A_\alpha(\pi,v)\) as the \textit{\(\alpha\)-complexity} of the field \(v\) along \(\pi\).

In the special case where \(v_t\) is a transport velocity field generating \(\pi = (\pi_t)_{t\in[0,T]}\), that is, $\partial_t \pi_t + \nabla\cdot(\pi_t v_t)=0$,
we call \(\mathcal A_\alpha(\pi,v)\) the \textit{\(\alpha\)-energy} (or \textit{\(\alpha\)-action}) of the path $\pi$. In particular, we will take \(v_t\) to be the minimizer of \(\|v_t\|_{L^2(\pi_t)}\) among all transport velocity fields generating \((\pi_t)_{t\in[0,T]}\). Such a minimizer exists provided that the metric derivative $\lim_{h\to 0}\frac{W_2(\pi_{t+h},\pi_t)}{|h|}$
exists for all \(t\in[0,T]\); see \cite[Theorems 8.3.1 and 8.4.5]{ambrosio2005gradient}.

$\cA_\alpha(\pi,v)$ is a non-decreasing function in $\alpha$ and reproduces the action $\cA (\pi,v)= \int_0^T \|v_t\|_{L_2(\pi_t)}^2 \dd t$ \cite{GuoTaoChen2024} as $\cA_0 (\pi,v):= \lim_{\alpha \rightarrow 0^+}\cA_\alpha(\pi,v)$. Analogous to the action studied in \cite{GuoTaoChen2024}, the $\alpha$-complexity serves as a measure of the complexity of the path $(\pi_t)_{t \in [0,T]}$ and vector fields $v=(v_t)_{t \in [0,T]}$, which characterizes the convergence rate of SALD and VA-SALD under slowdown. 

\paragraph{Continuous-time SALD.} We present a general result for SALD under a general time-rescaling scheme. In the following, we extend the definition of the LSI constant to allow the value $0$; that is, we set $\cLSI{t}=0$ when the distribution does not satisfy LSI with any positive constant. Let $s=s(t)$ denote the inverse function of $t=t(s)$.

\begin{theorem}\label{thm:forward-KL}
    Assume that $\pi_t$ satisfies LSI with constant $\cLSI{t} \geq 0$ for all $t \in [0,T]$, and that there exists a constant $\alpha_0 > 0$ such that $\mathfrak{E}_{\alpha_0}(\pi_t,v_t) < +\infty$ for all $t \in [0,T]$, where $v$ is the transport velocity field of $\pi$. Then, for any $\alpha \in (0,\alpha_0]$, the probability law $(\rho_s)_{s \in [0,S]}$ of \textbf{SALD} \eqref{eq:SALD} satisfies
    \begin{align*}
        \KL(\rho_{S} \| \pi_T)
        &\leq \exp\left( - \int_0^{T} \dot{s}(t)\cLSI{t} \dd t\right)
        \exp\left( \int_0^{T} \frac{1}{2}\dot{s}(t)^{-1}\alpha^{-1}\dd t\right)\KL(\rho_0 \| \pi_0) \\
        &+ \int_0^{T} \exp\left(\int_t^{T} \frac{1}{2}\dot{s}(u)^{-1}\alpha^{-1} \dd u\right)  \frac{1}{2}\dot{s}(t)^{-1}\mathfrak{E}_{\alpha}(\pi_t,v_t) \dd t.
    \end{align*}
\end{theorem}

We next specialize the above results to the case of linear slowdown, where $t(s)=s/r$ with $r\geq1$.

\begin{corollary}[Continuous-time SALD with linear slowdown]\label{corforward-KL}  
    Under the same assumptions as in Theorem~\ref{thm:forward-KL}, for any \(\alpha \in (0,\alpha_0]\),
    \begin{align*}
        \KL(\rho_{S} \| \pi_T)
        \leq
        \exp\left( - r \int_0^{T}\cLSI{t} \dd t\right)
        \exp\left( \frac{T}{2r\alpha}\right)\KL(\rho_0 \| \pi_0) 
        + \frac{1}{2r} \exp\left(\frac{T}{2r\alpha} \right) \cA_{\alpha}(\pi,v).
    \end{align*}
\end{corollary}

Importantly, a positive LSI constant is not required throughout the entire path. Exponential contraction of the initialization error only requires $\int_0^T \cLSI{t} \dd t > 0$ yielding an exponential contraction factor of$\exp\left(-r \int_0^T \cLSI{t} \dd t \right)$ as $r\to \infty$. Since the early-stage moving targets $\pi_t$ are typically designed to be easy to sample from and to satisfy an LSI, SALD naturally exploits these targets to reduce the initialization error. The second term, governed by $\alpha$-complexity, also vanishes as $r \to \infty$. Consequently, Theorem \ref{thm:forward-KL} guarantees the quantitative convergence rate of SALD under a linear slowdown schedule. Furthermore, since $\cA_\alpha(\pi,v) \searrow \cA_0(\pi,v)$ as $\alpha \searrow 0$, we obtain the following convergence rate regarding $r\to\infty$; for sufficiently small $\alpha > 0$, 
\[ 
    \KL(\rho_{S} \| \pi_T) 
    = O\left( \exp\left( - r \int_0^{T}\cLSI{t} \dd t\right)\KL(\rho_0 \| \pi_0) 
    + \frac{\cA_0(\pi,v)}{r} \right).
\]

\paragraph{Comparison with Girsanov-based path-space analysis.}
Guo et al.~\cite{GuoTaoChen2024} analyze SALD through a Girsanov-based comparison of path measures, followed by data processing to obtain a terminal marginal bound. While this approach captures the cost of moving along the target path, it does not directly explain contraction toward the moving targets. In contrast, our analysis differentiates the marginal KL divergence \(\KL(\rho_s\|\tilde\pi_s)\) itself. The resulting differential inequality converts the LSIs of the intermediate targets into contraction, while controlling the remaining error by the $\alpha$-complexity. Consequently, our reverse KL bound can reduce initialization mismatch \(\rho_0\neq \pi_0\) through slowdown, a mechanism that is not visible from a purely path-space comparison.

\paragraph{Discrete-time SALD.} We also present a similar result for discrete-time SALD \eqref{eq:sald_em_terminal_disc_prop_additive_final}. Define $s_k:=k\eta,~t_k:=t(s_k),~\rho_k^\eta:=\Law(X_k^\eta)$, where $k=0,\dots,K,\ K=\frac{S}{\eta}$. Akin to \cite{GuoTaoChen2024}, we consider Lipschitz assumptions with constants $L_{\pi,\mathrm{space}},L_{\pi,\mathrm{time}}>0$ and a measurable function  $M(\cdot)$ such that for all $x,y\in\mathbb R^d$, $t\in[t_k,t_{k+1}]$ 
\begin{equation}\label{eq:lip_SALD_1}
   \|\nabla\log\pi_{t_k}(x)-\nabla\log\pi_{t_k}(y)\|
   \le L_{\pi,\mathrm{space}}\|x-y\|,
\end{equation}
and for all $s < s'~(s,s' \in [0,S]),~x \in \R^d$,
\begin{equation}\label{eq:lip_SALD_2}
    \|\nabla\log\pi_{t(s')}(x)-\nabla\log\pi_{t(s)}(x)\|
    \le L_{\pi,\mathrm{time}}(s'-s)(1+M(x)).
\end{equation}

Furthermore, we define $\bar{\Gamma}:= \int_0^T \Gamma(t)\dd t,~\bar{\Delta}_{\alpha'}:=\int_0^T\Delta(t)\dd t$, where 
\begin{align*}
    &\Gamma(t) 
    := 2L_{\pi,\mathrm{time}}^2
    + 64 L_{\pi,\mathrm{space}}^2
    \Bigl(
    \dot{s}(t)^{-2} + 1 + \eta^2L_{\pi,\mathrm{time}}^2  \Bigr),\\
    &\Delta(t) := 
    64\eta L_{\pi,\mathrm{space}}^2 
    \mathfrak{E}_{\alpha'}(\pi_{t},\nabla\log\pi_{t})
    + 2\eta L_{\pi,\mathrm{time}}^2 (1+32\eta^2 L_{\pi,\mathrm{space}}^2)\mathfrak{E}_{\alpha'}(\pi_{t},1+M)
    + 16 d L_{\pi,\mathrm{space}}^2.
\end{align*}

$\bar{\Gamma}$ and $\bar{\Delta}_{\alpha'}$ arise from the time-discretization. They are controlled by smoothness and complexities and remain $O(1)$ regarding $\eta \to 0$ and $\alpha' \to 0$.
We consider the linear slowdown $t(s)=s/r~(r\geq 1)$.

\begin{theorem}\label{thm:forward-KL-discrete}
Assume the same assumptions of Theorem~\ref{thm:forward-KL} and Lipschitz continuity conditions in Eq.~\eqref{eq:lip_SALD_1},~\eqref{eq:lip_SALD_2} . Assume also there exists $\alpha_0'>0$,
$\mathfrak E_{\alpha_0'}(\pi_t,\nabla\log\pi_t)<\infty$ and
$\mathfrak E_{\alpha_0'}(\pi_t,1+M)<\infty, \ \forall t\in[0,T]$, and
$4\eta^2L_{\pi,\mathrm{space}}^2<\frac12$.
Then, for any $\alpha \in (0,\alpha_0)$ and $\alpha'\in(0,\alpha_0']$, discrete-time SALD \eqref{eq:sald_em_terminal_disc_prop_additive_final} satisfies
\begin{align*}
    \KL(\rho_K^\eta\|\pi_T)
    &\le
    \exp\!\left(
    -r\int_0^T \cLSI{t}\,\dd t
     \right)\exp\!\left(\frac{T}{ r\alpha}
    +\frac{2r\eta^2}{ \alpha'}\bar{\Gamma}
    \right)
    \KL(\rho_0\|\pi_0)
    \\
    &\quad+
    \exp\!\left(
    \frac{T}{ r\alpha}
    +\frac{2r\eta^2}{ \alpha'}\bar{\Gamma}
    \right)
    \left( \frac{1}{ r} \cA_{\alpha}(\pi,v)
    + 2r \eta \bar{\Delta}_{\alpha'}
    \right).
\end{align*}
\end{theorem}

Theorem \ref{thm:forward-KL-discrete} yields the following convergence rate regarding $r \to \infty$ and $\eta \to 0$ under $r \eta = O(1)$:
\[ 
    \KL(\rho_K^\eta\|\pi_T)
    = O\left( \exp\left( -r \int_0^T \cLSI{t} \dd t\right)\KL(\rho_0\|\pi_0)
    + \frac{1}{r}\cA_{\alpha}(\pi,v) + r \eta \bar{\Delta}_{\alpha'} \right).
\]

Set $\Lambda := \int_0^T \cLSI{t} \dd t$. To achieve $\varepsilon^2$-accurate solutions, $\KL(\rho_K^\eta\|\pi_T) \leq \varepsilon^2$, it is sufficient to choose $r = O\left(\max\{ \Lambda^{-1} \log \varepsilon^{-1},~ \cA_{\alpha}(\pi,v)\varepsilon^{-2}\}\right)$ and $\eta = O(\varepsilon^2 r^{-1}\bar{\Delta}_{\alpha'}^{-1})$. This results in the iteration complexity
\[
    K 
    = \frac{rT}{\eta} 
    = O\left( \max\left\{ 
        \frac{\bar{\Delta}_{\alpha'}}{\varepsilon^2 \Lambda^2} \left(\log \frac{1}{\varepsilon}\right)^2
        ,~ \frac{\bar{\Delta}_{\alpha'}\cA_{\alpha}(\pi,v)^2}{\varepsilon^{6}}
    \right\} \right).
\]

\cite{GuoTaoChen2024} established an iteration complexity of $O(\varepsilon^{-6})$ for the forward KL-divergence $\KL(\pi_T \| \cdot)$ without requiring an LSI. Our result for the reverse KL-divergence achieves a comparable complexity while additionally allowing for a mismatch between the initial distributions, i.e., $\rho_0 \neq \pi_0$ by exploiting the contraction of intermediate distributions, provided that the accumulated LSI $\int_0^T \cLSI{t} \dd t$ is positive. This condition is typically satisfied in practice, as $\pi_t$ is annealed toward a simpler distribution. 

The assumptions in Theorem~\ref{thm:forward-KL-discrete} are satisfied by a broad class of reverse diffusion paths. In particular, consider the VP forward process $(q_\tau)_{\tau \in [0,T]}$ initialized from $p_{\rm data}\propto e^{-V_0}$, and set $\pi_t=q_{T-t}$. If $V_0$ has Lipschitz continuous derivatives up to third order and satisfies a two-point dissipativity condition, then the reverse VP marginal path satisfies the score Lipschitz conditions \eqref{eq:lip_SALD_1}--\eqref{eq:lip_SALD_2}. The required $\alpha$-complexity terms are also finite for sufficiently small $\alpha, \alpha'>0$, as the dissipativity condition yields suitable quadratic exponential moment bounds. The full verification is given in Appendix~\ref{app:vp-dissipative-verification}.

\section{Convergence Analysis of Velocity-Aware SALD}\label{sec:convergence_velocityAwareSALD}
We analyze the general VA-SALD dynamics \eqref{eq:SALD_Ito} in the guided generation setting introduced in Section \ref{subsec:guided-generation}. Fix a reference vector field $c_t$. We call $m_t$ a correction field relative to $c_t$ if  
\begin{equation}\label{eq:va-sald-poisson}
\partial_t\pi_t+\nabla\cdot(\pi_tc_t)
=
-\nabla\cdot(\pi_tm_t).
\end{equation}
Then $v_t:=c_t+m_t$ is a transport velocity field of the guided path $(\pi_t)_{t\in[0,T]}$. The field $m_t$ represents the part of the guided evolution $\pi_t$ that is not captured by the selected reference field $c_t$. Consequently, the convergence of VA-SALD is governed by the complexity of $m_t$, rather than by that of the full guided velocity $v_t=c_t+m_t$.

\begin{theorem}\label{thm:general-moving-target-SALD}
    Assume that \(\pi_t\) satisfies LSI with constant \(\cLSI{t}\ge 0\), and that there exists \(\alpha_0>0\) such that $\mathfrak{E}_{\alpha_0}(\pi_t,m_t)<+\infty,~\forall t\in[0,T]$. Then, for any \(\alpha\in(0,\alpha_0]\), the law \((\rho_s)_{s\in[0,S]}\) of $(X_s)_{s\in[0,S]}$ defined by VA-SALD \eqref{eq:SALD_Ito} satisfies
    \begin{align}\label{eq:general_moving_target_KL_bound}
        \KL(\rho_S\|\pi_T)
        &\le
        \exp\left(
        -\int_0^T \frac{\sigma_t^2}{2}\dot s(t)\cLSI{t}\,\dd t
        \right)
        \exp\left(
        \int_0^T \sigma_t^{-2}\dot s(t)^{-1}\alpha^{-1}\,\dd t
        \right)
        \KL(\rho_0\|\pi_0)
        \nonumber\\
        &\quad+
        \int_0^T
        \exp\left(
        \int_t^T \sigma_u^{-2}\dot s(u)^{-1}\alpha^{-1}\,\dd u
        \right)
        \sigma_t^{-2}\dot s(t)^{-1}\mathfrak{E}_{\alpha}(\pi_t,m_t)\,\dd t.
    \end{align}
\end{theorem}

We specialize the above results to the case of linear slowdown, where $t(s)=s/r$ with $r\geq 1$. Set $\sigma_{\mathrm{min}}:= \min_{t\in [0,T]}\sigma_t$ and assume $\sigma_{\mathrm{min}} >0$.

\begin{corollary}\label{cor:general-moving-target-SALD}
    Under the same assumptions in Theorem \ref{thm:general-moving-target-SALD}, for any $\alpha \in (0,\alpha_0]$, VA-SALD \eqref{eq:SALD_Ito} satisfies
    \[\resizebox{0.98\linewidth}{!}{$\begin{aligned}\label{eq:ito_forward_KL_bound_calibrated_linear-slowdown}
        \KL(\rho_S\|\pi_T)
        &\le
        \exp\left(
        - \frac{r \sigma_{\mathrm{min}}^2}{2} \int_0^T \cLSI{t} \dd t
        \right)
        \exp\left(
        \frac{T}{r \sigma_{\mathrm{min}}^{2} \alpha}
        \right)
        \KL(\rho_0\|\pi_0)
        + \frac{1}{r \sigma_{\mathrm{min}}^{2}}
        \exp\left( 
         \frac{T}{r \sigma_{\mathrm{min}}^{2} \alpha}
        \right)
        \cA_{\alpha}(\pi,m).
    \end{aligned}$}\]
\end{corollary}

Furthermore, as $\cA_\alpha(\pi,m) \searrow \cA_0(\pi,m)$ as $\alpha \searrow 0$, we have the following convergence rate; for sufficiently small $\alpha > 0$, 
\[ 
    \KL(\rho_{S} \| \pi_T) 
    = O\left( \exp\left( - \frac{r \sigma_{\mathrm{min}}^2}{2} \int_0^{T}\cLSI{t} \dd t\right)\KL(\rho_0 \| \pi_0) 
    + \frac{\cA_0(\pi,m)}{r} \right).
\]

A discrete-time analysis is also provided in Theorem \ref{thm:general-moving-target-SALD-discrete} in Appendix \ref{sec:discrete_time_VA_SALD}.

According to Corollary \ref{cor:general-moving-target-SALD}, a linear slowdown implies $\KL(\rho_S\|\pi_T)\to 0$ as $r\to\infty$ under $\int_0^T \cLSI{t} \dd t>0$. The key advantage of VA-SALD over a direct application of SALD to the guided path $\pi_t\propto p_t e^{-f_t}$ lies in the associated complexity term. For the original SALD, the relevant complexity is governed by $\cA_\alpha(\pi,v)$, since the dynamics must track the entire guided path directly. In contrast, VA-SALD replaces this complexity with the correction-only complexity $\cA_\alpha(\pi,m)$, and there is a favorable regime in which VA-SALD is expected to provide an advantage. Under Option I, the residual represents the transport not captured by the pretrained marginal velocity. Under Option II, it is induced by the centered Doob residual of the prescribed guide.

\paragraph{Option I ($c_t^I = u_t$)}
Let $u_t:\mathbb R^d\to\mathbb R^d~(t\in[0,T])$ be a velocity field generating the pretrained path $(p_t)_{t\in[0,T]}$, namely $\partial_t p_t+\nabla\cdot(p_tu_t)=0$. For the guided path $\pi_t \propto p_t e^{-f_t}$, the same velocity field does not in general generate $(\pi_t)_{t\in[0,T]}$. Indeed,
\begin{equation}\label{eq:residual-term}
    \partial_t \pi_t+\nabla\cdot(\pi_t c_t^I) 
    = -\pi_t\Big(g_t-\E_{\pi_t}[g_t]\Big),
\end{equation}
where $g_t:=\partial_t f_t+\nabla f_t^\top u_t$ (see Proposition \ref{prop:guided_path_residual} in Appendix~\ref{app:convergence_velocityAwareSALD}). Thus, the residual term on the right-hand side precisely quantifies the part of the guided evolution that is not captured by the pretrained velocity field $u_t$, and a solution $m_t^I$ that solves $\nabla\cdot(\pi_t m_t^I)=\pi_t (g_t-\E_{\pi_t}[g_t])$ is the correction field that compensates for this mismatch.
If $\pi_t$ satisfies Poincar\'{e} inequality with a constant $\cPI{t}>0$ for all $t \in [0,T]$ and $m_t^I$ is the minimum $L_2(\pi_t)$-norm solution, we can bound the complexity term as follows. Applying Lemma \ref{lem:velocity-norm-bound} in Appendix \ref{subsec:velocity-norm} with $g = \E_{\pi_t}[g_t] - g_t$, we get $\|m_t^I\|_{L_2(\pi_t)} \leq \frac{1}{\sqrt{\cPI{t}}}\| g_t - \E_{\pi_t}[g_t] \|_{L_2(\pi_t)}$ and hence,
\begin{equation*}
    \cA_0(\pi, m^I)
    \leq \int_0^T \frac{1}{\cPI{t}}\Var_{\pi_t}[g_t] \dd t.
\end{equation*}
Therefore, Option I is favorable when the guided evolution is close to the pretrained marginal.
\vspace{-2mm}
\paragraph{Option II ($c_t^{\two}=u_t -\frac{\sigma_t^2}{2}\nabla f_t$)}
Let $h_t(x)=\E \left[e^{-f(X_T)} \mid X_t=x \right]$ be the exact guide associated with the Doob's $h$-transform and $F_t^h:= -\log h_t$, and define the operator $\cL_t^h$ as follows: for any function $\varphi: \R^d \rightarrow \R$,
\[\cL_t^h \varphi := (-B_t + \sigma_t^2 \nabla \log p_t - \sigma_t^2 \nabla F_t^h )^\top \nabla \varphi + \frac{\sigma_t^2}{2}\Delta \varphi,   \]
which corresponds to the infinitesimal operator of Doob's $h$-transform.
For the annealed guide function $f=(f_t)_{t \in [0,T]}$, we further define
\[
    \cR_t(f) := (\partial_t + \cL_t^h) \delta_t - \frac{\sigma_t^2}{2}\| \nabla \delta_t \|^2,
\]
where $\delta_t = f_t - F_t^h$ is the discrepancy between the guidance function $f_t$ and the exact guidance function $F_t^h$ associated with the Doob's $h$-transform.
Then, it follows that
\begin{equation}\label{eq:residual-term-ii}
    \partial_t \pi_t+\nabla\cdot(\pi_t c_t^{\two}) 
    = -\pi_t\Big(\cR_t(f)-\E_{\pi_t}[\cR_t(f)]\Big).
\end{equation}
The residual term on the right-hand side quantifies the gap between $f_t$ and $F_t^h$.
Let $m_t^{\two}$ be the corrector field that solves $\nabla\cdot(\pi_t m_t^{\two})=\pi_t (\cR_t(f)-\E_{\pi_t}[\cR_t(f)])$.
Following the same argument as Option I, when $\pi_t$ satisfies Poincar\'{e} inequality with a constant $\cPI{t}>0$ and $m_t^{\two}$ is the minimum $L_2(\pi_t)$-norm solution, we get $\|m_t^{\two}\|_{L_2(\pi_t)} \leq \frac{1}{\sqrt{\cPI{t}}}\| \cR_t(f) - \E_{\pi_t}[\cR_t(f)] \|_{L_2(\pi_t)}$, and 
\begin{equation*}
    \cA_0(\pi, m^{\two})
    \leq \int_0^T \frac{1}{\cPI{t}}\Var_{\pi_t}[\cR_t(f)] \dd t.
\end{equation*}
Therefore, Option II is favorable when the prescribed guide $f_t$ is close to $F_t^h$ and $\cA_\alpha(\pi, m^{\two})=0$ when $f_t = F_t^h$. For more details, see Appendix \ref{subsec:velocity-norm}.

\section{Experiments}\label{sec:experiments}
We empirically evaluate SALD and VA-SALD in both controlled and practical guided generation settings. The synthetic experiments are designed to test the theoretical prediction that increasing the slowdown budget improves tracking of the terminal distribution, where distributional errors can be computed directly. The image-generation experiments then examine whether the velocity-aware slowdown principle improves inference-time, training-free guidance with pretrained generative models and black-box rewards.

\subsection{Synthetic Data}

We evaluate guided generation on two 2D VP tasks: a two-Gaussian law with a
two-moons guide, and an eight-Gaussian ring with a penalty on the four left-half
modes. All methods use the same VP family, score, guide, and matched
particle-step budgets. For DOIT~\cite{zhu2026training},
$r$ denotes the computational cost; see Appendix \ref{app:additional_syth} for the precise setup and
implementation details. Figure~\ref{fig:synthetic-main} shows that SALD improves
monotonically with $r$, VA-SALD achieves low KL with smaller budgets, and DOIT
remains far from the guided terminal target under matched budgets.

\subsection{Guided Image Generation}

We further evaluate our VA-SALD with Option I on training-free guided image generation tasks. 
Specifically, we focus on practical inference-time guidance for high-resolution image generation using \textbf{black-box} reward functions.

We employ the pre-trained Stable-Diffusion 3.5 Medium (SD-3.5M)~\cite{esser2024scaling}  as our backbone and the aesthetic score as the black-box reward function.  We evaluate our VA-SALD by comparing it with two closely related baselines, FM-ZG and FM-Evolv~\cite{wei2025evolvable},  that use the default (Flow-Matching) SDE sampler in~\cite{liu2025flow} and zeroth-order gradient estimator for guided generation. More details can be found in Appendix \ref{app:additional_stablediff}.
Figure~\ref{lion_gc8} shows that our VA-SALD achieves more stable guided generation compared with FM-ZG and FM-Evolv. 

Our image-generation experiments are intended to illustrate the stability of VA-SALD under black-box inference-time guidance. Beyond the representative example in Figure~\ref{lion_gc8}, we provide additional text-to-image generation results in Appendix~\ref{app:additional_stablediff}, where reward functions, including aesthetic score, pickscore, and CLIPscore, are used as guidance functions. These experiments suggest that velocity-aware slowdown can help preserve the pretrained generative trajectory while incorporating external guidance, although a more systematic evaluation across broader prompt sets and human-preference metrics remains an important direction for future work.

\begin{figure}[t]
    \centering
    \subfloat[]{\includegraphics[width=0.248\textwidth]{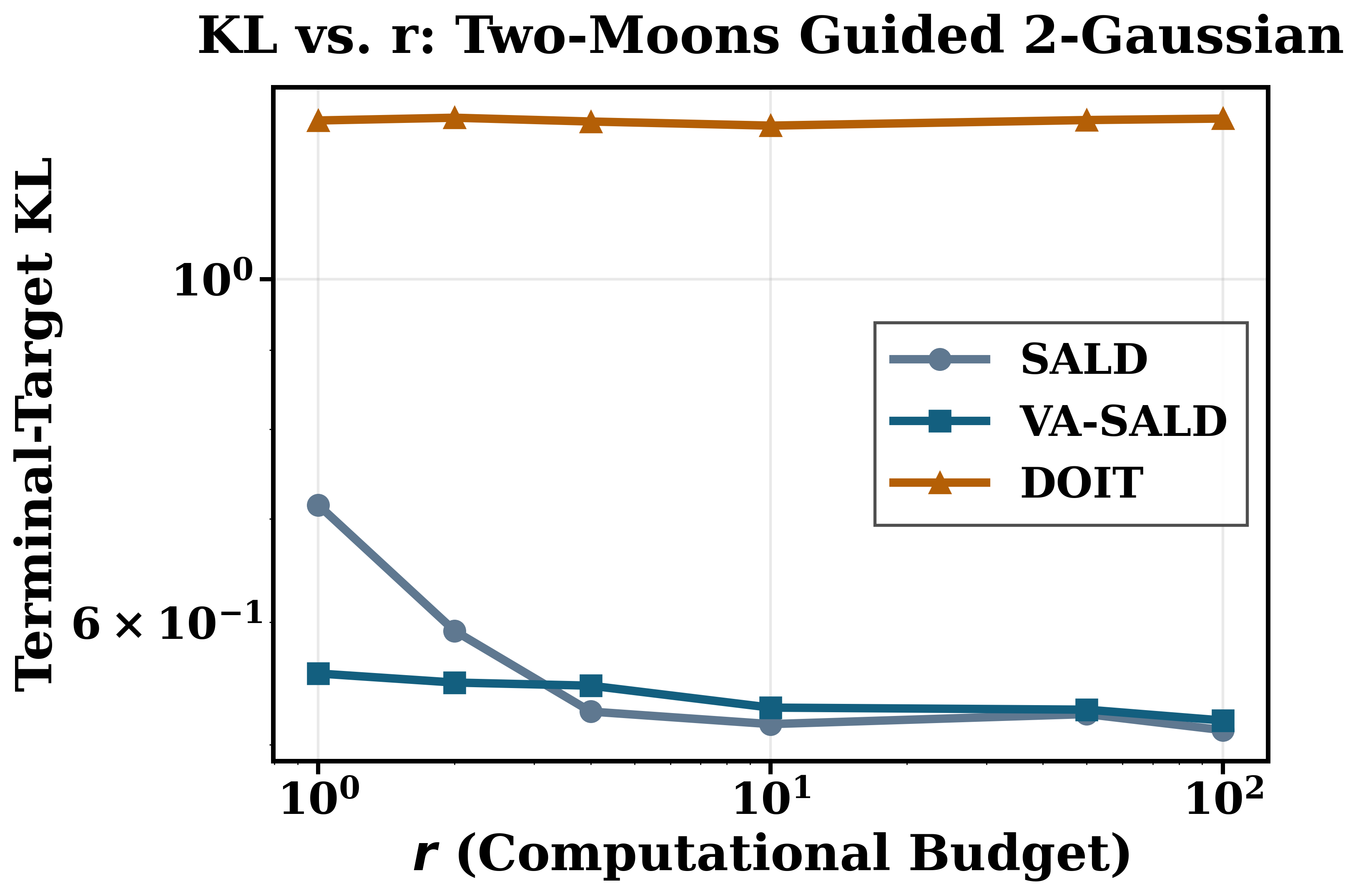}\label{fig:synthetic-main-a}}
    \subfloat[]{\includegraphics[width=0.248\textwidth]{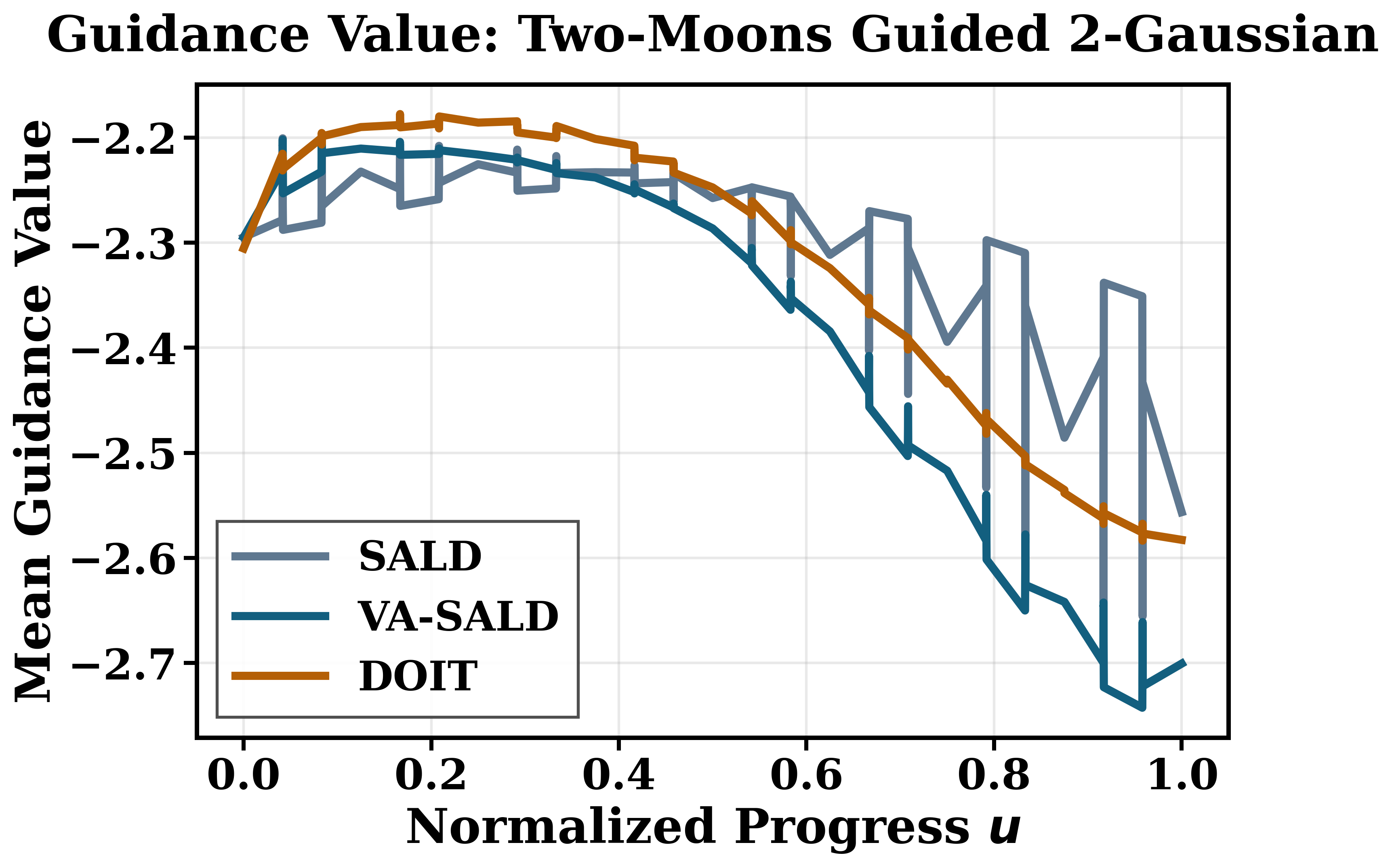}\label{fig:synthetic-main-b}}
    \subfloat[]{\includegraphics[width=0.248\textwidth]{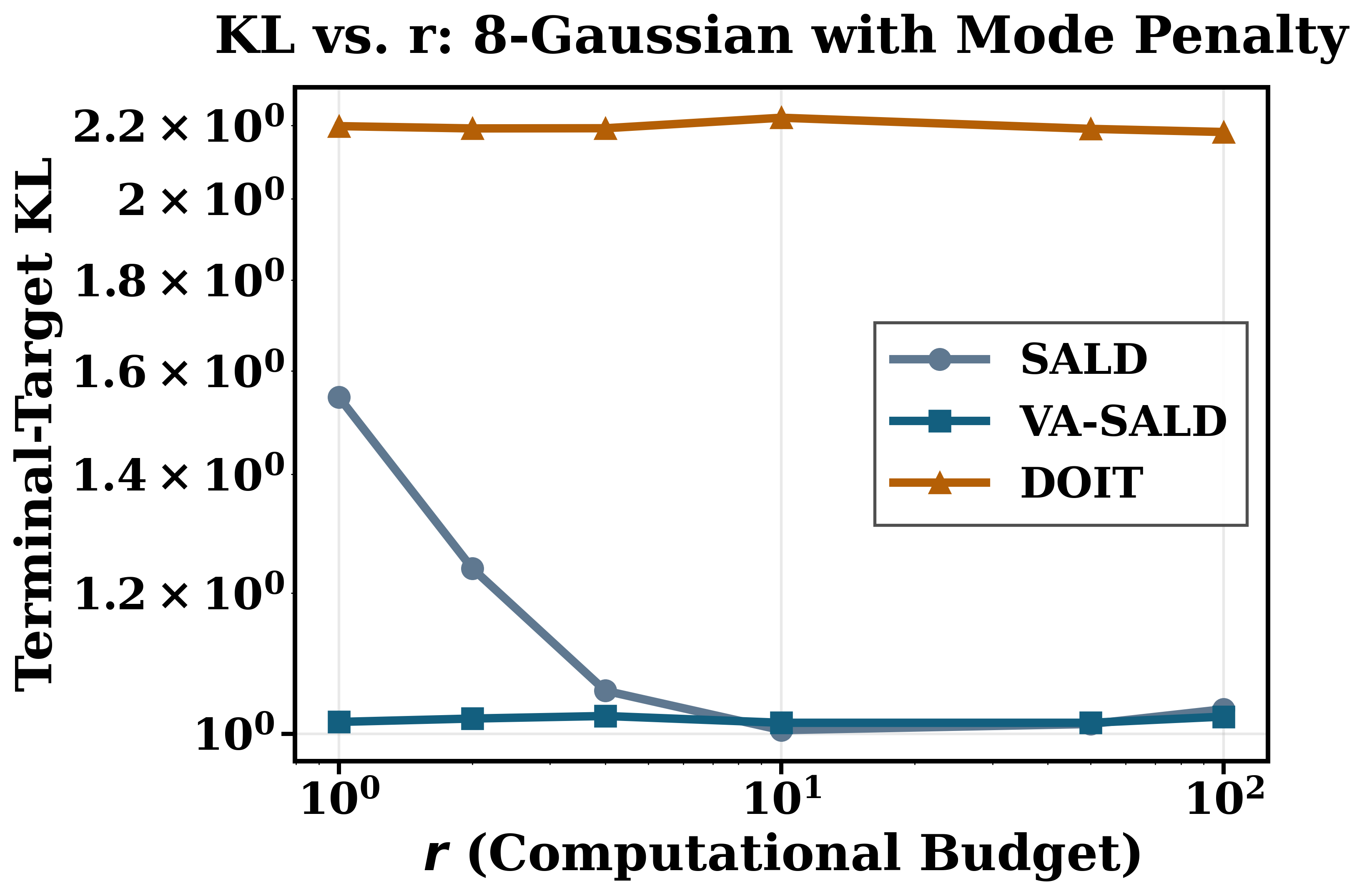}\label{fig:synthetic-main-c}}
    \subfloat[]{\includegraphics[width=0.248\textwidth]{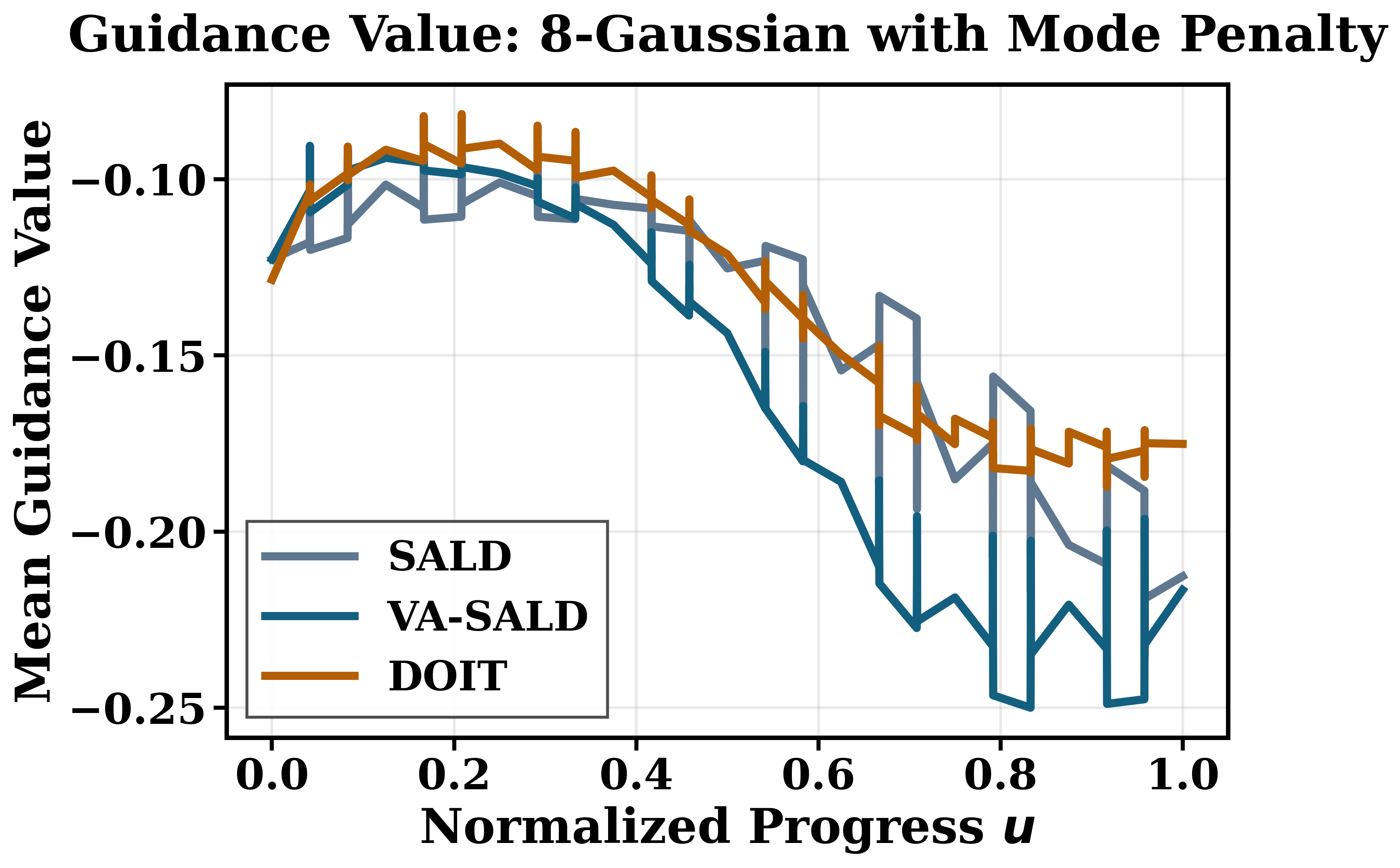}\label{fig:synthetic-main-d}}

    \caption{Synthetic guided VP experiments.
    We compare SALD, VA-SALD, and DOIT on two guided VP benchmarks:
two-Gaussian/two-moons (a,b) and eight-Gaussian/left-mode penalty (c,d).
Panels (a,c) show terminal-target KL versus budget $r$, and panels (b,d)
show mean guidance value along reverse progress. Unlike DOIT, SALD and
VA-SALD consistently reduce terminal KL as $r$ grows; VA-SALD also yields
the strongest guidance-value profiles.}
    \label{fig:synthetic-main}
\end{figure}

\begin{figure}[t] 
    \centering
\includegraphics[width=1\textwidth]{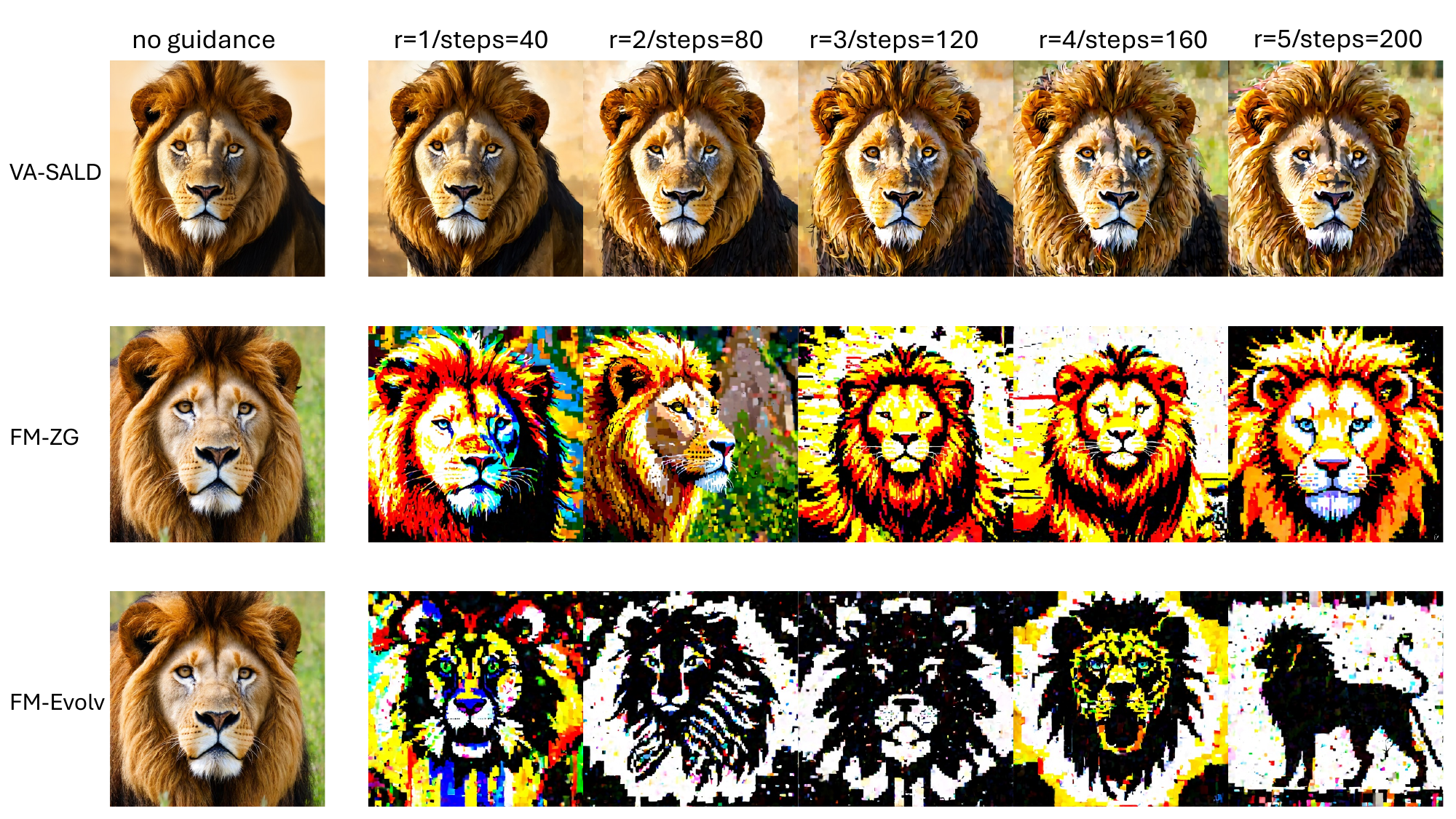} 
    \caption{Guided generation on "lion"   with guidance of black-box reward functions based on neural Aesthetic scorer. Across increasing budgets, VA-SALD produces stable and progressively refined
samples, whereas FM-ZG and FM-Evolv exhibit severe artifacts and unstable
semantic preservation.}
\label{lion_gc8}
\end{figure}

\section*{Conclusion and Discussion}
We studied Slowly Annealed Langevin Dynamics (SALD) for tracking moving target distributions and approximating their terminal target via time slowdown. Our reverse KL analysis, based on marginal KL differentiation, shows how slowdown amplifies contraction induced by functional inequalities of intermediate targets. We also introduced Velocity-Aware SALD (VA-SALD) for training-free guided generation with pretrained score-based generative models. By incorporating the pretrained marginal velocity, VA-SALD separates the complexity of the pretrained path from the additional deviation induced by guidance, yielding convergence bounds governed by the correction field \(m_t\) rather than the full guided velocity.

Several limitations and directions remain. First, slowdown improves tracking by giving the sampler more algorithmic time to follow the moving target, but it also increases the number of iterations. Combining SALD with acceleration or distillation techniques is an important direction for retaining the benefit of slow tracking while improving generation efficiency. Second, although VA-SALD avoids paying for the full complexity of the guided moving target path and instead depends on the guide-induced correction \(m_t\), it remains unclear how to choose the intermediate guide schedule \((f_t)_{t\in[0,T]}\) for a given terminal guide \(f\). Designing guide schedules that reduce the complexity of \((m_t)_{t\in[0,T]}\) remains an open problem. Moreover, our theoretical guarantees are established in an oracle setting, assuming exact access to the pretrained score and to the gradient of a smooth guide.

\section*{Acknowledgements}
This research is supported by the National Research Foundation, Singapore and the Ministry of Digital Development and Information under the AI Visiting Professorship Programme (award number AIVP-2024-004). Any opinions, findings and conclusions or recommendations expressed in this material are those of the author(s) and do not reflect the views of National Research Foundation, Singapore and the Ministry of Digital Development and Information. Yueming Lyu is partially supported by the Career Development Fund (CDF) of the Agency for Science, Technology and Research (A*STAR) (No: C243512014)

\bibliographystyle{plain}
\bibliography{ref.bib}



\clearpage


{

{
\renewcommand{\contentsname}{Table of Contents}
\tableofcontents
}

\newpage

{\centering \Large\bfseries Appendix \par}
\bigskip

\section{Other Related Work}
\paragraph{Annealed Langevin}

Annealed Langevin dynamics is based on the idea of sampling through a path of intermediately-noised distributions rather than directly targeting a (potentially difficult to sample from) terminal law. In score-based generative modeling, this manifests in SMLD (score-matching with Langevin dynamics) and continuous-time score-based diffusion models, where a time-dependent score \(\nabla \log \pi_t\) is used to move samples through progressively less noisy marginals toward the data distribution \cite{song2019generative, song2021score, baldassari2026dimension}. In sampling theory, annealing and tempering are used to improve the exploration of multi-modal or non-log-concave targets, by introducing intermediate distributions with more favorable geometry \cite{ge2018simulated, chehab2025provable}. Recent work has developed non-asymptotic theory for several variants of annealed Langevin methods, including geometric tempering \cite{chehab2025provable}, diffusion annealed Langevin Monte Carlo along Gaussian or heavy-tailed
interpolation paths \cite{cordero2025non}, diffusion annealed Langevin dynamics from the viewpoint of prescribed
marginal flows and functional inequalities \cite{cattiaux2025diffusion}, dimension-robust preconditioned annealed Langevin
dynamics for mixture-like multimodal targets \cite{baldassari2026dimension}, and posterior sampling methods that combine diffusion priors with annealed Langevin updates \cite{xunposterior}.

The work closest to ours is the slowly annealed Langevin Monte Carlo analysis \cite{GuoTaoChen2024}, which shows that traversing an annealing path more slowly can improve sampling complexity
for non-log-concave targets. Their analysis proceeds through a path-space comparison based on Girsanov's theorem, and the resulting guarantee is governed by the action of the annealing path. Our analysis is complementary. Instead of comparing path measures, we directly differentiate the marginal reverse KL divergence between the sampler law and the instantaneous target. This reveals two distinct mechanisms behind slowdown: the moving-target forcing term is reduced by the time-rescaling, while the dissipative Fisher-information term can be converted into contraction whenever the intermediate targets satisfy log-Sobolev inequalities. As a result, our bounds explicitly separate initialization mismatch, contraction along the path, and the energy of the target evolution. This marginal viewpoint also leads naturally to discrete-time guarantees for Euler--Maruyama implementations and to iteration-complexity bounds (e.g. for Gaussian-mixture annealing paths).

Our formulation is also related to the two-loop structure of SMLD \cite{song2019generative}, where one alternates between changing the noise level and applying several Langevin correction steps at a fixed level. SALD can be viewed as a continuous moving-target analogue of this intuition, but it is not simply the continuous-time limit of SMLD: the target evolves continuously in algorithmic time, and the analysis applies to general moving distributions rather than only to diffusion-model noise marginals. This distinction is important for our theory, since the key object is not only the score field but also the transport cost, or energy of the prescribed target path.

\paragraph{Guided generation and alignment}

Guided generation modifies a pretrained generative sampler to favor samples satisfying a certain condition, reward, preference, or constraint. Classical diffusion guidance mechanisms include the classifier guidance and classifier-free guidance paradigms \cite{song2021score, ho2022classifier}. Recent work has further clarified the structure of such guidance rules; for instance, classifier-free guidance can be interpreted through the lens of predictor-corrector mechanisms \cite{bradley2025classifierfree}. More broadly, training-free guidance methods seek to steer a pretrained diffusion model (at inference time) without retraining the base score network, often by adding a reward-gradient or condition-dependent drift to the reverse sampling process \cite{chen2024overview, jiao2025towards, ye2024tfg, liu2026alignment}. Additionally, \cite{kim2025testtime} proposes a training-free test-time method based on a (tempered) sequential Monte Carlo (SMC) method to align diffusion models with arbitrary rewards.  In addition, several works focus on addressing black-box guidance. \cite{lyu2024covariance} formulates the training-free guided generation as a sequential black-box optimization problem of a sequence of Gaussian distributions and explicitly optimizes the drift term. Furthermore,  \cite{wei2025evolvable} employs the zeroth-order gradient update at each inference-time sampling step. TreeG~\cite{guo2025training} builds a tree-search scheme along the inference-time sampling process.  DSearch~\cite{li2025dynamic} investigates a dynamic tree expansion for inference-time guidance.  

Our work contributes to this perspective by considering a moving-target sampling framework. Given a pretrained marginal path \((p_t)_{t \in [0,T]}\) and a guide \(f_t\), the natural guided path is 
\[\pi_t(x) \propto p_t(x) \exp(-f_t(x)).\]

A direct application of SALD would require tracking the full evolution of \((\pi_t)_{t \in [0,T]}\). In contrast, VA-SALD explicitly uses the transport velocity of the pretrained path \((p_t)_{t \in [0,T]}\) and applies Langevin correction only with respect to the guided density. This separates the base generative evolution from the additional guide-induced deformation. In our theory, this distinction appears through the correction vector field \(m_t\) with the relevant energy complexity term \(\mathfrak{E}_\alpha(\pi_t, m_t)\). This is in contrast to direct SALD on the guided path which involves the complexity and energy of the full guided velocity. Thus, VA-SALD is especially well-suited to settings where the pretrained model already captures most of the marginal evolution and the guide only introduces a comparatively small correction, or where the prescribed guide is close to the guide associated with Doob's $h$-transform.

\paragraph{Doob's $h$-Transform to align diffusion models}

A principled way to modify Markov diffusion so that it targets its desired target tilted distribution is given by Doob's \(h\)-transform. If
\[
    h_t(x) := \mathbb{E}\!\left[\exp(-f(X_T)) \mid X_t=x\right],
\]
then adding the correction $\nabla \log h_t$ to the reverse-time score yields the exact transformed process whose terminal law is proportional to \(p_T e^{-f}\), under the idealized setting where \(h_t\) is known exactly. This makes the \(h\)-transform a natural theoretical tool for conditional generation, preference alignment, and reward-guided diffusion sampling \cite{tang2024stochastic, zhu2026training, kawatadirect, chang2026inference}

\cite{tang2024stochastic} tackles the problem of supervised fine-tuning with endogenous conditioning, in which they learn the $h$-function through a martingale-covariation loss. In the alignment for diffusion models in the Direct Preference Optimization (DPO) setting, a recent work \cite{kawatadirect} formulates a distributional optimization problem and subsequently enable sampling from the learned distribution through the $h$-transform independent of isoperimetric inequalities. Doob's matching \cite{chang2026inference} is an inference-time alignment method that estimates the $h$-function via a gradient-regularized regression objective whose output is variationally stable (outputs converge as well as their gradients, yielding good approximation). DOIT (Doob-Oriented Inference-time Transformation) \cite{zhu2026training} instead approximates the dynamic Doob correction at inference time through simulation, allowing training-free adaptation to more general reward functions, including non-differentiable rewards.

Our approach is different in both mechanism and objective; VA-SALD does not attempt to learn or estimate the exact \(h_t\), and thus does not require solving a backwards conditional-expectation correction, training a time-dependent guidance network, or repeatedly estimating the exact Doob correction. Instead, it allows one to prescribe a guide \(f_t\) directly and uses slowdown to control the bias introduced by this simpler guidance rule. The resulting guarantee is not an exact \(h\)-transform guarantee; rather, it quantifies how well the sampler tracks the guided moving target under functional inequlities of \(\pi_t\) and the energy of the guide-induced correction field \(m_t\). Thus, our work is positioned as a lightweight, training-free alternative to exact \(h\)-transform alignment since it trades exact Doob optimality for a simple, inference-time algorithm whose error is tracked by a moving-target KL analysis.  

\section{Auxiliary Lemmas}\label{sec:lemmas}
We present useful lemmas on which our convergence theory relies.

\begin{lemma}[Gr\"onwall's inequality]\label{lem:gronwall}
    Let $a_t, b_t$ be continuous functions and $K_t$ be a differentiable function in $t \in [0,t_1]$. Suppose $\frac{\dd}{\dd t}K_t \leq -a_t K_t + b_t~(t\in [0,t_1])$. Then, it follows that,
    \begin{align*}
        K_{t_1} 
        \leq \exp\left( -\int_0^{t_1} a_{u} \dd u \right) K_0 
        + \int_0^{t_1} \exp\left( -\int_t^{t_1} a_u \dd u \right) b_t \dd t.
    \end{align*}
\end{lemma}
\begin{proof}
    We get following inequality:
    \begin{align*} 
        \frac{\dd}{\dd t}\left( \exp\left( \int_0^t a_u \dd u \right) K_t \right)
        &= a_t \exp\left( \int_0^t a_u \dd u \right) K_t 
        + \exp\left( \int_0^t a_u \dd u \right) \frac{\dd K_t}{\dd t} \\
        &\leq \exp\left( \int_0^t a_u \dd u \right) b_t.
    \end{align*}
    Integrating the above inequality, we get $\exp\left( \int_0^{t_1} a_{u} \dd u \right) K_{t_1} \leq K_0 + \int_0^{t_1}\exp\left( \int_0^t a_u \dd u \right) b_t \dd t$. Therefore,
    \begin{align*}
        K_{t_1} 
        &\leq \exp\left( -\int_0^{t_1} a_{u} \dd u \right) K_0 
        + \exp\left( -\int_0^{t_1} a_{u} \dd u \right)\int_0^{t_1}\exp\left( \int_0^t a_u \dd u \right) b_t \dd t \\
        &= \exp\left( -\int_0^{t_1} a_{u} \dd u \right) K_0 
        + \int_0^{t_1} \exp\left( -\int_t^{t_1} a_u \dd u \right) b_t \dd t.
    \end{align*}
\end{proof}

\begin{lemma}[Corollary 4.15, \cite{boucheron2013concentration}]\label{lem:dv_variation}
    Let $\mu$ and $\nu$ be the probability distributions on the same space. Then,
    \[
        \KL(\nu \| \mu) = \sup_Z \left\{ \E_{\nu}[Z] - \log \E_{\mu}[\exp(Z)] \right\},
    \]
    where the supremum is taken over all random variables such that $\log \E_{\mu}[\exp(Z)] < + \infty$.
\end{lemma}

\section{Bounding the Complexity under Poincar\'{e} Inequality}\label{subsec:velocity-norm}
We present a way to estimate the complexity under Poincar\'{e} Inequality (PI).
\begin{definition}[Poincar\'{e} Inequality (PI)]\label{def:PI}
    We say a probability distribution $\mu$ satisfies PI with constant $\cPI~>0$ if for all smooth $\varphi$ 
    \[
        \Var_\mu[\varphi] 
        = \E_{\mu}[ \varphi^2 ] - \E_{\mu}[\varphi]^2 
        \le 
        \frac{1}{\cPI~} \int \|\nabla \varphi\|^2 \dd\mu.
    \]
\end{definition}

Let $\mu$ be a probability distribution that satisfies PI with a constant $\cPI~ > 0$.
We consider the weighted mean-zero Sobolev space defined as follows:
\begin{equation}
    \dot{\bH}^1(\mu) 
    = \left\{ \varphi \in \bH^1(\mu) \left| \int \varphi(x) \mu(x) \dd x = 0 \right. \right\},
\end{equation}
where $\bH^1(\mu)=W^{1,2}(\mu)$ is a weighted Sobolev space. 
We equip $\dot{\bH}^1(\mu)$ with the inner-product defined as $\ip{\varphi}{\psi}_{ \dot{\bH}^1(\mu)} = \int \nabla \varphi(x)^\top \nabla \psi(x) \mu(x)\dd x$ for all $\varphi, \psi \in \dot{\bH}^1(\mu)$.
PI implies that for any $\varphi \in \dot{\bH}^1(\mu)$,  
\begin{align*} 
    \| \varphi \|^2_{\dot{\bH}^1(\mu)} 
    \leq \| \varphi \|^2_{\bH^1(\mu)} 
    = \int \varphi^2(x) \mu(x) \dd x 
    + \int \| \nabla \varphi(x) \|_2^2 \mu(x)\dd x 
    \leq \left( 1 + \frac{1}{\cPI~}\right) \| \varphi \|^2_{\dot{\bH}^1(\mu)}.
\end{align*}
Therefore $\|\cdot \|_{\dot{\bH}^1(\mu)}$ defines the norm equivalent to $\| \cdot \|_{\bH^1(\mu)}$ on $ \dot{\bH}^1(\mu)$.

\begin{lemma}\label{lem:velocity-norm-bound}
    Given $g \in L_2(\mu)$, consider the PDE: $\nabla \cdot (\mu v) = - g \mu$ where $v=(v_1,\ldots,v_d),$~$(v_i \in L_2(\mu_t))$.
    If $\mu$ satisfies PI with a constant $\cPI~$, then the minimum $L_2(\mu)$-norm solution $v$ exists, with $\|v\|_{L_2(\mu)} \leq \frac{1}{\sqrt{\cPI~}}\| g \|_{L_2(\mu)}$.
\end{lemma}
\begin{proof}
    Here, we restrict the class of solutions to those of the form $v = \nabla \varphi$ with $\varphi \in \dot{\bH}^1(\mu)$.
    The weak form of $\nabla \cdot (\mu \nabla \varphi) = - g \mu$ is as follows; for any $\psi \in \dot{\bH}^1(\mu)$,
    \begin{equation}\label{eq:weak-PDE}
        \ip{\varphi}{\psi}_{ \dot{\bH}^1(\mu)} 
        = \int \psi(x) g(x) \mu(x)\dd x =: T_\mu(\psi).
    \end{equation}

    Under PI, the following inequality holds 
    \begin{equation*}
        T_\mu(\psi) 
        \leq \| \psi \|_{L_2(\mu)} \| g \|_{L_2(\mu)}
        =  \sqrt{\Var_\mu [\psi]} \| g \|_{L_2(\mu)}
        \leq \frac{1}{\sqrt{\cPI~}}  \| \psi \|_{\dot{\bH}^1(\mu)}\| g \|_{L_2(\mu)},
    \end{equation*}
    meaning that $T_\mu: \dot{\bH}^1(\mu) \to \R$ is the bounded linear operator with the operator norm $\frac{1}{\sqrt{\cPI~}}\| g \|_{L_2(\mu)}$. Therefore, the Riesz representation theorem yields the solution $\varphi$ of Eq.~\eqref{eq:weak-PDE} with 
    \[ 
        \| \nabla \varphi \|_{L_2(\mu)}
        = \| \varphi \|_{\dot{\bH}^1(\mu)} \leq \frac{1}{\sqrt{\cPI~}}\| g \|_{L_2(\mu)}.
    \]

    Let $w$ be another solution. Then, $\nabla\cdot (\mu (w - \nabla \varphi))=0$ and by integration by parts, 
    \[
        \int \nabla \varphi(x)^\top (w - \nabla \varphi)(x) \mu(x)\dd x 
        = -  \int \varphi(x) \nabla \cdot (w - \nabla \varphi)(x) \mu(x)\dd x 
        = 0.
    \]
    Therefore, we can conclude that $v\nabla \varphi$ attains the minimum norm solution in $L_2(\mu)$ as follows;
    \begin{align*}
        \| w \|_{L_2(\mu)}^2 
        &= \| \nabla \varphi +  (w - \nabla \varphi )\|_{L_2(\mu)}^2 \\
        &= \|  \nabla \varphi \|_{L_2(\mu)}^2 
        + \| w - \nabla \varphi \|_{L_2(\mu)}^2 \\
        &\geq \|  \nabla \varphi \|_{L_2(\mu)}^2.
    \end{align*}
\end{proof}

We here demonstrate that Lemma \ref{lem:velocity-norm-bound} immediately yields the bound on the energy of moving targets $(\pi_t)_{t \in [0,T]}$. We suppose $\pi_t$ satisfies Poincar\'{e} inequality with a constant $\cPI{t}$ for all $t \in [0,T]$ and consider PDE: $\nabla\cdot (\pi_t v_t) = -\pi_t (\partial_t \log \pi_t) = - \partial_t \pi_t$.
Then, Lemma \ref{lem:velocity-norm-bound} with $\mu=\pi_t$ and $g = \partial_t \log \pi_t$ guarantees the solution of $v_t$ with $\|v_t\|_{L_2(\pi_t)} \leq \frac{1}{\sqrt{\cPI{t}}} \| \partial_t \log \pi_t  \|_{L_2(\pi_t)}$. Therefore,

\begin{equation*}
    \cA_0(\pi,v) \leq \int_0^T \frac{1}{\cPI{t}} \| \partial_t \log \pi_t  \|_{L_2(\pi_t)}^2 \dd t.
\end{equation*}

This improves the constant in the estimate of \cite{GuoTaoChen2024} from $2$ to $1$. 

Lemma \ref{lem:velocity-norm-bound} can be used to evaluate the complexity terms of VA-SALD (see Section \ref{sec:convergence_velocityAwareSALD}, \ref{subsec:complexity-of-option-I}, \ref{subsec:complexity-of-option-II}).

\section{Analysis of SALD}\label{sec:proof_of_SALD}
In this section, we provide the convergence analysis of SALD described in Section \ref{sec:convergence_sald}.

\subsection{Analysis of Continuous-Time SALD}\label{sec:continuous_time_SALD}
We first provide the proof of the convergence rate of SALD (Theorem~\ref{thm:forward-KL}).

\begin{proof}[Proof of Theorem~\ref{thm:forward-KL}]

    Differentiating $\KL(\rho_s \| \tilde{\pi}_s )$ in time $s$, we get
    \begin{equation}\label{eq:KL-derivative-0}
        \frac{\dd}{\dd s}\KL(\rho_s \| \tilde{\pi}_s )
        = \int \partial_s \rho_s \log \frac{\rho_s}{\tilde{\pi}_s}\dd x
        - \int \frac{\rho_s}{\tilde{\pi}_s}\partial_s \tilde{\pi}_s \dd x
    \end{equation}
    since $\int \partial_s \rho_s \dd x=0$.

    Using the following Fokker--Planck equation associated with~\eqref{eq:SALD}
    \[
        \partial_s\rho_s = \nabla\cdot\left(\rho_s \nabla \log \frac{\rho_s}{\tilde{\pi}_s}\right),
    \]
    the first term of Eq.~\eqref{eq:KL-derivative-0} becomes
    \begin{align}\label{eq:KL-derivative-1}
        \int \partial_s \rho_s \log \frac{\rho_s}{\tilde{\pi}_s}\dd x
        &= - \int \rho_s \left\| \nabla \log \frac{\rho_s}{\tilde{\pi}_s} \right\|^2 \dd x \notag \\
        &= - \FI(\rho_s \| \tilde{\pi}_s).
    \end{align}

    Let $v_t$ $(t \in [0,T])$ be a velocity field satisfying the continuity equation
    \[
        \partial_t \pi_t + \nabla\cdot (v_t \pi_t)=0.
    \]
    Then $\tilde{v}_s := \dot{t}(s)v_{t(s)}$ generates $\tilde{\pi}_s$ because
    \[
        \partial_s \tilde{\pi}_s
        = \dot{t}(s)\partial_t \pi_t |_{t=t(s)}
        = -\dot{t}(s)\nabla\cdot(v_{t(s)}\pi_{t(s)})
        = -\nabla\cdot(\tilde{v}_s \tilde{\pi}_s).
    \]

    The second term of Eq.~\eqref{eq:KL-derivative-0} can be evaluated as follows:
    \begin{align}\label{eq:KL-derivative-2}
       - \int \frac{\rho_s}{\tilde{\pi}_s}\partial_s \tilde{\pi}_s \dd x
       &= \int \frac{\rho_s}{\tilde{\pi}_s} \nabla\cdot (\tilde{v}_s \tilde{\pi}_s) \dd x \notag \\
       &= - \int \nabla \frac{\rho_s}{\tilde{\pi}_s}^\top  (\tilde{v}_s \tilde{\pi}_s) \dd x \notag \\
       &= - \int \rho_s \nabla \log \frac{\rho_s}{\tilde{\pi}_s}^\top  \tilde{v}_s  \dd x \notag \\
       &\leq \sqrt{\FI(\rho_s \| \tilde{\pi}_s)} \| \tilde{v}_s \|_{L_2(\rho_s)} \notag \\
       &\leq \frac{1}{2}\FI(\rho_s \| \tilde{\pi}_s)
       + \frac{1}{2}\| \tilde{v}_s \|_{L_2(\rho_s)}^2.
    \end{align}

    Combining Eq.~\eqref{eq:KL-derivative-0}, \eqref{eq:KL-derivative-1}, \eqref{eq:KL-derivative-2}, and LSI, we get
    \begin{align*}
        \frac{\dd}{\dd s}\KL(\rho_s \| \tilde{\pi}_s)
        &\leq - \frac{1}{2}\FI(\rho_s \| \tilde{\pi}_s)
        + \frac{1}{2}\|\tilde{v}_s\|_{L_2(\rho_s)}^2 \\
        &\leq - \cLSI{t(s)} \KL(\rho_s \| \tilde{\pi}_s)
        + \dot{t}(s)^2\frac{1}{2}\|v_{t(s)}\|_{L_2(\rho_s)}^2.
    \end{align*}
    Since
    \[
        \frac{\dd}{\dd t}\KL(\rho_{s(t)} \| \tilde{\pi}_{s(t)})
        =\dot{s}(t)\frac{\dd}{\dd s}\KL(\rho_s \| \tilde{\pi}_{s})\Big|_{s=s(t)},
    \]
    we further get
    \begin{align*}
        \frac{\dd}{\dd t}\KL(\rho_{s(t)} \| \pi_t)
        \leq - \dot{s}(t)\cLSI{t} \KL(\rho_{s(t)} \| \pi_t)
        + \frac{\dot{s}(t)^{-1}}{2}\|v_{t}\|_{L_2(\rho_{s(t)})}^2.
    \end{align*}

    By Lemma~\ref{lem:dv_variation},
    \[
        \|v_t\|_{L_2(\rho_{s(t)})}^2
        \leq
        \frac{1}{\alpha}\KL(\rho_{s(t)} \| \pi_t)
        + \frac{1}{\alpha} \log \E_{\pi_t}\left[ \exp(\alpha \|v_t\|^2) \right].
    \]
    Therefore,
    \begin{align*}
        \frac{\dd}{\dd t}\KL(\rho_{s(t)} \| \pi_t)
        &\leq - \left(\dot{s}(t)\cLSI{t} - \frac{1}{2}\dot{s}(t)^{-1}\alpha^{-1} \right)\KL(\rho_{s(t)} \| \pi_t)
        + \frac{1}{2}\dot{s}(t)^{-1}\mathfrak{E}_{\alpha}(\pi_t,v_t).
    \end{align*}

    Applying Lemma~\ref{lem:gronwall}, we get
    \begin{align*}
        \KL(\rho_{S} \| \pi_T)
        &\leq \exp\left( - \int_0^{T} \left(\dot{s}(t)\cLSI{t} - \frac{1}{2}\dot{s}(t)^{-1}\alpha^{-1}\right)\dd t \right)\KL(\rho_0 \| \pi_0) \\
        &\quad+ \int_0^{T} \exp\left(-\int_t^{T} \left(\dot{s}(u)\cLSI{u} - \frac{1}{2}\dot{s}(u)^{-1}\alpha^{-1}\right) \dd u\right)  \frac{1}{2}\dot{s}(t)^{-1}\mathfrak{E}_{\alpha}(\pi_t,v_t) \dd t \\
        &\leq \exp\left( - \int_0^{T} \dot{s}(t)\cLSI{t} \dd t\right)
        \exp\left( \int_0^{T} \frac{1}{2}\dot{s}(t)^{-1}\alpha^{-1}\dd t\right)\KL(\rho_0 \| \pi_0) \\
        &\quad+ \int_0^{T} \exp\left(\int_t^{T} \frac{1}{2}\dot{s}(u)^{-1}\alpha^{-1} \dd u\right)  \frac{1}{2}\dot{s}(t)^{-1}\mathfrak{E}_{\alpha}(\pi_t,v_t) \dd t.
    \end{align*}


\end{proof}

\subsection{Analysis of Discrete-Time SALD}\label{sec:discrete_time_SALD}
We next prove the convergence rate of discrete-SALD. To do so, we here define the continuous interpolation $(\hat X_s)_{s\in[0,S]}$ of Eq.~\eqref{eq:sald_em_terminal_disc_prop_additive_final} as follows.
For time $s\in[s_k,s_{k+1}]$, 
 \begin{equation}\label{eq:frozen_interp_terminal_disc_prop_additive_final}
     \hat X_s
     :=
     X_k^\eta
     +(s-s_k)\,\nabla\log\pi_{t_k}(X_k^\eta)
     +\sqrt2\,(W_s-W_{s_k}).
 \end{equation}

The next lemma provides the bound on $\int \hat\rho_s\langle \delta_{\pi_t}(x),\nabla\log\frac{\hat\rho_s}{\tilde\pi_s}\rangle\, \dd x$ needed for the analysis of discrete-time SALD. We omit its proof here, since it follows directly from the general lemma proved later in Lemma \ref{lem:frozen_delta_cross_lip} by setting $c \equiv 0$ and $\sigma_\eta(t)=\sqrt{2}$.

For \(t\in[t_k,t_{k+1}]\), define the frozen-field error
\begin{equation}\label{eq:discrete_delta_def}
    \delta_{\pi_t}(x)
    := \nabla\log\pi_t(x)
    -
    \E\left[ \nabla\log\pi_{t_k}(X_k^\eta) \middle|~\hat X_s=x \right].
\end{equation}

 Set $\hat\rho_s:=\Law(\hat X_s)$.
\begin{lemma}\label{lem:frozen_delta_cross_lip_sald}
Let $A_s:=\nabla\log\frac{\hat\rho_s}{\tilde\pi_s}$.
We make the following assumptions.
\begin{enumerate}[itemsep=0mm,leftmargin=5mm,topsep=0mm] 
    \item There exist constant $L_{\pi,\mathrm{space}}>0$ such that for all \(x,y\in\R^d\),~$t \in [0,T]$
    \begin{align*}
    \|\nabla\log\pi_{t}(x)-\nabla\log\pi_{t}(y)\|
    \le
    L_{\pi,\mathrm{space}}\|x-y\|.
    \end{align*}
    \item There exist a measurable function \(M\) and constants $L_{\pi,\mathrm{time}}>0$ such that for every $s < s', (s,s'\in [0,S])$, and every \(x\in\R^d\),
    \begin{align*}
    \|\nabla\log\pi_{t(s')}(x)-\nabla\log\pi_{t(s)}(x)\|
    \le
    L_{\pi,\mathrm{time}}(s'-s)(1+M(x)).
    \label{eq:score_time_regularity_lip}
    \end{align*}

    \item There exists \(\alpha_0'>0\) such that
    \[
    \mathfrak{E}_{\alpha_0'}(\pi_t,\nabla\log\pi_t)<+\infty,
    \qquad
    \mathfrak{E}_{\alpha_0'}(\pi_t,1+M)<+\infty,
    \qquad \forall t\in[0,T].
    \]

    \item We assume $\eta^2 L_{\pi,\mathrm{space}}^2 < \frac{1}{8}$.
\end{enumerate}
Then, for every \(\alpha'\in(0,\alpha_0']\) and every \(s\in[s_k,s_{k+1}]\),
\begin{equation}
\begin{aligned}
-\int \hat\rho_s\langle \delta_{\pi_{t(s)}}(x),A_s\rangle\, \dd x
&\le
\frac{1}{4}\FI(\hat\rho_s\|\tilde\pi_s)
+  2\eta^2\alpha'^{-1} \Gamma(t(s))
\KL(\hat\rho_s\|\tilde\pi_s) +2\eta \Delta(t(s)).
\end{aligned}
\end{equation}
Here,
\begin{align*}
    &\Gamma(t) 
    := 2L_{\pi,\mathrm{time}}^2
    + 64 L_{\pi,\mathrm{space}}^2
    \Bigl(
    \dot{s}(t)^{-2} + 1 + \eta^2L_{\pi,\mathrm{time}}^2  \Bigr), \\
    &\Delta(t) := 
    64\eta L_{\pi,\mathrm{space}}^2 
    \mathfrak{E}_{\alpha'}(\pi_{t},\nabla\log\pi_{t})
    + 2\eta L_{\pi,\mathrm{time}}^2 (1+32\eta^2 L_{\pi,\mathrm{space}}^2)\mathfrak{E}_{\alpha'}(\pi_{t},1+M)
    + 16 d L_{\pi,\mathrm{space}}^2.
\end{align*}
\end{lemma}

We here present the proof of Theorem~\ref{thm:forward-KL-discrete}, based on Lemma \ref{lem:frozen_delta_cross_lip_sald}. 
 
\begin{proof}[Proof of Theorem~\ref{thm:forward-KL-discrete}]
    Fix \(k\in\{0,\dots,K-1\}\). For \(s\in[s_k,s_{k+1}]\), $\hat\rho_{s_k}=\rho_k^\eta,~\hat\rho_{s_{k+1}}=\rho_{k+1}^\eta$.

    Differentiating \(\KL(\hat\rho_s\|\tilde\pi_s)\) with respect to \(s\), we get
    \begin{equation}\label{eq:KL-derivative-0-discrete}
        \frac{\dd}{\dd s}\KL(\hat\rho_s\|\tilde\pi_s)
        =
        \int \partial_s\hat\rho_s\log\frac{\hat\rho_s}{\tilde\pi_s}\,\dd x
        -
        \int \frac{\hat\rho_s}{\tilde\pi_s}\partial_s\tilde\pi_s\,\dd x,
    \end{equation}
    since \(\int \partial_s\hat\rho_s\,\dd x=0\).

    We define    
    \[
        \bar b_{k,s}(x)
        :=
        \E\!\left[
        \nabla\log\pi_{t_k}(X_k^\eta)\,\middle|\,\hat X_s=x
        \right],
    \]

    Using the Fokker--Planck equation associated with~\eqref{eq:frozen_interp_terminal_disc_prop_additive_final},
    \[
        \partial_s\hat\rho_s
        =
        -\nabla\cdot(\hat\rho_s\bar b_{k,s})
        +\Delta \hat\rho_s,
        \qquad s\in[s_k,s_{k+1}],
    \]
    together with
    \[
        \Delta \hat\rho_s
        =
        \nabla\cdot(\hat\rho_s\nabla\log\hat\rho_s)
        =
        \nabla\cdot\!\left(\hat\rho_s\nabla\log\frac{\hat\rho_s}{\tilde\pi_s}\right)
        +
        \nabla\cdot(\hat\rho_s\nabla\log\tilde\pi_s),
    \]
    we obtain
    \begin{align*}
        \partial_s\hat\rho_s
        &=
        -\nabla\cdot(\hat\rho_s\bar b_{k,s})
        +\nabla\cdot\!\left(\hat\rho_s\nabla\log\frac{\hat\rho_s}{\tilde\pi_s}\right)
        +\nabla\cdot(\hat\rho_s\nabla\log\tilde\pi_s) \\
        &=
        \nabla\cdot\!\left(\hat\rho_s\nabla\log\frac{\hat\rho_s}{\tilde\pi_s}\right)
        +\nabla\cdot\!\left(\hat\rho_s(\nabla\log\tilde\pi_s-\bar b_{k,s})\right),
        \qquad s\in[s_k,s_{k+1}].
    \end{align*}
    Hence the first term of Eq.~\eqref{eq:KL-derivative-0-discrete} becomes
    \begin{align}\label{eq:KL-derivative-1-discrete}
        \int \partial_s\hat\rho_s\log\frac{\hat\rho_s}{\tilde\pi_s}\,\dd x
        &=
        \int \nabla\cdot\!\left(\hat\rho_s\nabla\log\frac{\hat\rho_s}{\tilde\pi_s}\right)
        \log\frac{\hat\rho_s}{\tilde\pi_s}\,\dd x
        \notag\\
        &\quad+
        \int \nabla\cdot\!\left(\hat\rho_s(\nabla\log\tilde\pi_s-\bar b_{k,s})\right)
        \log\frac{\hat\rho_s}{\tilde\pi_s}\,\dd x
        \notag\\
        &=
        -\int \hat\rho_s\left\|\nabla\log\frac{\hat\rho_s}{\tilde\pi_s}\right\|^2\,\dd x
        -\int \hat\rho_s
        \Bigl\langle
        \nabla\log\tilde\pi_s-\bar b_{k,s},
        \nabla\log\frac{\hat\rho_s}{\tilde\pi_s}
        \Bigr\rangle\,\dd x
        \notag\\
        &=
        -\FI(\hat\rho_s\|\tilde\pi_s)
        -\int \hat\rho_s
        \Bigl\langle
        \nabla\log\tilde\pi_s-\bar b_{k,s},
        \nabla\log\frac{\hat\rho_s}{\tilde\pi_s}
        \Bigr\rangle\,\dd x.
    \end{align}

    Let \(v_t\) be a velocity field that generate $\pi_t$ satisfying $\partial_t\pi_t+\nabla\cdot(v_t\pi_t)=0$, $t\in[0,T]$.
    As in the continuous-time proof, \(\tilde v_s:=\dot t(s)\,v_{t(s)}\) generates \(\tilde\pi_s\), namely $\partial_s\tilde\pi_s+\nabla\cdot(\tilde v_s\tilde\pi_s)=0$.

    Therefore, the second term of Eq.~\eqref{eq:KL-derivative-0-discrete} can be evaluated as
    \begin{align}\label{eq:KL-derivative-2-discrete}
       - \int \frac{\hat\rho_s}{\tilde\pi_s}\partial_s \tilde\pi_s \dd x
       &= \int \frac{\hat\rho_s}{\tilde\pi_s} \nabla\cdot (\tilde v_s \tilde\pi_s) \dd x \notag \\
       &= - \int \nabla \frac{\hat\rho_s}{\tilde\pi_s}^{\!\top} (\tilde v_s \tilde\pi_s) \dd x \notag \\
       &= - \int \hat\rho_s
       \Bigl\langle
       \nabla \log \frac{\hat\rho_s}{\tilde\pi_s},
       \tilde v_s
       \Bigr\rangle
       \dd x.
    \end{align}

    Combining Eq.~\eqref{eq:KL-derivative-0-discrete},
    \eqref{eq:KL-derivative-1-discrete}, and
    \eqref{eq:KL-derivative-2-discrete}, we obtain
    \begin{align}\label{eq:KL-derivative-3-discrete}
        \frac{\dd}{\dd s}\KL(\hat\rho_s\|\tilde\pi_s)
        &=
        -\FI(\hat\rho_s\|\tilde\pi_s)
        -\int \hat\rho_s
        \Bigl\langle
        \nabla \log \tilde\pi_s-\bar b_{k,s},
        \nabla \log \frac{\hat\rho_s}{\tilde\pi_s}
        \Bigr\rangle
        \dd x
        \notag\\
        &\quad
        -\int \hat\rho_s
        \Bigl\langle
        \nabla \log \frac{\hat\rho_s}{\tilde\pi_s},
        \tilde v_s
        \Bigr\rangle
        \dd x.
    \end{align}

    By Lemma \ref{lem:frozen_delta_cross_lip_sald} as well as Cauchy--Schwarz and Young's inequality,
    \begin{align}
        -\int \hat\rho_s
        \Bigl\langle
        \nabla \log \tilde\pi_s-\bar b_{k,s},
        \nabla \log \frac{\hat\rho_s}{\tilde\pi_s}
        \Bigr\rangle
        \dd x
        &\le
        \frac{1}{4} \FI(\hat\rho_s\|\tilde\pi_s)
        +
        2\eta^2\alpha'^{-1}\Gamma(t(s)) \KL(\hat\rho_s\|\tilde\pi_s) 
        + 2\eta \Delta(t(s)),
        \label{eq:frozen-cross-bound-discrete}
        \\
        -\int \hat\rho_s
        \Bigl\langle
        \nabla \log \frac{\hat\rho_s}{\tilde\pi_s},
        \tilde v_s
        \Bigr\rangle
        \dd x
        &\le
        \frac14 \FI(\hat\rho_s\|\tilde\pi_s)
        +
        \|\tilde v_s\|_{L^2(\hat\rho_s)}^2.
        \label{eq:moving-cross-bound-discrete}
    \end{align}
    Substituting these estimates into \eqref{eq:KL-derivative-3-discrete}, we get
    \begin{align}\label{eq:KL-derivative-5-discrete}
        \frac{\dd}{\dd s}\KL(\hat\rho_s\|\tilde\pi_s)
        &\le
        -\frac12 \FI(\hat\rho_s\|\tilde\pi_s)
        +\|\tilde v_s\|_{L^2(\hat\rho_s)}^2
        +2\eta^2\alpha'^{-1} \Gamma(t(s)) \KL(\hat\rho_s\|\tilde\pi_s) +2 \eta\Delta(t(s)) \notag \\
        &\le
        -\cLSI{t(s)}\KL(\hat\rho_s\|\tilde\pi_s)
        +\|\tilde v_s\|_{L^2(\hat\rho_s)}^2
        +2 \eta^2\alpha'^{-1}\Gamma(t(s)) \KL(\hat\rho_s\|\tilde\pi_s) +2 \eta \Delta(t(s)),
    \end{align}
    where we applied LSI on $\tilde{\pi}_s$.

    We now estimate $\|\tilde v_s\|_{L^2(\hat\rho_s)}^2$ by Lemma~\ref{lem:dv_variation}. Since $\tilde v_s=\dot t(s) v_{t(s)}$, Lemma~\ref{lem:dv_variation} with $\nu=\hat\rho_s,~\mu=\tilde\pi_s,~Z=\alpha\|v_{t(s)}\|^2$
    yields
    \begin{align}\label{eq:dv-v-term-discrete}
        \|\tilde v_s\|_{L^2(\hat\rho_s)}^2
        &=
        \dot t(s)^2\|v_{t(s)}\|_{L^2(\hat\rho_s)}^2
        \notag\\
        &\le
        \dot t(s)^2
        \left(
        \frac{1}{\alpha}\KL(\hat\rho_s\|\tilde\pi_s)
        +\frac{1}{\alpha}\log \E_{\tilde\pi_s}\!\left[\exp\!\bigl(\alpha\|v_{t(s)}\|^2\bigr)\right]
        \right)
        \notag\\
        &=
        \dot t(s)^2
        \left(
        \frac{1}{\alpha}\KL(\hat\rho_s\|\tilde\pi_s)
        +\mathfrak{E}_{\alpha}(\pi_{t(s)}, v_{t(s)})
        \right).
    \end{align}

    Substituting \eqref{eq:dv-v-term-discrete} into
    \eqref{eq:KL-derivative-5-discrete}, we obtain
    \begin{align}\label{eq:KL-derivative-6-discrete}
        \frac{\dd}{\dd s}\KL(\hat\rho_s\|\tilde\pi_s)
        &\le
        -\left(
        \cLSI{t(s)}
        -\dot t(s)^2 \alpha^{-1}
        - 2\eta^2\alpha'^{-1} \Gamma(t(s))
        \right)\KL(\hat\rho_s\|\tilde\pi_s)
        \notag\\
        &\quad+
        \dot t(s)^2 \mathfrak{E}_{\alpha}(\pi_{t(s)},v_{t(s)})
        + 2\eta\Delta(t(s)).
    \end{align}

    Since
    \[
        \frac{\dd}{\dd t}\KL(\hat\rho_{s(t)}\|\pi_t)
        = \frac{\dd}{\dd t}\KL(\hat\rho_{s(t)}\|\tilde{\pi}_{s(t)})
        = \dot s(t) \frac{\dd}{\dd s}\KL(\hat\rho_s\|\tilde\pi_s)\Big|_{s=s(t)},
    \]
    Eq.~\eqref{eq:KL-derivative-6-discrete} implies, for \(t\in[t_k,t_{k+1}]\),
    \begin{align}\label{eq:KL-derivative-7-discrete}
        \frac{\dd}{\dd t}\KL(\hat\rho_{s(t)}\|\pi_t)
        &\le
        -\left(
        \dot s(t)\cLSI{t}
        -\dot s(t)^{-1}\alpha^{-1}
        - 2\dot s(t) \eta^2\alpha'^{-1}\Gamma(t)
        \right)\KL(\hat\rho_{s(t)}\|\pi_t)
        \notag\\
        &\quad+
        \dot s(t)^{-1}\mathfrak{E}_{\alpha}(\pi_t,v_t)
        +
         2 \dot s(t)\Delta(t)\eta.
    \end{align}

    Applying Lemma~\ref{lem:gronwall}, we get

    \begin{align*}
        \KL(\rho_K^\eta\|\pi_T)
        &\le
        \exp\!\left(
        -\int_0^{T}
        \Bigl(
        \dot s(t)\cLSI{t}
        -\dot s(t)^{-1}\alpha^{-1}
        - 
        2\dot s(t)\eta^2\alpha'^{-1}\Gamma(t)
        \Bigr)\dd t
        \right)
        \KL(\rho_0\|\pi_0)
        \\
        &\quad+
        \int_0^{T}
        \exp\!\left(
        -\int_t^{T}
        \Bigl(
        \dot s(u)\cLSI{u}
        -\dot s(u)^{-1}\alpha^{-1}
        - 2\dot s(u)\eta^2\alpha'^{-1}\Gamma(u)
        \Bigr)\dd u
        \right)
        \\
        &\qquad\qquad\cdot
        \Bigl(
        \dot s(t)^{-1}\mathfrak{E}_{\alpha}(\pi_t,v_t)
        +
         2\dot s(t)\Delta(t)\eta
        \Bigr)\dd t.
    \end{align*}
\end{proof}

\section{Analysis of Velocity-Aware SALD}\label{app:convergence_velocityAwareSALD}
In this section, we provide the convergence analysis of VA-SALD for guided generation.

\subsection{Analysis of Continuous-Time Velocity-Aware SALD}\label{sec:continuous_time_VA_SALD}
We provide the convergence analysis of VA-SALD \eqref{eq:SALD_Ito} with the arbitrary reference vector field $c_t$.


\begin{proof}[Proof of Theorem~\ref{thm:general-moving-target-SALD}]
Differentiating \(\KL(\rho_s\|\pi_{t(s)})\) in time \(s\), we get
\begin{equation}
\label{eq:general_KL_derivative_0}
\frac{\dd}{\dd s}\KL(\rho_s\|\pi_{t(s)})
=
\int \partial_s \rho_s \log\frac{\rho_s}{\pi_{t(s)}}\,\dd x
-
\int \frac{\rho_s}{\pi_{t(s)}}\partial_s \pi_{t(s)}\,\dd x
\end{equation}
since \(\int \partial_s \rho_s\,\dd x=0\).

Using the SDE \eqref{eq:SALD_Ito}, the law \(\rho_s\) satisfies the Fokker--Planck equation
\begin{equation}
\label{eq:general_moving_target_FP}
\partial_s \rho_s
=
-\nabla\cdot\left(\dot t(s)c_{t(s)}\,\rho_s\right)
+
\frac{\sigma_{t(s)}^2}{2}
\nabla\cdot\left(
\rho_s\nabla\log\frac{\rho_s}{\pi_{t(s)}}
\right).
\end{equation}
Let $A_s:=\nabla\log\frac{\rho_s}{\pi_{t(s)}}$.
Then the first term in \eqref{eq:general_KL_derivative_0} becomes
\begin{align}
\label{eq:general_KL_derivative_1}
\int \partial_s \rho_s \log\frac{\rho_s}{\pi_{t(s)}}\,\dd x
&=
\int
\left[
-\nabla\cdot\left(\dot t(s)c_{t(s)}\,\rho_s\right)
+
\frac{\sigma_{t(s)}^2}{2}\nabla\cdot(\rho_s A_s)
\right]
\log\frac{\rho_s}{\pi_{t(s)}}\,\dd x
\nonumber\\
&=
\dot t(s)\int \rho_s\, c_{t(s)}^\top A_s\,\dd x
-
\frac{\sigma_{t(s)}^2}{2}\int \rho_s \|A_s\|^2\,\dd x
\nonumber\\
&=
\dot t(s)\int \rho_s\, c_{t(s)}^\top A_s\,\dd x
-
\frac{\sigma_{t(s)}^2}{2}\FI(\rho_s\|\pi_{t(s)}).
\end{align}

Next, since \(v_t\) generates \(\pi_t\), the rescaled velocity $\tilde v_s:=\dot t(s)\,v_{t(s)}$ generates \(\pi_{t(s)}\), because
\[
\partial_s \pi_{t(s)}
=
\dot t(s)\partial_t \pi_t\big|_{t=t(s)}
=
-\dot t(s)\nabla\cdot(v_{t(s)}\pi_{t(s)})
=
-\nabla\cdot(\tilde v_s\pi_{t(s)}).
\]
Therefore the second term in \eqref{eq:general_KL_derivative_0} is
\begin{align}
\label{eq:general_KL_derivative_2}
-\int \frac{\rho_s}{\pi_{t(s)}}\partial_s \pi_{t(s)}\,\dd x
&=
\int \frac{\rho_s}{\pi_{t(s)}} \nabla\cdot(\tilde v_s\pi_{t(s)})\,\dd x
\nonumber\\
&=
-\int \rho_s\, \tilde v_s^\top A_s\,\dd x.
\end{align}

Combining \eqref{eq:general_KL_derivative_0}, \eqref{eq:general_KL_derivative_1}, and \eqref{eq:general_KL_derivative_2}, we obtain
\[
\frac{\dd}{\dd s}\KL(\rho_s\|\pi_{t(s)})
=
-\frac{\sigma_{t(s)}^2}{2}\FI(\rho_s\|\pi_{t(s)})
-
\dot t(s)\int \rho_s
\left(
v_{t(s)}-c_{t(s)}
\right)^\top A_s\,\dd x.
\]
That is,
\[
\frac{\dd}{\dd s}\KL(\rho_s\|\pi_{t(s)})
=
-\frac{\sigma_{t(s)}^2}{2}\FI(\rho_s\|\pi_{t(s)})
-
\dot t(s)\int \rho_s\, m_{t(s)}^\top A_s\,\dd x.
\]
By H\"older's inequality and \(ab\le \frac{\varepsilon}{2}a^2+\frac{1}{2\varepsilon}b^2\) with \(\varepsilon=2\dot{t}(s) / \sigma_{t(s)}^2\), 
\begin{align}
\label{eq:general_KL_derivative_3}
\frac{\dd}{\dd s}\KL(\rho_s\|\pi_{t(s)})
&\le
-\frac{\sigma_{t(s)}^2}{4}\FI(\rho_s\|\pi_{t(s)})
+
\sigma_{t(s)}^{-2}\dot t(s)^2
\int \rho_s \|m_{t(s)}\|^2\,\dd x.
\end{align}

Now define $K(t):=\KL(\rho_{s(t)}\|\pi_t) = \KL(\rho_{s(t)} \| \tilde{\pi}_{s(t)})$.
Multiplying \eqref{eq:general_KL_derivative_3} by \(\dot s(t)\), and using \(\dot t(s(t))=\dot s(t)^{-1}\), we obtain
\begin{align*}
\frac{\dd}{\dd t}K(t)
&\le
-\frac{\sigma_t^2}{4}\dot s(t)\FI(\rho_{s(t)}\|\pi_t)
+
\sigma_t^{-2}\dot s(t)^{-1}
\|m_t\|_{L_2(\rho_{s(t)})}^2.
\end{align*}
Combining this with the LSI for \(\pi_t\), we get
\begin{align*}
\frac{\dd}{\dd t}K(t)
&\le
-\frac{\sigma_t^2}{2}\dot s(t)\cLSI{t}\,K(t)
+
\sigma_t^{-2}\dot s(t)^{-1}
\|m_t\|_{L_2(\rho_{s(t)})}^2.
\end{align*}

The last term can be evaluated by Lemma~\ref{lem:dv_variation}:
\[
\|m_t\|_{L_2(\rho_{s(t)})}^2
\le
\frac{1}{\alpha}\KL(\rho_{s(t)}\|\pi_t)
+
\frac{1}{\alpha}\log\E_{\pi_t}\!\left[
\exp\!\bigl(\alpha \|m_t\|^2\bigr)
\right]
=
\frac{1}{\alpha}K(t)+\mathfrak{E}_{\alpha}(\pi_t,m_t).
\]
Therefore
\[
\frac{\dd}{\dd t}K(t)
\le
-\left(
\frac{\sigma_t^2}{2}\dot s(t)\cLSI{t}
-
\sigma_t^{-2}\dot s(t)^{-1}\alpha^{-1}
\right)K(t)
+
\sigma_t^{-2}\dot s(t)^{-1}\mathfrak{E}_{\alpha}(\pi_t,m_t).
\]
Applying Lemma~\ref{lem:gronwall}, we obtain
\begin{align*}
\KL(\rho_S\|\pi_T)
&\le
\exp\left(
-\int_0^T
\left(
\frac{\sigma_t^2}{2}\dot s(t)\cLSI{t}
-
\sigma_t^{-2}\dot s(t)^{-1}\alpha^{-1}
\right)\dd t
\right)\KL(\rho_0\|\pi_0)
\\
&\quad+
\int_0^T
\exp\left(
-\int_t^T
\left(
\frac{\sigma_u^2}{2}\dot s(u)\cLSI{u}
-
\sigma_u^{-2}\dot s(u)^{-1}\alpha^{-1}
\right)\dd u
\right)
\sigma_t^{-2}\dot s(t)^{-1}\mathfrak{E}_{\alpha}(\pi_t, m_t) \dd t.
\end{align*}
This inequality yields Eq.~\eqref{eq:general_moving_target_KL_bound}.

\end{proof}

\subsection{Analysis of Discrete-Time Velocity-Aware SALD}\label{sec:discrete_time_VA_SALD}

In this section, we present the convergence analysis of VA-SALD with linear slowdown.

The Euler--Maruyama discretization of the general form of VA-SALD \eqref{eq:SALD_Ito} is defined as follows.
\paragraph{Discrete-time general VA-SALD.} $X_0^\eta\sim \rho_0,~\xi_k\stackrel{\mathrm{i.i.d.}}{\sim}\mathcal N(0,I_d)$,
\begin{equation} \label{eq:SALD_general_EM}
 \begin{aligned}
 X_{k+1}^\eta
 &=
 X_k^\eta
 +
 \eta\left(
 \dot t_k c_{t_k}(X_k^\eta)
 +
 \frac{\sigma_{t_k}^2}{2}\nabla\log \pi_{t_k}(X_k^\eta) \right)
 +
 \sigma_{t_k}\sqrt{\eta}\,\xi_k,
 \end{aligned}
\end{equation}
where $t_k = t(k\eta)$, and $s_k = k\eta$.

Next, we introduce the continuous interpolation of the above step as follows. For \(t\in[0,T]\), let \(k(t)\in\{0,\dots,K-1\}\) be the unique index such that
\(t\in[t_{k(t)},t_{k(t)+1})\), and define the piecewise constant diffusion level
\[
\sigma_\eta(t):=\sigma_{t_{k(t)}}.
\]

Then, for \(s\in[s_k,s_{k+1}]\), we define
\begin{equation}
\label{eq:general_moving_target_SALD_frozen_interp}
\hat X_s
:=
X_k^\eta
+
(s-s_k)\Bigl(
\dot t_k\,c_{t_k}(X_k^\eta)
+
\frac{\sigma_{\eta}^2}{2}\nabla\log\pi_{t_k}(X_k^\eta)
\Bigr)
+
\sigma_{\eta}(W_s-W_{s_k}).
\end{equation}

For \(t\in[t_k,t_{k+1}]\), with \(s=s(t)\in[s_k,s_{k+1}]\), define
\begin{align*}
    \varphi_{t}(x)
    &:= \dot t(s)\,c_t(x)
    +\frac{\sigma_{\eta}^2}{2}\nabla\log\pi_t(x), \\
\end{align*}
and define the frozen-field error
\begin{equation}
\label{eq:general_discrete_delta_def}
\delta_{\pi_t}^{\mathrm{VA}}(x)
:=
\dot t(s)\,c_t(x)
+
\frac{\sigma_{\eta}^2}{2}\nabla\log\pi_t(x)
-
\E\!\left[
\dot t_k\,c_{t_k}(X_k^\eta)
+
\frac{\sigma_{\eta}^2}{2}\nabla\log\pi_{t_k}(X_k^\eta)
\,\middle|\,\hat X_s=x
\right].
\end{equation}

That is, $\delta_{\pi_t}^{\mathrm{VA}}(x)=\varphi_{t}(x)-\E\left[ \varphi_{t_k}(X_k^\eta) \middle|~\hat X_s=x\right]$.

The next lemma provides a bound on $\int \hat\rho_s\langle \delta_{\pi_t}^{\mathrm{VA}}(x),\nabla\log\frac{\hat\rho_s}{\tilde\pi_s}\rangle\, \dd x$.
Assume we work in the linear slowdown case, namely \(\dot t(s)\) is constant on \([0,S]\). Set $\hat\rho_s:=\Law(\hat X_s)$.

\begin{lemma}\label{lem:frozen_delta_cross_lip}
Let $A_s:=\nabla\log\frac{\hat\rho_s}{\tilde\pi_s}$.
We make the following assumptions.
\begin{enumerate}[itemsep=0mm,leftmargin=5mm,topsep=0mm] 
    \item There exist constants \(L_{c,\mathrm{space}},L_{\pi,\mathrm{space}}>0\) such that for all \(x,y\in\R^d\),~$t \in [0,T]$
    \begin{align}
    \|c_{t}(x)-c_{t}(y)\|
    &\le
    L_{c,\mathrm{space}}\|x-y\|,
    \label{eq:c_space_regularity_lip}
    \\
    \|\nabla\log\pi_{t}(x)-\nabla\log\pi_{t}(y)\|
    &\le
    L_{\pi,\mathrm{space}}\|x-y\|.
    \label{eq:score_space_regularity_lip}
    \end{align}

    \item There exist a measurable function \(M\) and constants \(L_{c,\mathrm{time}},L_{\pi,\mathrm{time}}>0\) such that for every $s < s', (s,s'\in [0,S])$, and every \(x\in\R^d\),
    \begin{align}
    \|c_{t(s')}(x)-c_{t(s)}(x)\|
    &\le
    L_{c,\mathrm{time}}(s'-s)(1+M(x)),
    \label{eq:c_time_regularity_lip}
    \\
    \|\nabla\log\pi_{t(s')}(x)-\nabla\log\pi_{t(s)}(x)\|
    &\le
    L_{\pi,\mathrm{time}}(s'-s)(1+M(x)).
    \label{eq:score_time_regularity_lip}
    \end{align}

    \item There exists \(\alpha_0'>0\) such that
    \[
    \mathfrak{E}_{\alpha_0'}(\pi_t,c_t)<+\infty,
    \qquad
    \mathfrak{E}_{\alpha_0'}(\pi_t,\nabla\log\pi_t)<+\infty,
    \qquad
    \mathfrak{E}_{\alpha_0'}(\pi_t,1+M)<+\infty,
    \qquad \forall t\in[0,T].
    \]

    \item We assume
    \[
    4\eta^2\Bigl(\dot t(s)L_{c,\mathrm{space}}+\frac{\sigma_{\eta}(t(s))^2}{2}L_{\pi,\mathrm{space}}\Bigr)^2<\frac12.
    \]
\end{enumerate}
Then, for every \(\alpha'\in(0,\alpha_0']\) and every \(s\in[s_k,s_{k+1}]\),
\begin{equation}
\begin{aligned}
-\int \hat\rho_s\langle \delta_{\pi_{t(s)}}^{\mathrm{VA}}(x),A_s\rangle\, \dd x
&\le
\frac{\sigma_{\eta}(t(s))^2}{8}\FI(\hat\rho_s\|\tilde\pi_s)
+  2\eta^2\alpha'^{-1} \Gamma(t(s))
\KL(\hat\rho_s\|\tilde\pi_s) +2\eta \Delta(t(s)).
\end{aligned}\label{eq:frozen_delta_cross_bound}
\end{equation}
Here
\begin{align*}
    \Gamma(t) 
    &:= \sigma_{\eta}(t)^{-2}\Biggl\{ 
    4\dot{s}(t)^{-2}L_{c,\mathrm{time}}^2+\sigma_{\eta} (t)^4L_{\pi,\mathrm{time}}^2\notag\\
    &+8\left(4\dot{s}(t)^{-2}L_{c,\mathrm{space}}^2+\sigma_{\eta} (t)^4L_{\pi,\mathrm{space}}^2 \right) \Bigl(
    4\dot{s}(t)^{-2}+\sigma_{\eta}(t)^4+\eta^2\bigl(4\dot{s}(t)^{-2}L_{c,\mathrm{time}}^2+\sigma_{\eta}(t)^4L_{\pi,\mathrm{time}}^2\bigr)
    \Bigr) \Biggr\},
\end{align*}
and
\begin{align*}
    \Delta(t) &:= 
    \eta\sigma_{\eta}(t)^{-2}\Biggl\{ 8\Bigl(
    4\dot{s}(t)^{-2}L_{c,\mathrm{space}}^2+\sigma_{\eta}(t)^4L_{\pi,\mathrm{space}}^2
    \Bigr) 
    \Bigl( 4 \dot{s}(t)^{-2} \,\mathfrak E_{\alpha'}(\pi_{t},c_{t}) 
    +\sigma_{\eta}(t)^4\,\mathfrak E_{\alpha'}(\pi_{t},\nabla\log\pi_{t}) \Bigr) \notag \\
    &+ \Bigl(4\dot{s}(t)^{-2}L_{c,\mathrm{time}}^2+\sigma_{\eta} (t)^4L_{\pi,\mathrm{time}}^2\Bigr) \Bigl( 
    1 +8 \eta^2 \Bigl(
    4\dot{s}(t)^{-2}L_{c,\mathrm{space}}^2+\sigma_{\eta}(t)^4L_{\pi,\mathrm{space}}^2
    \Bigr)  \Bigr)
    \mathfrak E_{\alpha'}(\pi_{t},1+M) \Biggr\} \notag \\
    &+ 4 d \Bigl(
    4\dot{s}(t)^{-2}L_{c,\mathrm{space}}^2+\sigma_{\eta}(t)^4L_{\pi,\mathrm{space}}^2
    \Bigr). 
\end{align*}
\end{lemma}

\begin{proof}

Consider $s \in [s_k, s_{k+1})$.
Let the joint distribution of \(X_k^\eta\) and \(\hat X_s\) be defined as \(\hat \rho_{t_k,s}\). Then,
\begin{align}
&-\int \hat\rho_s\langle \delta_{\pi_{t(s)}}^{\mathrm{VA}}(x),A_s\rangle\, \dd x \\
&=
-\int \hat \rho_{t_k,s}(x^\eta_k, \hat x_s)
\left( \varphi_{t(s)}(\hat x_s) - \varphi_{t_k}(x^\eta_k)\right)^\top A_s \dd \hat x_s \dd x^\eta_k \notag\\
&\leq \frac{\sigma_{\eta}(t(s))^2}{8}\mathrm{FI}(\hat\rho_s \| \tilde \pi_s) 
+ 2\sigma_{\eta}(t(s))^{-2}\int  \hat \rho_{t_k,s}(x^\eta_k, \hat x_s)
 \| \varphi_{t(s)}(\hat x_s) - \varphi_{t_k}(x^\eta_k)\|^2\dd \hat x_s \dd x^\eta_k, \label{eq:delta-bound}
\end{align}
where the inequality is obtained by the Cauchy--Schwarz and Young's inequalities.

We proceed to bound \(\| \varphi_{t(s)}(\hat x_s) - \varphi_{t_k}(x^\eta_k)\|^2\):
\begin{align*}
\| \varphi_{t(s)}(\hat x_s) - \varphi_{t_k}(x^\eta_k)\|^2
&\leq 2 \| \varphi_{t(s)}(\hat x_s) - \varphi_{t_k}(\hat x_s)\|^2 + 2 \| \varphi_{t_k}(\hat x_s) - \varphi_{t_k}(x^\eta_k)\|^2.
\end{align*}
For the time-varying term,
\begin{align*}
2 \| \varphi_{t(s)}(\hat x_s) - \varphi_{t_k}(\hat x_s)\|^2
&\le \eta^2\Bigl(4\dot t(s)^2L_{c,\mathrm{time}}^2+\sigma_{\eta}(t_k)^4L_{\pi,\mathrm{time}}^2\Bigr)
(1+M(\hat x_s))^2.
\end{align*}
For the space-varying term,
\begin{align*}
2 \| \varphi_{t_k}(\hat x_s) - \varphi_{t_k}(x^\eta_k)\|^2
&\le
\Bigl(
4\dot t(s)^2L_{c,\mathrm{space}}^2+\sigma_{\eta}(t_k)^4L_{\pi,\mathrm{space}}^2
\Bigr)
\|\hat x_s-x^\eta_k\|^2.
\end{align*}
Hence,
\begin{align}
&\int \hat \rho_{t_k,s}(x^\eta_k,\hat x_s)
\| \varphi_{t(s)}(\hat x_s) - \varphi_{t_k}(x^\eta_k)\|^2
\, \dd \hat x_s \dd x^\eta_k \notag\\
&\le
\eta^2\Bigl(4\dot t(s)^2L_{c,\mathrm{time}}^2+\sigma_{\eta}(t_k)^4L_{\pi,\mathrm{time}}^2\Bigr)
\|1+M\|_{L_2(\hat{\rho}_s)}^2
\notag\\
&+\Bigl(
4\dot t(s)^2L_{c,\mathrm{space}}^2+\sigma_{\eta}(t_k)^4L_{\pi,\mathrm{space}}^2
\Bigr)
\E_{\hat \rho_{t_k,s}}\!\left[\|\hat X_s-X_k^\eta\|^2\right].
\label{eq:delta_l2_by_increment_revised}
\end{align}

We now bound the quantity \(\|\hat X_s - X^\eta_k\|^2\) under the law \(\hat \rho_{t_k,s}\).
Recall that \(\hat X_s\) is the continuous EM interpolation \eqref{eq:general_moving_target_SALD_frozen_interp}:
\[
\hat X_s
=
X_k^\eta
+
(s-s_k)\varphi_{t_k}(X^\eta_k)
+
\sigma_{\eta}(t(s))\sqrt{s - s_k}\;\xi, \qquad \xi \sim \cN(0,I).
\]
Thus, 
\begin{align}
\E_{\hat \rho_{t_k,s}}[\|\hat X_s - X^\eta_k\|^2]
&\le 2 (s-s_k)^2 \E\| \varphi_{t_k}(X^\eta_k)\|^2 + 2\sigma_{\eta}(t)^2(s-s_k) \E[\|\xi\|^2]
\notag\\
&\le 2 \eta^2 \E[\| \varphi_{t_k}(X^\eta_k)\|^2] + 2\sigma_{\eta}(t(s))^2 \eta d.
\label{eq:increment_basic_bound_revised}
\end{align}

Next, by the Lipschitz property of \(\varphi_{t_k}\),
\[
\|\varphi_{t_k}(X_k^\eta)\|^2
\le
2\|\varphi_{t_k}(\hat X_s)\|^2
+
2\Bigl(\dot t(s)L_{c,\mathrm{space}}+\frac{\sigma_{\eta}(t_k)^2}{2}L_{\pi,\mathrm{space}}\Bigr)^2
\|\hat X_s-X_k^\eta\|^2.
\]
Taking expectation and substituting into \eqref{eq:increment_basic_bound_revised}, we obtain
\begin{align*}
\E_{\hat \rho_{t_k,s}}\!\left[\|\hat X_s-X_k^\eta\|^2\right]
&\le
4\eta^2\|\varphi_{t_k}\|_{L^2(\hat\rho_s)}^2
+
4\eta^2\Bigl(\dot t(s)L_{c,\mathrm{space}}+\frac{\sigma_{\eta}(t_k)^2}{2}L_{\pi,\mathrm{space}}\Bigr)^2
\E_{\hat \rho_{t_k,s}}\!\left[\|\hat X_s-X_k^\eta\|^2\right]
\\
&+
2\sigma_{\eta}(t(s))^2\eta d.
\end{align*}
Since $4\eta^2\Bigl(\dot t(s)L_{c,\mathrm{space}}+\frac{\sigma_{\eta}(t(s))^2}{2}L_{\pi,\mathrm{space}}\Bigr)^2<\frac12$, 
we have 
\begin{align}
\E_{\hat \rho_{t_k,s}}\!\left[\|\hat X_s-X_k^\eta\|^2\right]
&\le
8\eta^2\|\varphi_{t_k}\|_{L^2(\hat\rho_s)}^2
+
4\sigma_{\eta}(t)^2\eta d.
\label{eq:increment_resolved_bound}
\end{align}

To estimate \(\|\varphi_{t_k}\|_{L^2(\hat\rho_s)}^2\), we compare \(\varphi_{t_k}\) with \(\varphi_{t(s)}\). By Lipschitz continuity in time,
\begin{align*}
\|\varphi_{t_k}(x)\|^2
&\le
2\|\varphi_{t(s)}(x)\|^2
+
2\|\varphi_{t(s)}(x)-\varphi_{t_k}(x)\|^2 \\
&\le
2\|\varphi_{t(s)}(x)\|^2
+
\eta^2\Bigl(4\dot t(s)^2L_{c,\mathrm{time}}^2+\sigma_{\eta}(t_k)^4L_{\pi,\mathrm{time}}^2\Bigr)(1+M(x))^2.
\end{align*}
Also,
\[
\|\varphi_{t(s)}(x)\|^2
\le
2\dot t(s)^2\|c_{t(s)}(x)\|^2
+
\frac{\sigma_{\eta}(t_k)^4}{2}\|\nabla\log\pi_{t(s)}(x)\|^2.
\]
Hence
\begin{align}
\|\varphi_{t_k}\|_{L^2(\hat\rho_s)}^2
&\le
4\dot t(s)^2 \|c_{t(s)}\|_{L^2(\hat\rho_s)}^2
+
\sigma_{\eta}(t_k)^4\|\nabla \log \pi_{t(s)}\|_{L^2(\hat\rho_s)}^2
\notag\\
&+ \eta^2\Bigl(4\dot t(s)^2L_{c,\mathrm{time}}^2+\sigma_{\eta}(t_k)^4L_{\pi,\mathrm{time}}^2\Bigr)
\|(1+M)\|_{L^2(\hat\rho_s)}^2.
\label{eq:phi_l2_split_revised}
\end{align}

Invoking Lemma~\ref{lem:dv_variation},
\begin{align*}
\|c_{t(s)}\|_{L^2(\hat\rho_s)}^2
&\le \alpha'^{-1}\KL(\hat\rho_s\|\tilde\pi_s)
+\mathfrak E_{\alpha'}(\pi_{t(s)},c_{t(s)}),\\
\|\nabla\log\pi_{t(s)}\|_{L^2(\hat\rho_s)}^2
&\le \alpha'^{-1}\KL(\hat\rho_s\|\tilde\pi_s)
+\mathfrak E_{\alpha'}(\pi_{t(s)},\nabla\log\pi_{t(s)}),\\
\|(1+M)\|_{L^2(\hat\rho_s)}^2
&\le \alpha'^{-1}\KL(\hat\rho_s\|\tilde\pi_s)
+\mathfrak E_{\alpha}(\pi_{t(s)},1+M).
\end{align*}
Substituting into \eqref{eq:phi_l2_split_revised}, we obtain
\begin{align}
\|\varphi_{t_k}\|_{L^2(\hat\rho_s)}^2
&\le
\Bigl(
4\dot t(s)^2+\sigma_{\eta}(t_k)^4+\eta^2\bigl(4\dot t(s)^2L_{c,\mathrm{time}}^2+\sigma_{\eta}(t_k)^4L_{\pi,\mathrm{time}}^2\bigr)
\Bigr)\alpha'^{-1}\KL(\hat\rho_s\|\tilde\pi_s)
\notag\\
&+
4\dot t(s)^2\,\mathfrak E_{\alpha'}(\pi_{t(s)},c_{t(s)})
+
\sigma_{\eta}(t_k)^4\,\mathfrak E_{\alpha'}(\pi_{t(s)},\nabla\log\pi_{t(s)})
\notag\\
&+
\eta^2\Bigl(4\dot t(s)^2L_{c,\mathrm{time}}^2+\sigma_{\eta}(t_k)^4L_{\pi,\mathrm{time}}^2\Bigr)
\mathfrak E_{\alpha'}(\pi_{t(s)},1+M).
\label{eq:phi_l2_bound_revised}
\end{align}

Substituting \eqref{eq:increment_resolved_bound} and \eqref{eq:phi_l2_bound_revised} into \eqref{eq:delta_l2_by_increment_revised}, we obtain
\begin{align}
    &\int \hat \rho_{t_k,s}(x^\eta_k,\hat x_s)
    \| \varphi_{t(s)}(\hat x_s) - \varphi_{t_k}(x^\eta_k)\|^2
    \, \dd \hat x_s \dd x^\eta_k \notag\\
    &\le
    \eta^2\Bigl(4\dot t(s)^2L_{c,\mathrm{time}}^2+\sigma_{\eta} (t_k)^4L_{\pi,\mathrm{time}}^2\Bigr)
    (\alpha'^{-1}\KL(\hat\rho_s\|\tilde\pi_s)
    +\mathfrak E_{\alpha}(\pi_{t(s)},1+M))
    \notag\\
    &+8\eta^2\Bigl(
    4\dot t(s)^2L_{c,\mathrm{space}}^2+\sigma_{\eta}(t_k)^4L_{\pi,\mathrm{space}}^2
    \Bigr)  \notag\\
    &\cdot\Bigl(
    4\dot t(s)^2+\sigma_{\eta}(t_k)^4+\eta^2\bigl(4\dot t(s)^2L_{c,\mathrm{time}}^2+\sigma_{\eta}(t_k)^4L_{\pi,\mathrm{time}}^2\bigr)
    \Bigr)\alpha'^{-1}\KL(\hat\rho_s\|\tilde\pi_s) \notag\\
    &+8\eta^2\Bigl(
    4\dot t(s)^2L_{c,\mathrm{space}}^2+\sigma_{\eta}(t_k)^4L_{\pi,\mathrm{space}}^2
    \Bigr) 
    \cdot \Biggl\{
    4\dot t(s)^2\,\mathfrak E_{\alpha'}(\pi_{t(s)},c_{t(s)}) 
    +\sigma_{\eta}(t_k)^4\,\mathfrak E_{\alpha'}(\pi_{t(s)},\nabla\log\pi_{t(s)})\notag \\
    &+\eta^2\Bigl(4\dot t(s)^2L_{c,\mathrm{time}}^2+\sigma_{\eta}(t_k)^4L_{\pi,\mathrm{time}}^2\Bigr)
    \mathfrak E_{\alpha'}(\pi_{t(s)},1+M) \Biggr\} \notag \\
    &+4\sigma_\eta(t_k)^2 \eta d \Bigl(
    4\dot t(s)^2 L_{c,\mathrm{space}}^2+\sigma_{\eta}(t_k)^4L_{\pi,\mathrm{space}}^2
    \Bigr) 
\end{align}

Combining this inequality and Eq.~\eqref{eq:delta-bound} concludes the proof.
\end{proof}

We here provide the convergence analysis of the discrete-time general VA-SALD \eqref{eq:SALD_general_EM}. 
Set $\rho_k^\eta:=\Law(X_k^\eta),~k=0,\dots,K$.

\begin{theorem}[Discrete-time general VA-SALD]
\label{thm:general-moving-target-SALD-discrete}
Assume the same conditions as in Theorem~\ref{thm:general-moving-target-SALD} and Lemma \ref{lem:frozen_delta_cross_lip}. Consider $\dot t(s) = \dot s(t)^{-1}$ is a constant. Then, for any \(\alpha\in(0,\alpha_0]\) and \(\alpha'\in(0,\alpha_0']\),
\begin{align}\label{eq:general_moving_target_KL_bound_discrete}
&\KL(\rho_K^\eta\|\pi_T) \notag\\
&\le
\exp\!\left(
-\int_0^T
\Bigl(
\frac{\sigma_\eta(t)^2}{2}\dot s(t)\cLSI{t}
-
2\sigma_\eta(t)^{-2}\dot s(t)^{-1}\alpha^{-1}
-
2\dot s(t) \eta^2\alpha'^{-1} \Gamma(t)
\Bigr)\dd t
\right)
\KL(\rho_0\|\pi_0)
\nonumber\\
&\quad+
\int_0^T
\exp\!\left(
-\int_t^T
\Bigl(
\frac{\sigma_\eta(u)^2}{2}\dot s(u)\cLSI{u}
-
2\sigma_\eta(u)^{-2}\dot s(u)^{-1}\alpha^{-1}
-
2\dot s(u) \eta^2\alpha'^{-1} \Gamma(u)
\Bigr)\dd u
\right)
\nonumber\\
&\qquad\qquad \cdot
\Bigl(2
\sigma_\eta(t)^{-2}\dot s(t)^{-1}\mathfrak{E}_{\alpha}(\pi_t,m_t)+ 2\dot s(t) \eta \Delta(t) \Bigr)\dd t,
\end{align}
where \(m_t=v_t-c_t\).
\end{theorem}

\begin{proof}
\textbf{Proof of Theorem~\ref{thm:general-moving-target-SALD-discrete}.}

Fix $k\in\{0,\dots,K-1\}$. For \(s\in[s_k,s_{k+1}]\), let \(\hat\rho_s:=\Law(\hat X_s)\).
Then \(\hat\rho_{s_k}=\rho_k^\eta,~\hat\rho_{s_{k+1}}=\rho_{k+1}^\eta\).

Differentiating \(\KL(\hat\rho_s\|\tilde\pi_s)\) with respect to \(s\), we obtain
\begin{equation}
\label{eq:general_KL_derivative_0_discrete}
\frac{\dd}{\dd s}\KL(\hat\rho_s\|\tilde\pi_s)
=
\int \partial_s\hat\rho_s\log\frac{\hat\rho_s}{\tilde\pi_s}\,\dd x
-
\int \frac{\hat\rho_s}{\tilde\pi_s}\partial_s\tilde\pi_s\,\dd x,
\end{equation}
since \(\int \partial_s\hat\rho_s\,\dd x=0\).

Define the frozen conditional drift field
\[
\bar b_{k,s}(x)
:=
\E\!\left[
\dot t_k\,c_{t_k}(X_k^\eta)
+
\frac{\sigma_{\eta}^2}{2}\nabla\log\pi_{t_k}(X_k^\eta)
\,\middle|\,\hat X_s=x
\right].
\]
By the Fokker--Planck equation associated with
\eqref{eq:general_moving_target_SALD_frozen_interp},
\[
\partial_s\hat\rho_s
=
-\nabla\cdot(\hat\rho_s\bar b_{k,s})
+
\frac{\sigma_{\eta}^2}{2}\Delta\hat\rho_s.
\]
Let
\(
A_s:=\nabla\log\frac{\hat\rho_s}{\tilde\pi_s}.
\)
Using
\[
\Delta\hat\rho_s
=
\nabla\cdot(\hat\rho_s\nabla\log\hat\rho_s)
=
\nabla\cdot(\hat\rho_s A_s)
+
\nabla\cdot(\hat\rho_s\nabla\log\tilde\pi_s),
\]
we rewrite
\begin{align*}
\partial_s\hat\rho_s
&=
-\nabla\cdot(\hat\rho_s\bar b_{k,s})
+
\frac{\sigma_{\eta}^2}{2}\nabla\cdot(\hat\rho_s A_s)
+
\frac{\sigma_{\eta}^2}{2}\nabla\cdot(\hat\rho_s\nabla\log\tilde\pi_s)
\\
&=
\frac{\sigma_{\eta}^2}{2}\nabla\cdot(\hat\rho_s A_s)
+
\nabla\cdot\!\left(
\hat\rho_s\Bigl(
\frac{\sigma_{\eta}^2}{2}\nabla\log\tilde\pi_s-\bar b_{k,s}
\Bigr)
\right).
\end{align*}
Hence
\begin{align}
\label{eq:general_KL_derivative_1_discrete}
\int \partial_s\hat\rho_s\log\frac{\hat\rho_s}{\tilde\pi_s}\,\dd x
&=
-\frac{\sigma_{\eta}^2}{2}\FI(\hat\rho_s\|\tilde\pi_s)
-
\int \hat\rho_s
\Bigl\langle
\frac{\sigma_{\eta}^2}{2}\nabla\log\tilde\pi_s-\bar b_{k,s},
A_s
\Bigr\rangle
\dd x.
\end{align}

Next, since \(v_t\) generates \(\pi_t\), the rescaled velocity
\[
\tilde v_s:=\dot t(s)\,v_{t(s)}
\]
generates \(\tilde\pi_s\). Therefore,
\begin{align}
\label{eq:general_KL_derivative_2_discrete}
-\int \frac{\hat\rho_s}{\tilde\pi_s}\partial_s\tilde\pi_s\,\dd x
&=
-\int \hat\rho_s
\langle \tilde v_s, A_s\rangle\,\dd x.
\end{align}

Combining
\eqref{eq:general_KL_derivative_0_discrete},
\eqref{eq:general_KL_derivative_1_discrete},
and
\eqref{eq:general_KL_derivative_2_discrete},
we get
\begin{align}
\label{eq:general_KL_derivative_3_discrete}
\frac{\dd}{\dd s}\KL(\hat\rho_s\|\tilde\pi_s)
&=
-\frac{\sigma_{\eta}^2}{2}\FI(\hat\rho_s\|\tilde\pi_s)
-
\int \hat\rho_s
\Bigl\langle
\frac{\sigma_{\eta}^2}{2}\nabla\log\tilde\pi_s-\bar b_{k,s}+\tilde v_s,
A_s
\Bigr\rangle
\dd x.
\end{align}

Now observe that, by definition of \(m_t=v_t-c_t\) and
\eqref{eq:general_discrete_delta_def},
for \(t=t(s)\in[t_k,t_{k+1}]\),
\[
\frac{\sigma_{\eta}^2}{2}\nabla\log\tilde\pi_s-\bar b_{k,s}+\tilde v_s
=
\delta_{\pi_{t(s)}}^{\mathrm{VA}}(x)
+
\dot t(s)\,m_{t(s)}.
\]
Hence \eqref{eq:general_KL_derivative_3_discrete} becomes
\begin{align}
\label{eq:general_KL_derivative_4_discrete}
\frac{\dd}{\dd s}\KL(\hat\rho_s\|\tilde\pi_s)
&=
-\frac{\sigma_{\eta}^2}{2}\FI(\hat\rho_s\|\tilde\pi_s)
-
\int \hat\rho_s\langle \delta_{\pi_{t(s)}}^{\mathrm{VA}}(x),A_s\rangle\,\dd x
-
\dot t(s)\int \hat\rho_s\langle m_{t(s)},A_s\rangle\,\dd x.
\end{align}


Applying Young's inequality in the form $ab\le \frac{\varepsilon}{2}a^2+\frac{1}{2\varepsilon}b^2$
with some constants $\varepsilon$ and Lemma \ref{lem:frozen_delta_cross_lip},
\begin{align}
\label{eq:general_cross_m_discrete}
-\dot t(s)\int \hat\rho_s\langle m_{t(s)},A_s\rangle\,\dd x
&\le
\frac{\sigma_{\eta}^2}{8}\FI(\hat\rho_s\|\tilde\pi_s)
+
2\sigma_{\eta}^{-2}\dot t(s)^2\|m_{t(s)}\|_{L^2(\hat\rho_s)}^2,
\\
\label{eq:general_cross_delta_discrete}
-\int \hat\rho_s\langle \delta_{\pi_{t(s)}}^{\mathrm{VA}}(x),A_s\rangle\, \dd x
&\le
\frac{\sigma_{\eta}^2}{8}\FI(\hat\rho_s\|\tilde{\pi}_{s})
+
2\Gamma(t(s))\eta^2\alpha'^{-1}
\KL(\hat\rho_s\|\tilde\pi_s) +
2\Delta(t(s))\eta.
\end{align}

Substituting these bounds into \eqref{eq:general_KL_derivative_4_discrete}, we obtain
\begin{align}
\label{eq:general_KL_derivative_5_discrete}
\frac{\dd}{\dd s}\KL(\hat\rho_s\|\tilde\pi_s)
&\le
-\frac{\sigma_{\eta}^2}{4}\FI(\hat\rho_s\|\tilde\pi_s)
+
2\sigma_{\eta}^{-2}\dot t(s)^2\|m_{t(s)}\|_{L^2(\hat\rho_s)}^2
+
2\Gamma(t(s))\eta^2\alpha'^{-1}
\KL(\hat\rho_s\|\tilde\pi_s) +
2\Delta(t(s))\eta.
\end{align}
Using the LSI for \(\tilde\pi_s=\pi_{t(s)}\),
\[
\FI(\hat\rho_s\|\tilde\pi_s)\ge 2\cLSI{t(s)}\KL(\hat\rho_s\|\tilde\pi_s),
\]
we get
\begin{align}
\label{eq:general_KL_derivative_6_discrete}
\frac{\dd}{\dd s}\KL(\hat\rho_s\|\tilde\pi_s)
&\le
-\frac{\sigma_{\eta}^2}{2}\cLSI{t(s)}\KL(\hat\rho_s\|\tilde\pi_s)
+
2\sigma_{\eta}^{-2}\dot t(s)^2\|m_{t(s)}\|_{L^2(\hat\rho_s)}^2
+
2\Gamma(t(s))\eta^2\alpha'^{-1}
\KL(\hat\rho_s\|\tilde\pi_s) +
2\Delta(t(s))\eta.
\end{align}

We now bound the last two terms by Lemma~\ref{lem:dv_variation}:
\begin{align}
\label{eq:general_dv_m_discrete}
\|m_{t(s)}\|_{L^2(\hat\rho_s)}^2
&\le
\frac{1}{\alpha}\KL(\hat\rho_s\|\tilde\pi_s)
+
\mathfrak{E}_{\alpha}(\pi_{t(s)},m_{t(s)}).
\end{align}

Substituting \eqref{eq:general_dv_m_discrete} and
$\Gamma(t)$ into
\eqref{eq:general_KL_derivative_6_discrete}, we find
\begin{align}
\label{eq:general_KL_derivative_7_discrete}
\frac{\dd}{\dd s}\KL(\hat\rho_s\|\tilde\pi_s)
&\le
-\Bigl(
\frac{\sigma_{\eta}^2}{2}\cLSI{t(s)}
-
2\sigma_{\eta}^{-2}\dot t(s)^2\alpha^{-1}
-
2\Gamma(t(s))\eta^2\alpha'^{-1}
\Bigr)\KL(\hat\rho_s\|\tilde\pi_s)
\nonumber\\
&\quad+
2\sigma_{\eta}^{-2}\dot t(s)^2\mathfrak{E}_{\alpha}(\pi_{t(s)},m_{t(s)})+2\Delta(t(s))\eta
\end{align}

Now define
\[
K(t):=\KL(\hat\rho_{s(t)}\|\pi_t)=\KL(\hat\rho_{s(t)}\|\tilde\pi_{s(t)}).
\]
Since
\[
\frac{\dd}{\dd t}K(t)
=
\dot s(t)\frac{\dd}{\dd s}\KL(\hat\rho_s\|\tilde\pi_s)\Big|_{s=s(t)},
\]
and \(\dot t(s(t))=\dot s(t)^{-1}\), Eq.~\eqref{eq:general_KL_derivative_7_discrete} yields
\begin{align}
\label{eq:general_KL_derivative_8_discrete}
\frac{\dd}{\dd t}K(t)
&\le
-\Bigl(
\frac{\sigma_\eta(t)^2}{2}\dot s(t)\cLSI{t}
-
2\sigma_\eta(t)^{-2}\dot s(t)^{-1}\alpha^{-1}
-
2\dot s(t) \eta^2\alpha'^{-1} \Gamma(t)
\Bigr)K(t)
\nonumber\\
&\quad+
2\sigma_\eta(t)^{-2}\dot s(t)^{-1}\mathfrak{E}_{\alpha}(\pi_{t},m_{t})+2\dot s(t) \Delta(t)\eta
\end{align}

Finally, applying Lemma~\ref{lem:gronwall} finishes the proof.
\end{proof}

\subsection{Evaluation of the Complexity Term for Option I}\label{subsec:complexity-of-option-I}
We provide a way to evaluate the complexity term of VA-SALD with Option I.
First, we clarify the residual term \eqref{eq:residual-term} that is not captured by pretrained velocity field $u_t$.
Let \((p_t)_{t\in[0,T]}\) be a smooth positive density path on \(\R^d\), and let
\(u_t:\R^d\to\R^d\) be a velocity field generating \((p_t)\), i.e.
\[
\partial_t p_t+\nabla\cdot(p_tu_t)=0.
\]
Let \(f_t:\R^d\to\R\) be a smooth guide, and define
\[
Z_t:=\int_{\R^d} p_t(x)e^{-f_t(x)}\,dx,
\qquad
\pi_t(x):=\frac{1}{Z_t}p_t(x)e^{-f_t(x)}.
\]
We also set
\[
g_t(x):=\partial_t f_t(x)+\nabla f_t(x)^\top u_t(x).
\]

\begin{proposition}[Residual identity for the guided path]
\label{prop:guided_path_residual}
For every \(t\in[0,T]\),
\begin{equation}
\label{eq:guided_path_residual}
\partial_t \pi_t+\nabla\cdot(\pi_t u_t)
=
-\pi_t\Bigl(g_t-\E_{\pi_t}[g_t]\Bigr).
\end{equation}
In particular, the right-hand side has zero \(\pi_t\)-mean:
\[
\int_{\R^d}\Bigl(g_t-\E_{\pi_t}[g_t]\Bigr)\pi_t\,dx=0.
\]
\end{proposition}

\begin{proof}
Since $\pi_t=\frac{1}{Z_t}p_t e^{-f_t}$, we first compute the derivative of \(Z_t\). Using
\(\partial_t p_t=-\nabla\cdot(p_tu_t)\) and integration by parts,
\begin{align*}
\dot Z_t
&=
\int_{\R^d}\partial_t\!\bigl(p_t e^{-f_t}\bigr)\,dx \\
&=
\int_{\R^d}(\partial_t p_t)e^{-f_t}\,dx
-
\int_{\R^d}p_t e^{-f_t}\partial_t f_t\,dx \\
&=
-\int_{\R^d}\nabla\cdot(p_tu_t)e^{-f_t}\,dx
-
\int_{\R^d}p_t e^{-f_t}\partial_t f_t\,dx \\
&=
\int_{\R^d}p_tu_t\cdot\nabla(e^{-f_t})\,dx
-
\int_{\R^d}p_t e^{-f_t}\partial_t f_t\,dx \\
&=
-\int_{\R^d}p_t e^{-f_t}
\Bigl(\nabla f_t^\top u_t+\partial_t f_t\Bigr)\,dx \\
&=
- Z_t\,\E_{\pi_t}[g_t].
\end{align*}
Hence
\begin{equation}
\label{eq:logZ_derivative}
\frac{\dot Z_t}{Z_t}=-\E_{\pi_t}[g_t].
\end{equation}

Next, differentiate \(\pi_t\) and use again \(\partial_t p_t=-\nabla\cdot(p_tu_t)\):
\begin{align*}
\partial_t \pi_t
&=
-\frac{\dot Z_t}{Z_t}\pi_t
+\frac{1}{Z_t}(\partial_t p_t)e^{-f_t}
-\frac{1}{Z_t}p_t e^{-f_t}\partial_t f_t \\
&=
-\frac{\dot Z_t}{Z_t}\pi_t
-\frac{1}{Z_t}\nabla\cdot(p_tu_t)e^{-f_t}
-\pi_t\,\partial_t f_t .
\end{align*}
Also,
\begin{align*}
\nabla\cdot(\pi_t u_t)
&=
\nabla\cdot\!\left(\frac{1}{Z_t}p_t e^{-f_t}u_t\right) \\
&=
\frac{1}{Z_t}\nabla\cdot(p_tu_t)e^{-f_t}
+\frac{1}{Z_t}p_tu_t\cdot\nabla(e^{-f_t}) \\
&=
\frac{1}{Z_t}\nabla\cdot(p_tu_t)e^{-f_t}
-\pi_t\,\nabla f_t^\top u_t .
\end{align*}
Adding the two identities, the divergence terms cancel and we obtain
\[
\partial_t \pi_t+\nabla\cdot(\pi_t u_t)
=
-\frac{\dot Z_t}{Z_t}\pi_t
-\pi_t\Bigl(\partial_t f_t+\nabla f_t^\top u_t\Bigr).
\]
Substituting \eqref{eq:logZ_derivative} gives
\[
\partial_t \pi_t+\nabla\cdot(\pi_t u_t)
=
\pi_t\,\E_{\pi_t}[g_t]-\pi_t g_t
=
-\pi_t\Bigl(g_t-\E_{\pi_t}[g_t]\Bigr),
\]
which is exactly \eqref{eq:guided_path_residual}. The mean-zero property is immediate by integration.
\end{proof}

We now estimate the bound on the complexity of $(m^I_t)_{t \in [0,T]}$ along $(\pi_t)_{t\in[0,T]}$.
Suppose $\pi_t$ satisfies Poincar\'{e} inequality with a constant $\cPI{t}>0$ for all $t \in [0,T]$. Then, by Lemma \ref{lem:velocity-norm-bound} with $g = \E_{\pi_t}[g_t] - g_t$, we can choose $m^I_t$ so that $\nabla\cdot(\pi_t m_t^I)=\pi_t (g_t-\E_{\pi_t}[g_t])$ and $\|m^I_t\|_{L_2(\pi_t)} \leq \frac{1}{\sqrt{\cPI{t}}}\| g_t - \E_{\pi_t}[g_t] \|_{L_2(\pi_t)}$. Therefore, we obtain the following bound:
\begin{equation*}\label{eq:action_bound_under_guidance}
    \cA_0(\pi, m^I)
    \leq \int_0^T \frac{1}{\cPI{t}}\Var_{\pi_t}[g_t] \dd t.
\end{equation*}

\subsection{Evaluation of the Complexity Term for Option II}\label{subsec:complexity-of-option-II}
We further clarify the complexity of Option II. Consider the pretrained reverse diffusion with
generator
\[
\mathcal{L}_t\varphi
:=
\left(
    -B_t+\sigma_t^2\nabla\log p_t
\right)^\top\nabla\varphi
+\frac{\sigma_t^2}{2}\Delta\varphi
=
\left(
    u_t+\frac{\sigma_t^2}{2}\nabla\log p_t
\right)^\top\nabla\varphi
+\frac{\sigma_t^2}{2}\Delta\varphi .
\]
Let
\[
h_t(x):=\mathbb{E}\!\left[e^{-f(X_T)}\mid X_t=x\right],
\qquad
F_t^h:=-\log h_t .
\]
Then \(h_t\) satisfies the backward Kolmogorov equation
\begin{equation}\label{eq:backward-Kolmogorov}
    (\partial_t+\mathcal{L}_t)h_t=0,
    \qquad
    h_T=e^{-f}.
\end{equation}
It is known that Doob's $h$-transform (i.e., VA-SALD with Option II and \(f_t=F_t^h\)) yields
\[
    \rho_t(x) = \pi_t(x) \propto  p_t(x)h_t(x).
\]
In particular, if the initial distribution is matched as
\(\rho_0=\pi_0\), the no-slowdown dynamics follows this marginal path
and satisfies
\[
    \rho_T(x) = \pi_T(x) \propto
    p_T(x)e^{-f(x)} = p_{\mathrm{data}}(x)e^{-f(x)}.
\]

Using \(h_t=e^{-F_t^h}\) in the backward Kolmogorov equation \eqref{eq:backward-Kolmogorov} gives
\[
\partial_t F_t^h
+\mathcal{L}_tF_t^h
-\frac{\sigma_t^2}{2}\|\nabla F_t^h\|^2
=0.
\]
Motivated by this identity, for a general prescribed guide \(f_t\), define
\[
\mathcal{R}_t(f)
:=
\partial_t f_t
+\mathcal{L}_t f_t
-\frac{\sigma_t^2}{2}\|\nabla f_t\|^2.
\]
Thus, \(\mathcal{R}_t(F^h)=0\). We note that for prescribed guides $f$, $\mathcal{R}_t(F^h) \neq 0$ and $\rho_t \neq \pi_t$ in general.

The next proposition says that $\mathcal{R}_t(f)$ characterizes the residual transport under Option II.
Recall that 
\[ 
    c_t^{\mathrm{II}} = u_t-\frac{\sigma_t^2}{2}\nabla f_t, \qquad\pi_t\propto p_t e^{-f_t}.
\]
\begin{proposition}[Residual identity for the guided path]
    For every $t\in [0,T]$,
    \[
        \partial_t\pi_t 
        +\nabla\cdot(\pi_t c_t^{\mathrm{II}})
        = 
        -\pi_t\left(
        \mathcal{R}_t(f) 
        - \mathbb{E}_{\pi_t}[\mathcal{R}_t(f)] \right),
    \]
\end{proposition}
\begin{proof}
Using
\[
\partial_t\pi_t+\nabla\cdot(\pi_tu_t)
=
-\pi_t\left(
\partial_t f_t+\langle u_t,\nabla f_t\rangle
-\mathbb{E}_{\pi_t}
\big[
\partial_t f_t+\langle u_t,\nabla f_t\rangle
\big]
\right)
\]
and
\[
\frac{1}{\pi_t}\nabla\cdot(\pi_t\nabla f_t)
=
\Delta f_t
+
\left\langle
\nabla\log p_t-\nabla f_t,\nabla f_t
\right\rangle,
\]
we obtain
\[
\partial_t\pi_t
+\nabla\cdot(\pi_t c_t^{\mathrm{II}})
=
-\pi_t\left(
\mathcal{R}_t(f)
-
\mathbb{E}_{\pi_t}[\mathcal{R}_t(f)]
\right),
\]
where we used, assuming vanishing boundary terms,
\[
\mathbb{E}_{\pi_t}\!\left[
\Delta f_t
+
\left\langle
\nabla\log p_t-\nabla f_t,\nabla f_t
\right\rangle
\right]
=
\int\nabla\cdot(\pi_t\nabla f_t)\,\mathrm{d}x
=0.
\]
\end{proof}

Consequently, the correction field \(m_t^{\mathrm{II}}\) satisfies
\[
\nabla\cdot(\pi_t m_t^{\mathrm{II}})
=
\pi_t\left(
\mathcal{R}_t(f)
-
\mathbb{E}_{\pi_t}[\mathcal{R}_t(f)]
\right).
\]
If \(m_t^{\mathrm{II}}\) is chosen as the minimum-\(L^2(\pi_t)\)-norm
solution, Lemma~3 yields
\[
\|m_t^{\mathrm{II}}\|_{L^2(\pi_t)}^2
\le
\frac{1}{C_t^{\mathrm{PI}}}
\operatorname{Var}_{\pi_t}[\mathcal{R}_t(f)],
\]
and hence
\[
\mathcal{A}_0(\pi,m^{\mathrm{II}})
\le
\int_0^T
\frac{1}{C_t^{\mathrm{PI}}}
\operatorname{Var}_{\pi_t}[\mathcal{R}_t(f)]
\,\mathrm{d}t.
\]
In particular, when \(f_t=F_t^h\), we have
\(\mathcal{R}_t(f)=0\), so the minimum-norm correction field is
\(m_t^{\mathrm{II}}=0\).

Finally, we derive another formulation of $\cR_t$ using 
\[
\delta_t:=f_t-F_t^h.
\]
Define the Doob-transformed generator
\[
\mathcal{L}_t^h\varphi
:=
\mathcal{L}_t\varphi
-\sigma_t^2
\langle\nabla F_t^h,\nabla\varphi\rangle .
\]
Since \(\mathcal{R}_t(F^h)=0\), expanding
\(f_t=F_t^h+\delta_t\) gives
\[
\begin{aligned}
\mathcal{R}_t(f)
&=
(\partial_t+\mathcal{L}_t)\delta_t
-\sigma_t^2
\langle\nabla F_t^h,\nabla\delta_t\rangle
-\frac{\sigma_t^2}{2}\|\nabla\delta_t\|^2 \\
&=
(\partial_t+\mathcal{L}_t^h)\delta_t
-\frac{\sigma_t^2}{2}\|\nabla\delta_t\|^2.
\end{aligned}
\]
Thus, the complexity of Option II is governed by the centered dynamic
inconsistency of the prescribed guide relative to the exact Doob
potential, rather than by the magnitude of the guide itself.

\section{Verification of Theorem~\ref{thm:forward-KL-discrete} for VP-reverse Marginals under Dissipativity}\label{app:vp-dissipative-verification}

We verify the assumptions of Theorem~\ref{thm:forward-KL-discrete} for the reverse marginal path of the VP forward process. Let
\[
    \dd Y_\tau=-\frac12\beta_\tau Y_\tau \dd\tau+\sqrt{\beta_\tau} \dd W_\tau,
    \qquad
    Y_0\sim q_0=p_{\rm data},
\]
where \(0\le \beta_\tau\le \beta_{\max}\) on \([0,T]\). Then
\[
    Y_\tau=a_\tau X_0+\gamma_\tau Z,
    \qquad
    a_\tau:=\exp\left(-\frac12\int_0^\tau\beta_s\,ds\right),
    \qquad
    \gamma_\tau^2:=1-a_\tau^2,
\]
with \(Z\sim N(0,I_d)\) independent of \(X_0\). Then, $(X_0,Y_\tau)$ follows the distribution proportional to 
\[ 
    \exp\left( -V_0(x_0) - \frac{1}{2\gamma_\tau^2}\|y_\tau - a_\tau x_0\|^2 \right). 
\]
We write \(q_\tau=\Law(Y_\tau)\) and consider the reverse marginal path
\[
    \pi_t:=q_{T-t},\qquad t\in[0,T].
\]
Since \(T<\infty\), we have \(a_\tau\ge a_T>0\).

\begin{assumption}[Dissipativity and smoothness]
\label{ass:regular-dissipative-data}
Assume that $q_0(x)=Z_0^{-1}e^{-V_0(x)}$ with \(V_0\in C^3(\mathbb R^d)\), and that there exist constants
\(m_0>0\), \(b_0\ge0\), \(G,L_2,L_3<\infty\) such that
\begin{align}\label{assumption:dissipativity}
    \langle x-y,\nabla V_0(x)-\nabla V_0(y)\rangle
    \ge m_0\|x-y\|^2-b_0,
    \qquad \forall x,y\in\mathbb R^d,
\end{align}
and
\begin{align}\label{assumption:smoothness}
    \|\nabla V_0(x)\|\le G+L_2\|x\|,
    ~~
    \|\nabla^2V_0(x)\|_{\rm op}\le L_2,
    ~~
    \|\nabla^3V_0(x)\|_{\rm op}\le L_3,
    \qquad \forall x\in\mathbb R^d.
\end{align}
\end{assumption}

\begin{lemma}[Radial dissipativity and quadratic exponential moment]\label{lem:radial-dissipativity-qem}
Under Assumption~\ref{ass:regular-dissipative-data}, we have 
\[
    \langle x,\nabla V_0(x)\rangle
    \ge
    c_1\|x\|^2-C_1,
\]
where $c_1:=\frac{m_0}{2},~C_1:=b_0+\frac{G^2}{2m_0}$.
Moreover, for \(X_0\sim q_0\propto e^{-V_0}\),
\[
    M_0(\theta):=\mathbb{E}[e^{\theta\|X_0\|^2}]
    \le
    \left(1-\frac{2\theta}{c_1}\right)^{-(d+C_1)/2},
    \qquad
    0< \forall \theta<\frac{c_1}{2}.
\]
Consequently, for the VP marginal \(q_\tau=\Law(a_\tau X_0+\gamma_\tau Z)\),
\[
    M_\tau(\theta):=
    \mathbb{E}_{q_\tau}[e^{\theta\|X\|^2}]
    \leq
    \mathcal{Q}(\theta),
    \qquad
    \forall \tau \in [0,T],
    ~~0<\forall\theta<\min\left\{\frac{c_1}{4},\frac14\right\},
\]
where 
\[
    \mathcal{Q}(\theta)
    :=
    \left(1-\frac{4\theta}{c_1}\right)^{-(d+C_1)/2}
    (1-4\theta)^{-d/2}.
\]
\end{lemma}
\begin{proof}
Taking \(y=0\) in the two-point dissipativity condition gives
\[
    \langle x,\nabla V_0(x)-\nabla V_0(0)\rangle
    \ge
    m_0\|x\|^2-b_0.
\]
Using \(\|\nabla V_0(0)\|\le G\), we obtain

\begin{equation}\label{eq:radial-dissipativity}
    \langle x,\nabla V_0(x)\rangle
    \ge
    m_0\|x\|^2-G\|x\|-b_0 
    \ge
    \frac{m_0}{2}\|x\|^2
    -
    \left(b_0+\frac{G^2}{2m_0}\right).
\end{equation}

Next, let \(M_0(\theta)=\mathbb E e^{\theta\|X_0\|^2}\). Applying integration by parts to $0 = \int \nabla\cdot \left (xe^{\theta \|x\|^2} e^{-V_0(x_0)} \right) \dd x$,
\[
    \mathbb E[
        \langle X_0,\nabla V_0(X_0)\rangle
        e^{\theta\|X_0\|^2}
    ]
    =
    \mathbb E[
        (d+2\theta\|X_0\|^2)
        e^{\theta\|X_0\|^2}
    ].
\]
Using Eq.~\eqref{eq:radial-dissipativity},
\[
    (c_1-2\theta)M_0'(\theta)
    \le
    (d+C_1)M_0(\theta).
\]
For \(0<\theta<c_1/2\), this yields
\[
    \frac{d}{d\theta}\log M_0(\theta)
    \le
    \frac{d+C_1}{c_1-2\theta}.
\]
Integrating from \(0\) to \(\theta\), and using \(M_0(0)=1\), gives
\[
    M_0(\theta)
    \le
    \left(1-\frac{2\theta}{c_1}\right)^{-(d+C_1)/2}.
\]

Finally, since \(Y_\tau=a_\tau X_0+\gamma_\tau Z\),
\[
    \|Y_\tau\|^2\le 2\|X_0\|^2+2\|Z\|^2.
\]
Therefore,
\[
    \mathbb E[] e^{\theta\|Y_\tau\|^2} ]
    \le
    M_0(2\theta)(1-4\theta)^{-d/2}.
\]
\end{proof}

\begin{lemma}[Posterior first moment bound]\label{lem:posterior-first-moment-stein}
Under Assumption~\ref{ass:regular-dissipative-data}, there exists $C_{\rm post}>0$ such that for all \(\tau\in[0,T]\) and \(y\in\mathbb R^d\),
\[
    \mathbb E[\|X_0\|\mid Y_\tau=y]
    \le
    \sqrt{\frac{d+C_1}{c_1}}+\frac{\|y\|}{a_\tau}
    \leq
    C_{\rm post}(1+\|y\|).
\]
\end{lemma}
\begin{proof}
Let $\Phi_{z,\lambda}(x):=V_0(x)+\lambda\|x-z\|^2$ and $q_{z,\lambda}(x) \propto \exp(-\Phi_{z,\lambda}(x))$.
Applying integration by parts to $\int_{-\infty}^\infty x_i \partial_i q_{z,\lambda}(x)\dd x$, we get
\[
    \mathbb E_{q_{z,\lambda}}
    [
        \langle X,\nabla\Phi_{z,\lambda}(X)\rangle
    ]=d.
\]
On the other hand,
\[
\begin{aligned}
    \langle x,\nabla\Phi_{z,\lambda}(x)\rangle
    &=
    \langle x,\nabla V_0(x)\rangle
    +2\lambda\|x\|^2-2\lambda\langle x,z\rangle \\
    &\ge
    c_1\|x\|^2-C_1
    +2\lambda\|x\|^2-\lambda\|x\|^2-\lambda\|z\|^2 \\
    &=
    (c_1+\lambda)\|x\|^2-C_1-\lambda\|z\|^2.
\end{aligned}
\]

Therefore, $(c_1+\lambda)\mathbb E_{q_{z,\lambda}}[\|X\|^2] \le d+C_1+\lambda\|z\|^2$.
Thus,
\begin{equation*}
    \mathbb E_{q_{z,\lambda}}[\|X\|]
    \le
    \left(
        \frac{d+C_1+\lambda\|z\|^2}{c_1+\lambda}
    \right)^{1/2} 
    \le
    \sqrt{\frac{d+C_1}{c_1}}+\|z\|.
\end{equation*}
For the VP posterior, \(z=y/a_\tau\) and
\(\lambda=a_\tau^2/(2\gamma_\tau^2)\), which gives the claim.
\end{proof}

\begin{lemma}[Score growth]\label{lem:score-growth-final}
Under Assumption~\ref{ass:regular-dissipative-data}, there exist constants
\(A,B<\infty\) such that
\[
    \|\nabla\log q_\tau(y)\|
    \le
    A+B\|y\|,
    \qquad
    \forall \tau\in[0,T],~\forall y\in\R^d.
\]
\end{lemma}

\begin{proof}
For \(\tau>0\), integration by parts gives
\[
    \nabla\log q_\tau(y)
    =
    -\frac1{a_\tau}
    \mathbb E[\nabla V_0(X_0)\mid Y_\tau=y].
\]
Therefore, using \eqref{assumption:smoothness} and Lemma~\ref{lem:posterior-first-moment-stein},
\[
    \|\nabla\log q_\tau(y)\|
    \leq
    \frac{1}{a_T}
    \mathbb E[G+L_2\|X_0\|\mid Y_\tau=y] 
    \leq
    \frac{G}{a_T} + \frac{C_{\rm post}L_2}{a_T}(1+\|y\|).
\]
At \(\tau=0\), by \eqref{assumption:smoothness}
\[
    \|\nabla\log q_0(y)\|= \|\nabla V_0(y)\| \leq G+L_2\|y\|,
\]
\end{proof}

Let $q_{\tau,y}(x)$ be a conditional probability of $X_0$ given $Y_\tau = y$, that is,
\begin{equation}\label{eq:conditional-probability}
    q_{\tau,y}(x)
    \propto
    \exp\left(
        -V_0(x)-\frac{\|y-a_\tau x\|^2}{2\gamma_\tau^2}
    \right).
\end{equation}
\begin{lemma}[Posterior centered moments]
\label{lem:posterior-centered-moments}
Let $X$ be a random variable that follows $q_{\tau,y}$.
Under Assumption~\ref{ass:regular-dissipative-data}, 
\begin{align*}
    &\E[\|X-\E[X]\|^2]
    \leq
    \frac{2d+b_0}{2(m_0+\eta_\tau)}, \\
    &\E[\|X-\E[X]\|^4]
    \leq
    \frac{(2d+4+b_0)(2d+b_0)}{(m_0+\eta_\tau)^2},
\end{align*}
where $\eta_\tau:=a_\tau^2/\gamma_\tau^2$.
\end{lemma}
\begin{proof}
Let $\Phi_{\tau,y}(x):=V_0(x)+\frac{\eta_\tau}{2}\left\|x-\frac{y}{a_\tau}\right\|^2$.
Then \(q_{\tau,y}\propto e^{-\Phi_{\tau,y}}\), and by \eqref{assumption:dissipativity},
\[
    \langle x-x',
        \nabla\Phi_{\tau,y}(x)-\nabla\Phi_{\tau,y}(x')\rangle  
    \geq
    (m_0+\eta_\tau)\|x-x'\|^2-b_0.
\]

Let \(X,X'\) be i.i.d. from \(q_{\tau,y}\). For \(p\ge2\), 
\begin{align*}
    0&=\sum_{i=1}^d\int_{-\infty}^\infty \partial_{x_i}( \|x-x'\|^{p-2}(x_i-x'_i) q_{\tau,y}(x)) \dd x\\
    &=\sum_{i=1}^d\int_{-\infty}^\infty \left\{ \partial_{x_i}( \|x-x'\|^{p-2}(x_i-x'_i) ) q_{\tau,y}(x) -\|x-x'\|^{p-2}(x_i-x'_i)\partial_{x_i}\Phi_{\tau,y}(x) q_{\tau,y}(x) \right\}\dd x \\
    &=(d+p-2)\int_{-\infty}^\infty  \|x-x'\|^{p-2} q_{\tau,y}(x) \dd x 
    - \int_{-\infty}^\infty \|x-x'\|^{p-2} \langle x-x', \Phi_{\tau,y}(x) \rangle q_{\tau,y}(x) \dd x.
\end{align*}
Taking the expectation w.r.t $X' \sim q_{\tau,y}(X')$ as well, we have
\[ 
    (d+p-2)\E[ \|X-X'\|^{p-2}] = \E[\|X-X'\|^{p-2} \langle X-X', \Phi_{\tau,y}(X) \rangle].
\]

The equation where $X$ and $X'$ are flipped is also obtained in the same way. As a result,
\begin{align*}
    \E\left[
        \|X-X'\|^{p-2}
        \langle X-X',
        \nabla\Phi_{\tau,y}(X)-\nabla\Phi_{\tau,y}(X')\rangle
    \right]  
    = 2(d+p-2)\E\|X-X'\|^{p-2}.
\end{align*}

Therefore, we get
\begin{equation}\label{eq:general-p-bound}
    (m_0+\eta_\tau)[\E\|X-X'\|^p]
    \le
    \{2(d+p-2)+b_0\}\E[\|X-X'\|^{p-2}].
\end{equation}
For $p=2$,
\[
    \E[\|X-X'\|^2]
    \leq
    \frac{2d+b_0}{m_0+\eta_\tau}.
\]
Since $\E[\|X-X'\|^2] = 2\E[\|X-\E[X]\|^2]$, the second-moment bound follows. 

For $p=4$, using the bound with $p=2$ and Eq.~\eqref{eq:general-p-bound},
\[
    \E[\|X-X'\|^4]
    \leq
    \frac{(2d+4+b_0)(2d+b_0)}{(m_0+\eta_\tau)^2}.
\]
This bound and Jensen's inequality $\E[\|X-\E[X]\|^4] \leq \E[\|X-X'\|^4]$ conclude the proof.
\end{proof}

\begin{lemma}[Uniform spatial Lipschitzness of the VP score]
\label{lem:score-hessian-final}
Under Assumption~\ref{ass:regular-dissipative-data},
\[
    \sup_{\tau\in[0,T]}\sup_{y\in\mathbb R^d}
    \|\nabla_y^2\log q_\tau(y)\|_{\rm op}
    \leq 
    \max\left\{
        L_2,
        \frac{L_2(2d+b_0)}{2a_T^2}
    \right\} =: H_2.
\]
Consequently, the condition \eqref{eq:lip_SALD_1} of Theorem~\ref{thm:forward-KL-discrete} holds for
$\pi_t=q_{T-t}$ with $L_{\pi,\mathrm{space}}=H_2$.
\end{lemma}

\begin{proof}
For \(\tau>0\), posterior calculus gives
\[
    \nabla_y^2\log q_\tau(y)
    =
    -\frac1{\gamma_\tau^2}
    \operatorname{Cov}_{q_{\tau,y}}(\nabla V_0(X),X).
\]
Let $X,X'$ be i.i.d. from $q_{\tau,y}$. Then
\[
    \operatorname{Cov}(\nabla V_0(X),X)
    = \frac{1}{2} \E[(\nabla V_0(X)-\nabla V_0(X'))(X-X')^\top].
\]
Using the $L_2$-Lipschitz continuity of $\nabla V_0$,
\[
    \|\operatorname{Cov}(\nabla V_0(X),X)\|_{\rm op}
    \leq \frac{L_2}{2} \E[ \| X - X' \|^2 ]
    = L_2 \E[\|X-\E[X]\|^2].
\]
By Lemma~\ref{lem:posterior-centered-moments},
\[
    \|\operatorname{Cov}(\nabla V_0(X),X)\|_{\rm op}
    \le
    \frac{L_2(2d+b_0)}{2(m_0+\eta_\tau)}.
\]
Thus
\[
    \|\nabla_y^2\log q_\tau(y)\|_{\rm op}
    \leq
    \frac{L_2(2d+b_0)}
    {2\gamma_\tau^2(m_0+\eta_\tau)}  
    =
    \frac{L_2(2d+b_0)}
    {2(a_\tau^2+m_0\gamma_\tau^2)}
    \leq
    \frac{L_2(2d+b_0)}{2a_T^2}.
\]

At \(\tau=0\), \(\nabla^2\log q_0=-\nabla^2V_0\), hence
\[
    \|\nabla^2\log q_0\|_{\rm op}\le L_2.
\]
The claim follows.
\end{proof}

\begin{lemma}[Uniform third derivative bound for the VP score]
\label{lem:score-third-final}
Under Assumption~\ref{ass:regular-dissipative-data},
\[
    \sup_{\tau\in[0,T]}\sup_{y\in\mathbb R^d}
    \|D_y^3\log q_\tau(y)\|_{\rm op}
    \le H_3
\]
for some finite constant \(H_3\) depending only on
\(a_T,L_2,L_3,m_0,b_0,d\).
\end{lemma}

\begin{proof}
For \(\tau>0\), posterior calculus gives, for unit vectors \(u,v,w\),
\begin{equation}\label{eq:third-derivative_score}
    D_y^3\log q_\tau(y)[u,v,w]
    =
    \frac{a_\tau^3}{\gamma_\tau^6}
    \mathbb E_{q_{\tau,y}}
    [
        \langle u,\bar X\rangle
        \langle v,\bar X\rangle
        \langle w,\bar X\rangle
    ],
\end{equation}
where $X \sim q_{\tau,y} \propto
    \exp\left(
        -V_0(x)-\frac{\|y-a_\tau x\|^2}{2\gamma_\tau^2}
    \right)$ and $\bar X=X-\E[X]$.

Set \(\eta_\tau=a_\tau^2/\gamma_\tau^2\). We split into two regimes.

First suppose \(\eta_\tau\ge2L_2\). Let \(m:=\E[X]\), \(H:=\nabla^2V_0(m)\), and write
\[
    \nabla V_0(X)-\E[\nabla V_0(X)]
    =
    H\bar X+\mathcal R(X)-\E[\mathcal{R}(X)] ,
\]
where
\[
    \mathcal R(X)
    :=
    \nabla V_0(X)-\nabla V_0(m)-H\bar X.
\]
By Eq.~\eqref{assumption:smoothness} (Assumption \ref{ass:regular-dissipative-data}),
\[
    \|\mathcal R(X)\|\le \frac{L_3}{2}\|\bar X\|^2.
\]
Set $\Phi_{\tau,y}(x):=V_0(x)+\frac{\eta_\tau}{2}\left\|x-\frac{y}{a_\tau}\right\|^2$ and
\[
    T(u,v,w):=
    \mathbb E[
        \langle u,\bar X\rangle
        \langle v,\bar X\rangle
        \langle w,\bar X\rangle
    ],
\]
By integration by parts 
\begin{align*}
    &0 = \E[\nabla f(X)]
    = \E[f(X)\nabla\Phi_{\tau,y}(X)]~~~~(f(x)=\langle u,x-m\rangle\langle v,x-m\rangle), \\
    &0 = \E[\nabla\Phi_{\tau,y}(X)],
\end{align*}
we obtain $0 = \E[f(X) (\nabla\Phi_{\tau,y}(X) - \E[\nabla\Phi_{\tau,y}(X)] )]$.
Combining this with 
\begin{align*}
    \nabla\Phi_{\tau,y}(X) - \E[\nabla\Phi_{\tau,y}(X)]
    &= \nabla V_0(X) - \E[\nabla V_0(X)] 
    + \eta_r \bar{X} \\
    &= H\bar X+\mathcal R(X)-\E[\mathcal{R}(X)] + \eta_r \bar{X},
\end{align*}
we get
\[
    T(u,v,(\eta_\tau I+H)w)
    =
    -
    \E[
        \langle u,\bar X\rangle
        \langle v,\bar X\rangle
        \langle w,\mathcal{R}(X)-\E[\mathcal{R}(X)]\rangle
    ].
\]

Since \(\|H\|_{\rm op}\le L_2\) and \(\eta_\tau\ge2L_2\), $\|\eta_\tau I+H\|_{\rm op} \geq \eta_\tau-L_2$.
Therefore,
\[
    (\eta_\tau-L_2) |T(u,v,w)|
    \leq
    \frac{L_3}{2}
    \left\{
        \E[\|\bar X\|^4]
        +
        (\E[\|\bar X\|^2])^2
    \right\}.
\]

By Lemma~\ref{lem:posterior-centered-moments}, the bracketed term is bounded by
\(\exists C(m_0,b_0,d)/(m_0+\eta_\tau)^2\). Since \(\eta_\tau\ge2L_2\), there exists $C>0$ such that 
\[
    |T(u,v,w)|
    \le
    \frac{C L_3}{\eta_\tau^3}.
\]
Plugging this into Eq.~\eqref{eq:third-derivative_score} yields
\[
    |D_y^3\log p_\tau(y)[u,v,w]|
    \le
    \frac{C L_3 a_\tau^3}{\gamma_\tau^6\eta_\tau^3}
    =
    \frac{C L_3}{a_\tau^3}
    \le
    \frac{C L_3}{a_T^3}.
\]

Next suppose \(\eta_\tau<2L_2\). Then $\frac{a_\tau^3}{\gamma_\tau^6} = \frac{\eta_\tau^3}{a_\tau^3} \leq \frac{(2L_2)^3}{a_T^3}$.

Moreover, Lemma~\ref{lem:posterior-centered-moments} gives a uniform fourth moment bound, and hence $\E\|\bar X\|^3\leq \exists C(m_0,b_0,d)$.

Therefore, Eq.~\eqref{eq:third-derivative_score} gives a uniform bound in this regime as well.

At $\tau=0$, $D^3\log q_0=-D^3V_0$, which is bounded by $L_3$. Combining the three cases proves the claim.
\end{proof}

\begin{lemma}[Time Lipschitzness of the VP score]
\label{lem:score-time-final}
Under Assumption~\ref{ass:regular-dissipative-data}, there exists
\(C_{\rm time}<\infty\) such that
\[
    \|\partial_\tau\nabla\log q_\tau(x)\|
    \le
    C_{\rm time}(1+\|x\|),
    \qquad
    \forall \tau\in[0,T],\ x\in\R^d.
\]
Consequently, for the linear slowdown \(t(s)=s/r\), \(r\ge1\),the condition \eqref{eq:lip_SALD_2} of Theorem~\ref{thm:forward-KL-discrete} holds with
\[
    M(x)=\|x\|,
    \qquad
    L_{\pi,\mathrm{time}}=\frac{C_{\rm time}}{r}.
\]
\end{lemma}

\begin{proof}
Let \(s_\tau(x):=\nabla\log q_\tau(x)\). The Fokker--Planck equation gives
\[
    \partial_\tau q_\tau
    =
    \frac{\beta_\tau}{2}\nabla\cdot(x q_\tau)
    +
    \frac{\beta_\tau}{2}\Delta q_\tau.
\]
Hence
\[
    \partial_\tau\log q_\tau
    =
    \frac{\beta_\tau}{2}
    \left[
        d+x^\top s_\tau+\nabla\cdot s_\tau+\|s_\tau\|^2
    \right].
\]
Taking the gradient,
\[
    \partial_\tau s_\tau
    =
    \frac{\beta_\tau}{2}
    \left[
        s_\tau+(\nabla s_\tau)x
        +\nabla(\nabla\cdot s_\tau)
        +2(\nabla s_\tau)s_\tau
    \right].
\]
Using Lemmas~\ref{lem:score-growth-final}, \ref{lem:score-hessian-final}, and \ref{lem:score-third-final}, we get
\[
    \|\partial_\tau s_\tau(x)\|
    \le
    C_{\rm time}(1+\|x\|).
\]
For \(\pi_{t(s)}=q_{T-t(s)}\), the mean value theorem gives
\[
\begin{aligned}
    &\|\nabla\log\pi_{t(s')}(x)-\nabla\log\pi_{t(s)}(x)\|  \\
    &\qquad\leq
    C_{\rm time}|t(s')-t(s)|(1+\|x\|)
    \leq
    \frac{C_{\rm time}}{r}(s'-s)(1+\|x\|)
\end{aligned}
\]
for \(s'\ge s\), because \(t(s)=s/r\) and \(r\ge1\).
\end{proof}

\begin{proposition}[Verification of Theorem~\ref{thm:forward-KL-discrete} assumptions]
\label{prop:verify-thm2-dissipative}
Under Assumption~\ref{ass:regular-dissipative-data}, the VP reverse marginal path
\[
    \pi_t=q_{T-t}
\]
satisfies the smoothness and complexity assumptions of Theorem~\ref{thm:forward-KL-discrete}. More precisely, under the linear slowdown $t(s)=s/r$, the conditions \eqref{eq:lip_SALD_1} and \eqref{eq:lip_SALD_2} holds with
\[
    L_{\pi,\mathrm{space}}=H_2,~~~~M(x)=\|x\|,~~~~L_{\pi,\mathrm{time}}=C_{\rm time}.
\]
Furthermore, for sufficiently small \(\alpha,\alpha'>0\),
\[
    A_\alpha(\pi,v)<\infty,
    \qquad
    \sup_{t\in[0,T]}\mathfrak E_{\alpha'}(\pi_t,\nabla\log\pi_t)<\infty,
    \qquad
    \sup_{t\in[0,T]}\mathfrak E_{\alpha'}(\pi_t,1+M)<\infty.
\]
\end{proposition}

\begin{proof}
The conditions \eqref{eq:lip_SALD_1} and \eqref{eq:lip_SALD_2} follow from Lemmas~\ref{lem:score-hessian-final} and \ref{lem:score-time-final}.

It remains to verify the complexity terms. We use the elementary fact that if
\[
    \|g(x)\|\le m+M\|x\|,
\]
then, for any $\gamma>0$ such that $2\gamma B^2<\bar\theta := \min\left\{\frac{c_1}{4},\frac{1}{4}\right\}$,
\[
\begin{aligned}
    \mathfrak{E}_\gamma(q_\tau,g)
    :=
    \frac1\gamma
    \log\E_{q_\tau}[e^{\gamma\|g(X)\|^2}] 
    \le
    2m^2+
    \frac1\gamma
    \log\mathcal Q(2\gamma M^2),
\end{aligned}
\]
by Lemma~\ref{lem:radial-dissipativity-qem}.

First, Lemma~\ref{lem:score-growth-final} gives $\|\nabla\log p_\tau(x)\|\le A+B\|x\|$, and hence if $2\alpha'B^2<\bar\theta$,
\[
    \sup_{t\in[0,T]}
    \mathfrak E_{\alpha'}(\pi_t,\nabla\log\pi_t)
    \le
    2A^2+
    \frac1{\alpha'}\log\mathcal Q(2\alpha'B^2).
\]
Second, since $M(x)=\|x\|$ (Lemma \ref{lem:score-time-final}), $(1+M(x))^2\leq 2+2\|x\|^2$.
Therefore, if \(2\alpha'<\bar\theta\),
\[
    \sup_{t\in[0,T]}
    \mathfrak E_{\alpha'}(\pi_t,1+M)
    \le
    2+
    \frac1{\alpha'}\log\mathcal Q(2\alpha').
\]

Finally, the transport velocity of the reverse VP marginal path is
\[
    v_{T-\tau}(x)
    =
    \frac{\beta_\tau}{2}
    \left(x+\nabla\log q_\tau(x)\right).
\]
Hence
\[
    \|v_{T-\tau}(x)\|
    \le
    A_v+B_v\|x\|,
\]
where one can take $A_v:=\frac{\beta_{\max}}2 A_s,~B_v:=\frac{\beta_{\max}}2(1+B_s)$.
Thus, if $2\alpha B_v^2<\bar\theta$,
\begin{align*}
    A_\alpha(\pi,v)
    =
    \int_0^T
    \mathfrak E_\alpha(\pi_t,v_t) \dd t 
    =
    \int_0^T
    \mathfrak E_\alpha(p_\tau,v_{T-\tau}) \dd \tau 
    \leq
    T\left[
        2A_v^2+
        \frac1\alpha\log\mathcal Q(2\alpha B_v^2)
    \right].
\end{align*}

This proves all required complexity bounds.
\end{proof}

\newpage

\section{Additional Experimental Details}\label{app:experiment_details}

\subsection{Additional Details of Synthetic Data}\label{app:additional_syth}

This appendix reports additional experimental details of two tasks.
Both tasks use the same scalar VP forward diffusion
\begin{equation}
    \mathrm{d}Y_\tau
    =
    -\frac{1}{2}\beta(\tau)Y_\tau\,\mathrm{d}\tau
    + \sqrt{\beta(\tau)}\,\mathrm{d}W_\tau,
    \qquad
    \beta(\tau)=\beta_{\min}
    +\frac{\beta_{\max}-\beta_{\min}}{T}\tau ,
\end{equation}
and the reverse-indexed marginal family $p_t=q_{T-t}$.  The SALD and
VA-SALD samplers use the slowed time parametrization $t=s/r$, where $r$
is the computational budget.  In contrast, the DOIT baseline is run on the
ordinary reverse VP time interval $t\in[0,T]$; its budget label is denoted
by $r_b$ below to emphasize that it controls only the number of Euler
transitions, not a time slowdown.

\paragraph{Unguided VP sanity check.}
Before adding guide functions, we verify that the Euler-Maruyama SALD sampler
recovers the terminal two-Gaussian mixture as the slow-down budget increases.

\begin{figure}[h]
    \centering
    \includegraphics[width=0.72\textwidth]{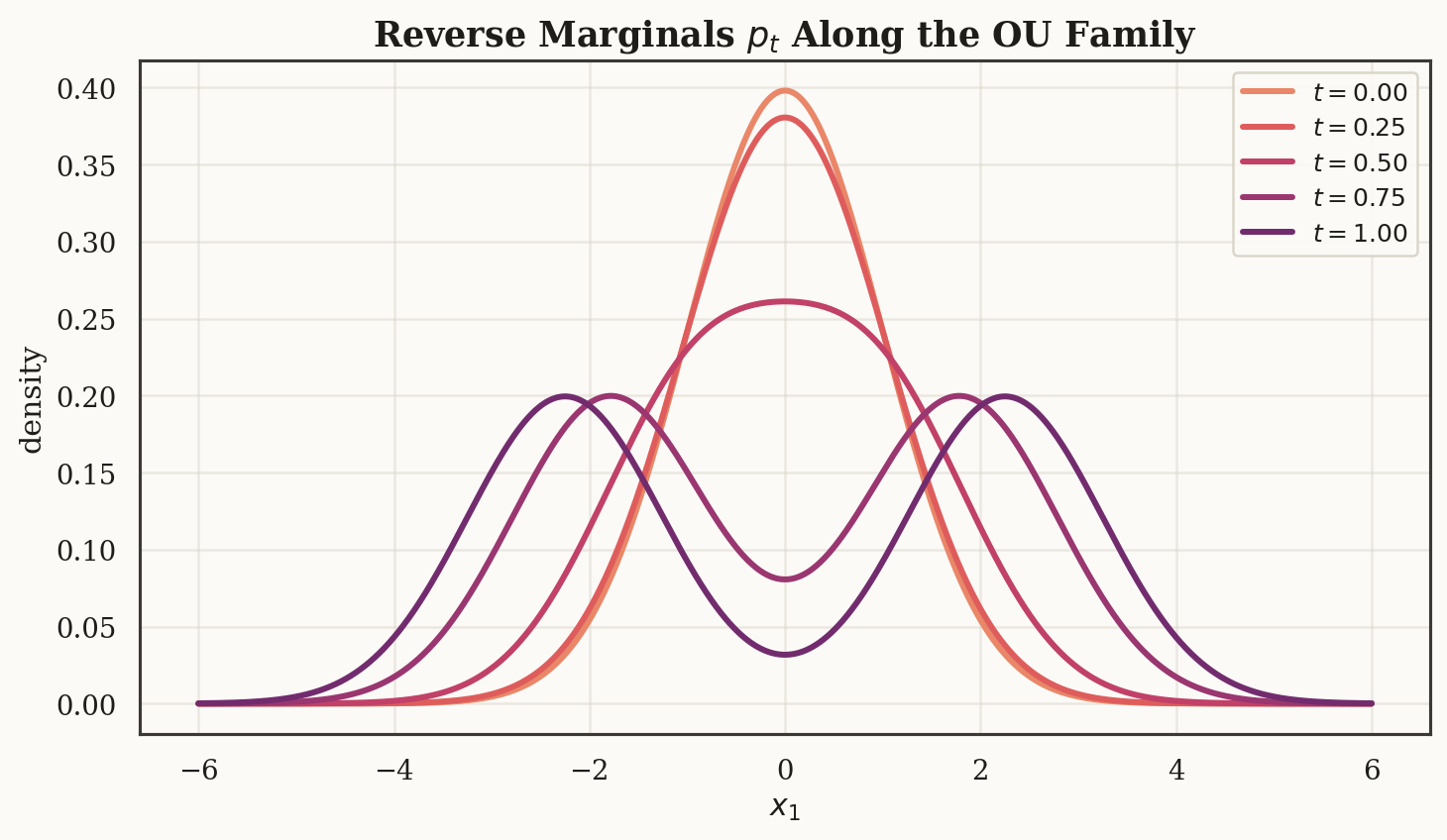}
    \caption{\textbf{Reverse-indexed VP marginal family for the unguided
    two-Gaussian experiment.}  The forward VP diffusion contracts the mixture
    means toward the origin, so the reverse initialization is close to a
    centered Gaussian while the terminal law remains bimodal.}
    \label{fig:app-target-family}
\end{figure}

\begin{figure}[h]
    \centering
    \includegraphics[width=0.92\textwidth]{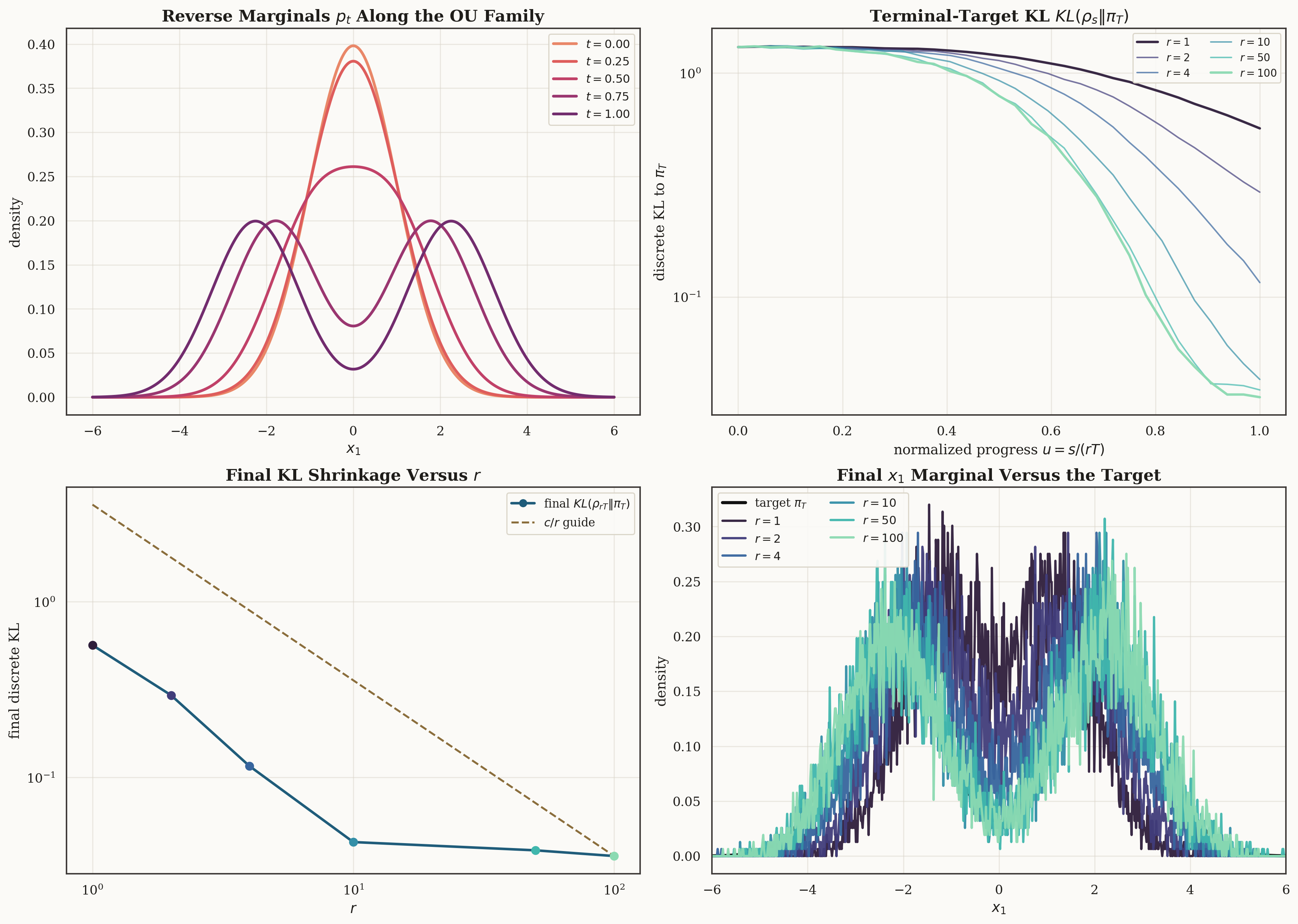}
    \caption{\textbf{Unguided SALD overview.}  The KL trajectory and terminal KL
    confirm that increasing the slow-down budget reduces the mismatch to the
    terminal two-Gaussian target until the finite-particle and discretization
    floor dominates.}
    \label{fig:app-sald-overview}
\end{figure}

\begin{figure}[h]
    \centering
    \includegraphics[width=0.92\textwidth]{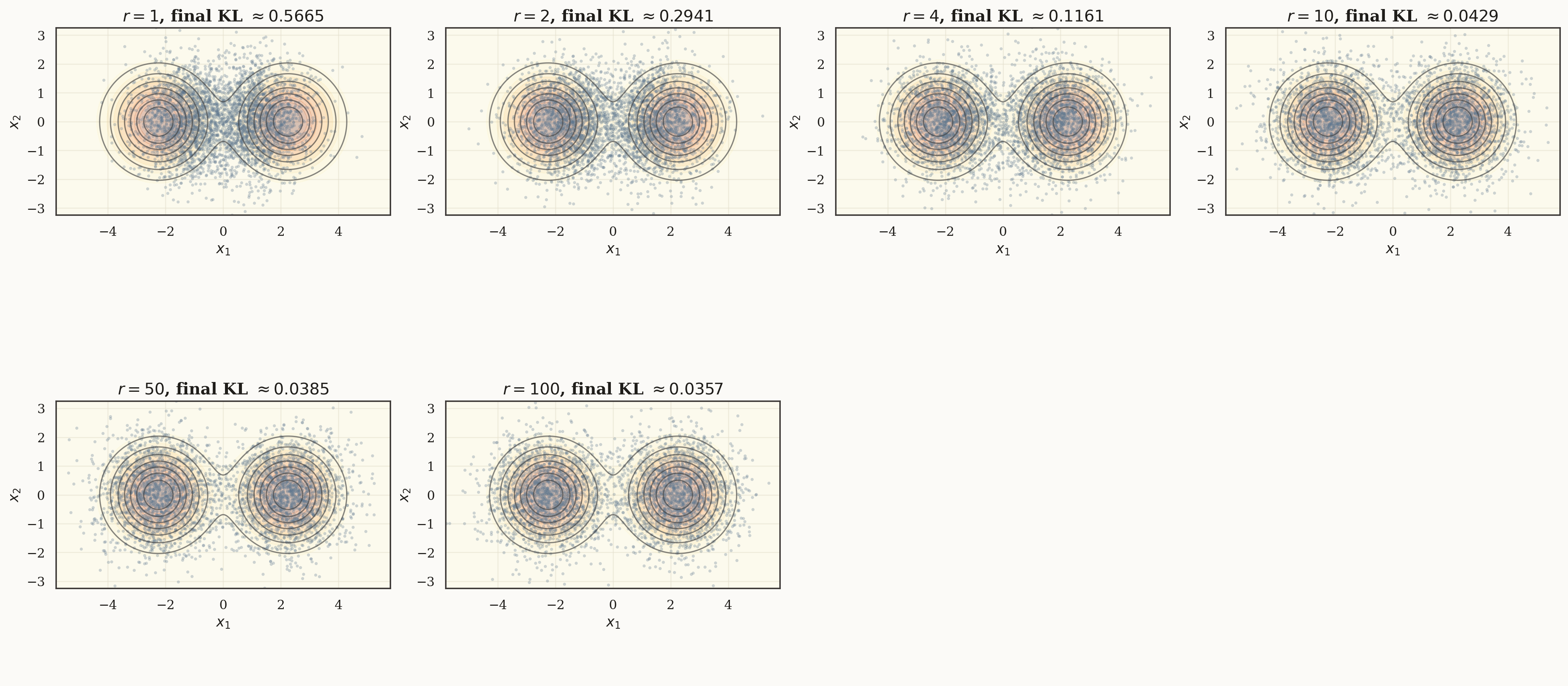}
    \caption{\textbf{Unguided SALD terminal samples.}  Larger budgets produce
    visibly sharper recovery of the two target modes.}
    \label{fig:app-sald-samples}
\end{figure}

\subsubsection{Two-moons guided two-Gaussian VP diffusion.}
The first guided task starts from a two-component Gaussian data distribution
and uses a two-moons guide.  The guide is represented by a fixed reference
cloud $\{y_j\}_{j=1}^N$ sampled from a translated and rescaled two-moons
distribution.  The potential and gradient are
\begin{equation}
    f_{\mathrm{moon}}(x)
    =
    \frac{1}{\lambda N}\sum_{j=1}^N \|x-y_j\|_2,
    \qquad
    \nabla f_{\mathrm{moon}}(x)
    =
    \frac{1}{\lambda N}\sum_{j=1}^N
    \frac{x-y_j}{\|x-y_j\|_2}.
\end{equation}
The guided moving target is
\begin{equation}
    \pi_t(x)\propto p_t(x)\exp\{-f_{\mathrm{moon}}(x)\}.
\end{equation}
Directly averaging over the full reference population inside every sampler
step is expensive, so the notebook precomputes $f_{\mathrm{moon}}$ and
$\nabla f_{\mathrm{moon}}$ on a two-dimensional grid and uses bilinear
interpolation during sampling.  All reported KL values for this task are
computed against the guided terminal target $\pi_T$, not against the
unguided two-Gaussian law.  The next figures isolate each method on this
same two-moons guided task.  These diagnostics show the guided terminal
density, the KL trajectory, terminal KL, and the mean guidance objective.

\begin{figure}[h]
    \centering
    \includegraphics[width=0.92\textwidth]{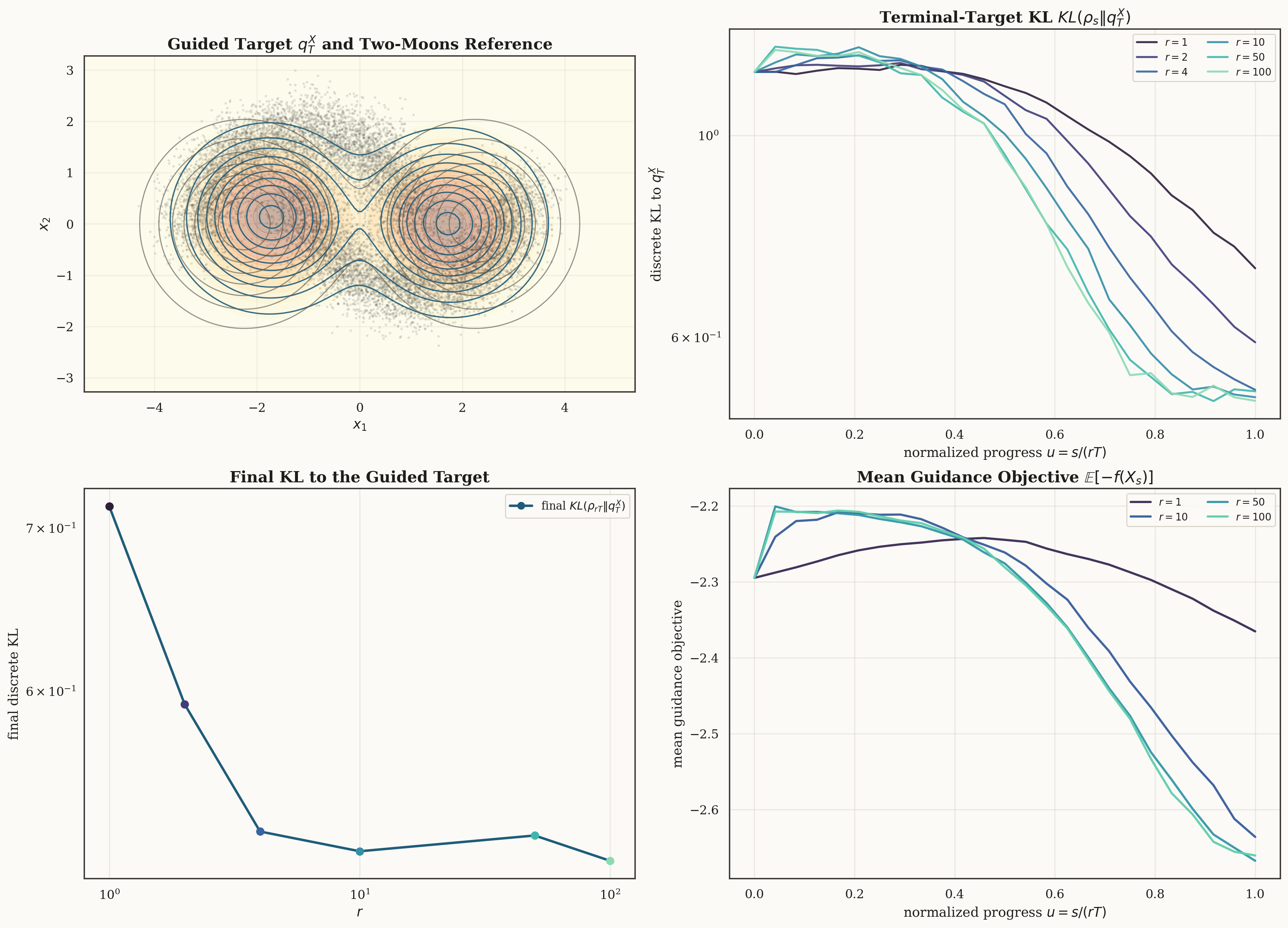}
    \caption{\textbf{SALD on the two-moons guided two-Gaussian task.}  SALD
    follows the expected budget trend: larger $r$ improves the terminal-target
    KL, with the curve eventually flattening near the Monte Carlo floor.}
    \label{fig:app-guided-overview}
\end{figure}

\begin{figure}[h]
    \centering
    \includegraphics[width=0.92\textwidth]{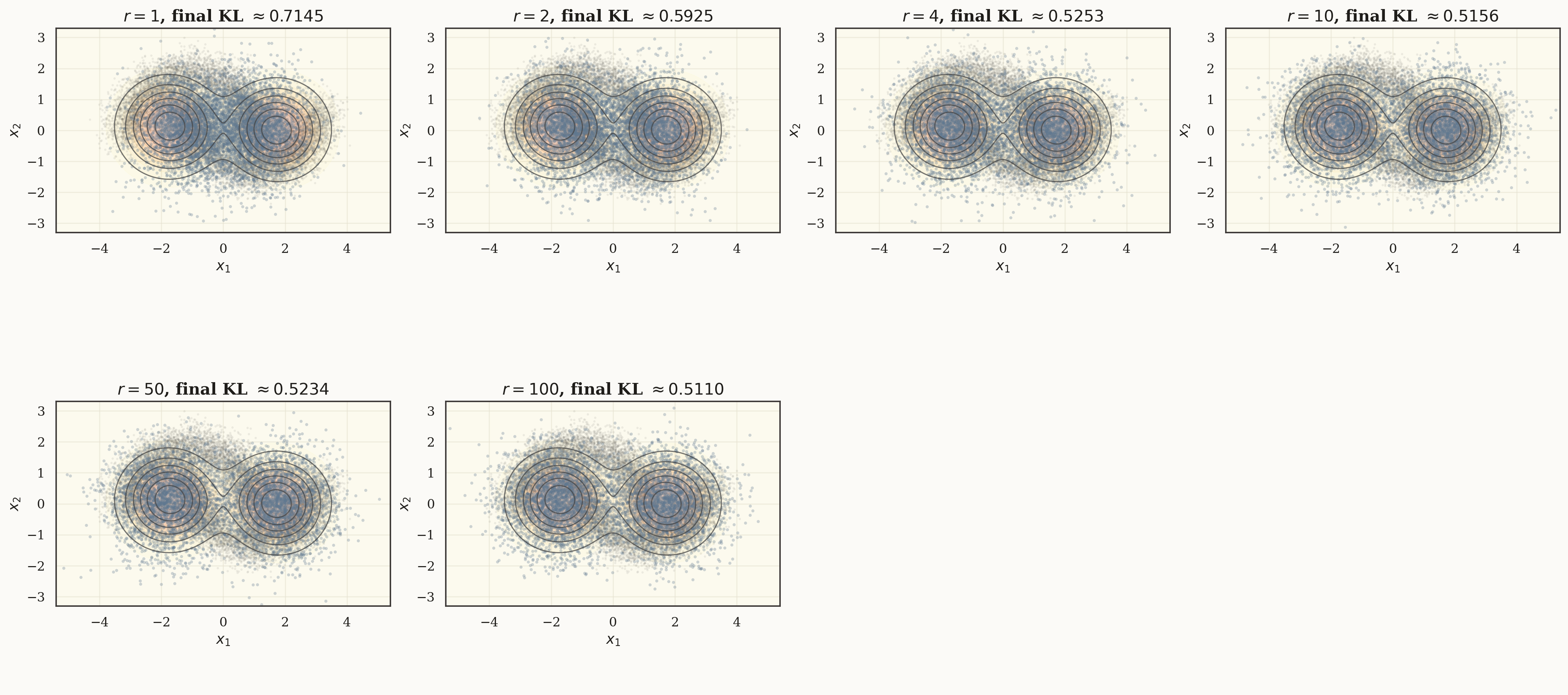}
    \caption{\textbf{SALD samples for the two-moons guided task.}  Terminal
    samples increasingly align with the two-moons guided target as the budget
    increases.}
    \label{fig:app-guided-samples}
\end{figure}

\begin{figure}[h]
    \centering
    \includegraphics[width=0.92\textwidth]{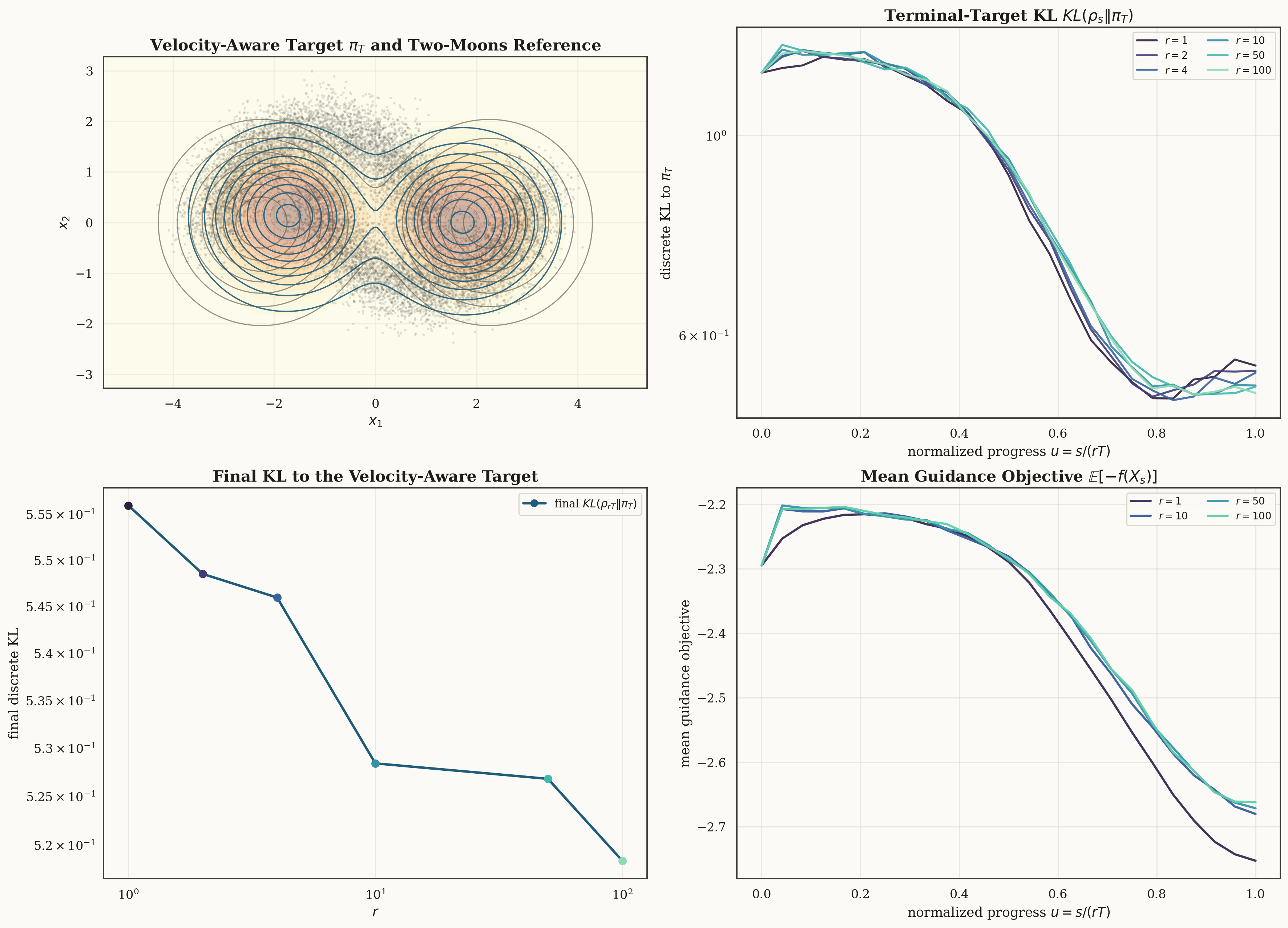}
    \caption{\textbf{VA-SALD on the two-moons guided task.}  Incorporating the
    VP velocity field yields low terminal KL across the tested budgets and a
    stable guidance-objective trajectory.}
    \label{fig:app-va-overview}
\end{figure}

\begin{figure}[h]
    \centering
    \includegraphics[width=0.92\textwidth]{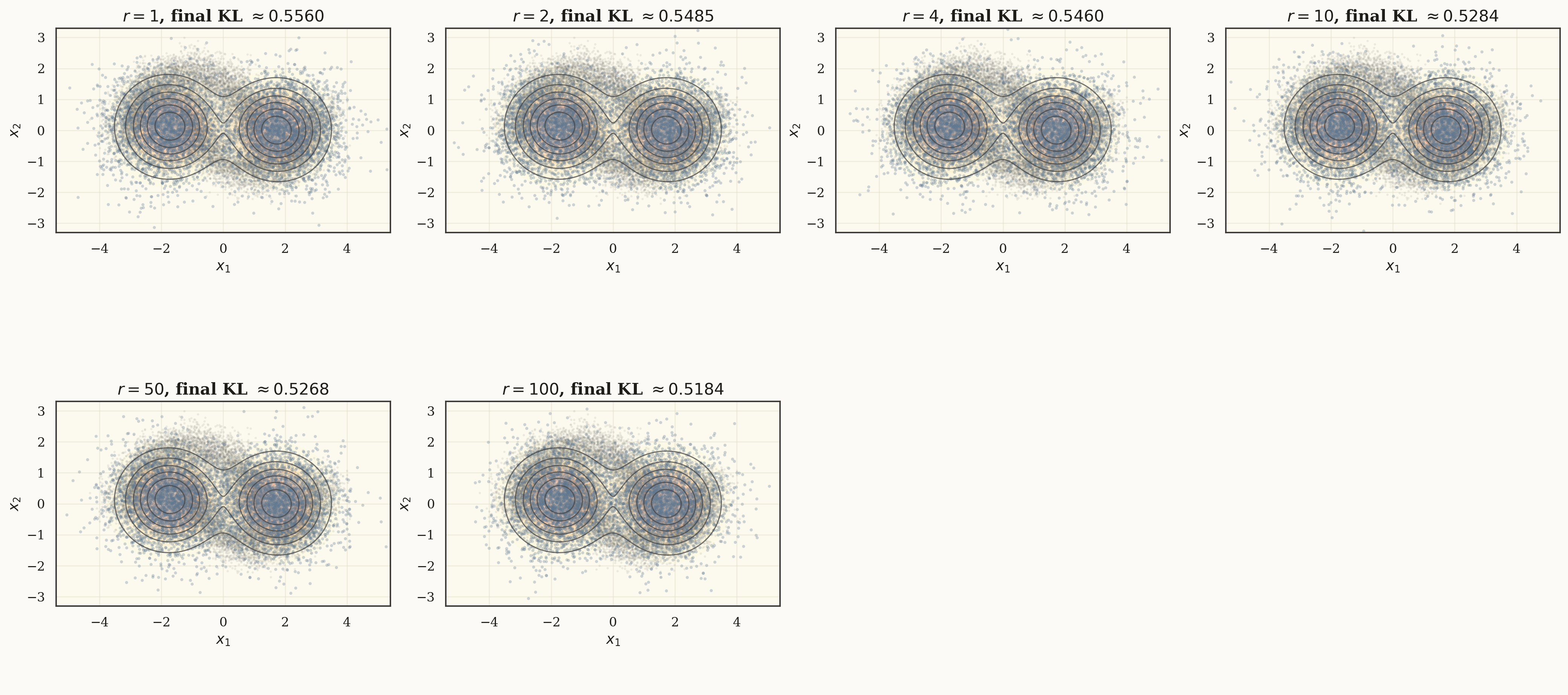}
    \caption{\textbf{VA-SALD samples for the two-moons guided task.}  The sample
    panels show that VA-SALD reaches the guided terminal geometry at small and
    large budgets.}
    \label{fig:app-va-samples}
\end{figure}

\begin{figure}[h]
    \centering
    \includegraphics[width=0.92\textwidth]{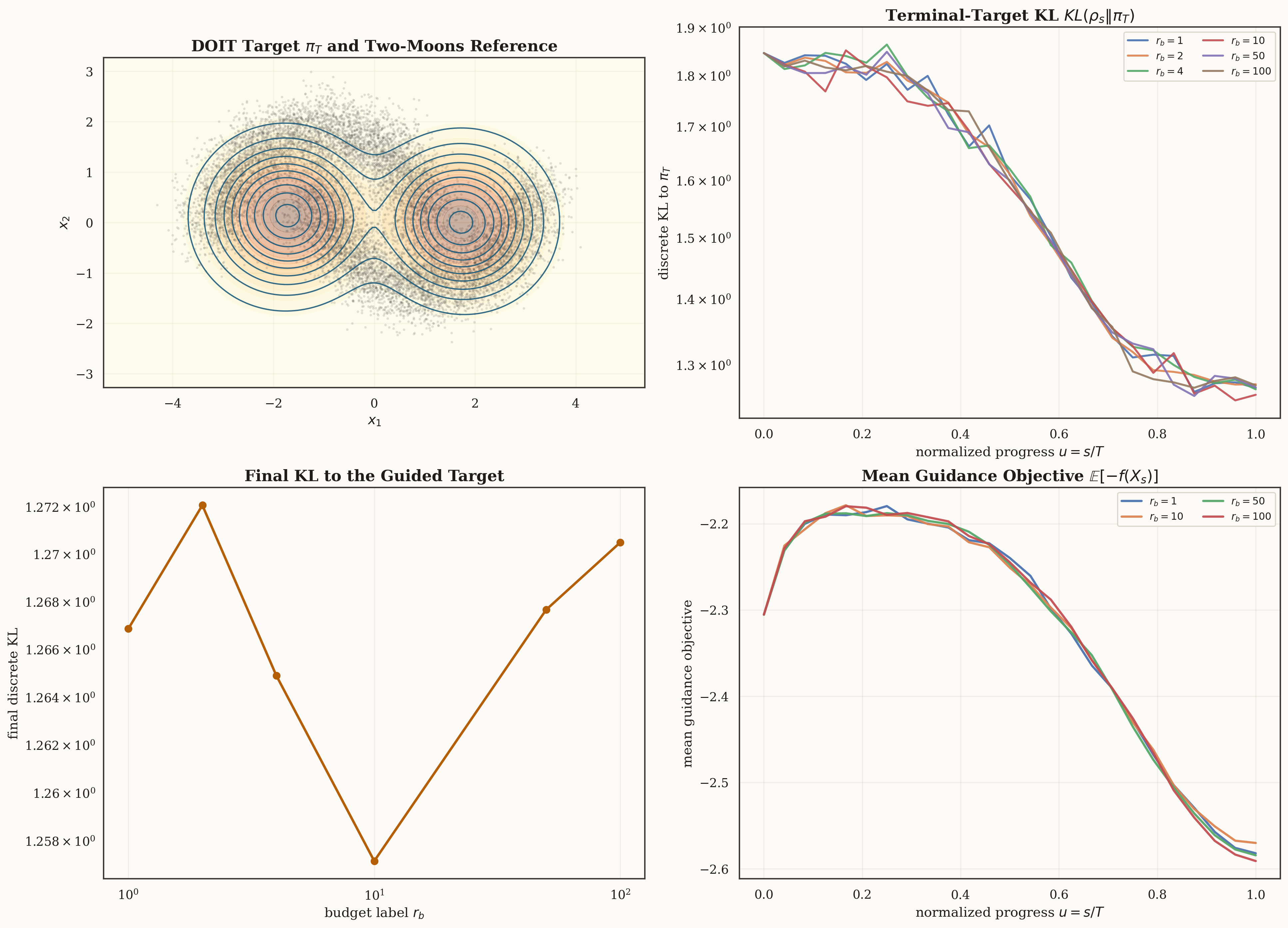}
    \caption{\textbf{DOIT on the two-moons guided task.}  DOIT uses the same VP
    reverse SDE base sampler and the same total proposal budget, but its local
    Doob correction does not close the terminal KL gap as effectively as SALD or
    VA-SALD.}
    \label{fig:app-doit-overview}
\end{figure}

\paragraph{DOIT \cite{zhu2026training} adaptation and budget matching.}
We adapt DOIT \cite{zhu2026training} (Opensource at \href{https://github.com/liamyzq/Doob_training_free_adaptation}{DOIT}) as a training-free local approximation to the Doob
$h$-transform.  For a base reverse transition density
$\phi_\theta(x_{\ell-1}\mid x_\ell)$ and terminal preference $h(x_0,0)$,
DOIT \cite{zhu2026training} estimates
\begin{equation}
    h(x_\ell,t_\ell)
    =
    \mathbb{E}\!\left[h(X_0,0)\mid X_{t_\ell}=x_\ell\right],
\end{equation}
and adds a correction proportional to $\nabla\log h$ to the base reverse
dynamics.  In our VP adaptation, DOIT \cite{zhu2026training} is not slowed down.  For a budget
label $r_b\in\{1,2,4,10,50,100\}$ we set
\begin{equation}\label{eq:role_rb_DOIT}
    N_{\mathrm{DOIT}}(r_b)
    =
    \left\lceil \frac{r_bT}{\eta_s}\right\rceil,
    \qquad
    \eta_{\mathrm{DOIT}}=\frac{T}{N_{\mathrm{DOIT}}(r_b)},
    \qquad
    t_k=k\eta_{\mathrm{DOIT}} .
\end{equation}
Thus DOIT \cite{zhu2026training} uses the same number of reverse transitions as SALD and VA-SALD
at budget $r_b$, but it traverses the original interval $[0,T]$ instead of
using $t=s/r_b$.

The base sampler for DOIT \cite{zhu2026training} is the original VP reverse SDE, not the
probability-flow velocity used in the VA-SALD derivation.  Its drift is
\begin{equation}
    b_{\mathrm{base}}(x,t)
    =
    \frac{\beta(T-t)}{2}x
    +\beta(T-t)\nabla\log p_t(x).
\end{equation}
For each particle, DOIT \cite{zhu2026training} draws $M$ local Gaussian proposals
\begin{equation}
    x_{k+1}^{(m)}
    =
    x_k
    +\eta_{\mathrm{DOIT}} b_{\mathrm{base}}(x_k,t_k)
    +\sqrt{\eta_{\mathrm{DOIT}}\beta(T-t_k)}\,z_m,
    \qquad z_m\sim\mathcal{N}(0,I),
\end{equation}
scores them with the reward $R(x)=-f(x)$, and forms Boltzmann weights
\begin{equation}
    w_m
    =
    \mathrm{softmax}_m
    \left(
    \frac{R(x_{k+1}^{(m)})-\max_j R(x_{k+1}^{(j)})}{\tau_R}
    \right).
\end{equation}
The corresponding local Doob direction is
\begin{equation}
    \widehat g_{\mathrm{Doob}}
    =
    \sum_{m=1}^M
    w_m
    \frac{z_m}{\sqrt{\eta_{\mathrm{DOIT}}\beta(T-t_k)}} ,
\end{equation}
and the adapted update is
\begin{equation}
    X_{k+1}
    =
    X_k
    +
    \eta_{\mathrm{DOIT}}
    \left(
    b_{\mathrm{base}}(X_k,t_k)
    +
    \gamma \beta(T-t_k)\widehat g_{\mathrm{Doob}}
    \right)
    +
    \sqrt{\eta_{\mathrm{DOIT}}\beta(T-t_k)}\,\xi_k .
\end{equation}
Since DOIT spends $M$ proposal evaluations per particle and step, a fair
particle-step comparison uses approximately $1/M$ as many DOIT particles as
SALD and VA-SALD particles.  In the notebook $M=4$, so DOIT uses $2500$
particles and has effective proposal-particle budget $2500\times 4=10000$,
matching the $10000$ particles used by SALD and VA-SALD.

\subsubsection{Eight-Gaussian mode-penalty VP diffusion.}
The second guided task uses an eight-component Gaussian mixture with centers
uniformly spaced on a circle around the origin.  With radius $R$ and angular
offset $\theta_0=\pi/8$, the centers are
\begin{equation}
    c_j
    =
    R\left(
    \cos\left(\theta_0+\frac{2\pi j}{8}\right),
    \sin\left(\theta_0+\frac{2\pi j}{8}\right)
    \right),
    \qquad j=0,\ldots,7,
\end{equation}
and
\begin{equation}
    p_{\mathrm{data}}(x)
    =
    \frac{1}{8}\sum_{j=0}^7 \varphi_2(x-c_j).
\end{equation}
Because the component covariance is the stationary covariance of the VP
diffusion, the marginals remain mixtures with decayed means,
\begin{equation}
    p_t(x)
    =
    \frac{1}{8}\sum_{j=0}^7
    \varphi_2\!\left(x-\alpha(T-t)c_j\right).
\end{equation}
The exact score used by all methods is
\begin{equation}
    \nabla\log p_t(x)
    =
    \sum_{j=0}^7
    \omega_j(x,t)\left(\alpha(T-t)c_j-x\right),
    \qquad
    \omega_j(x,t)
    =
    \frac{
    \varphi_2(x-\alpha(T-t)c_j)
    }{
    \sum_{\ell=0}^7
    \varphi_2(x-\alpha(T-t)c_\ell)
    } .
\end{equation}
The guide penalizes a chosen set $\mathcal{P}$ of modes.  We use the four
left-half modes, i.e. those with negative first coordinate, and define
\begin{equation}
    f_{\mathcal{P}}(x)
    =
    \lambda\sum_{j\in\mathcal{P}}
    \exp\left(
    -\frac{\|x-c_j\|^2}{2\ell_f^2}
    \right),
\end{equation}
with gradient
\begin{equation}
    \nabla f_{\mathcal{P}}(x)
    =
    -\frac{\lambda}{\ell_f^2}
    \sum_{j\in\mathcal{P}}
    \exp\left(
    -\frac{\|x-c_j\|^2}{2\ell_f^2}
    \right)(x-c_j).
\end{equation}
The guided terminal target is
$\pi_T(x)\propto p_T(x)\exp\{-f_{\mathcal{P}}(x)\}$, which shifts mass away
from the penalized modes while preserving the same VP marginal family and
exact mixture score.

\begin{figure}[h]
    \centering
    \includegraphics[width=0.72\textwidth]{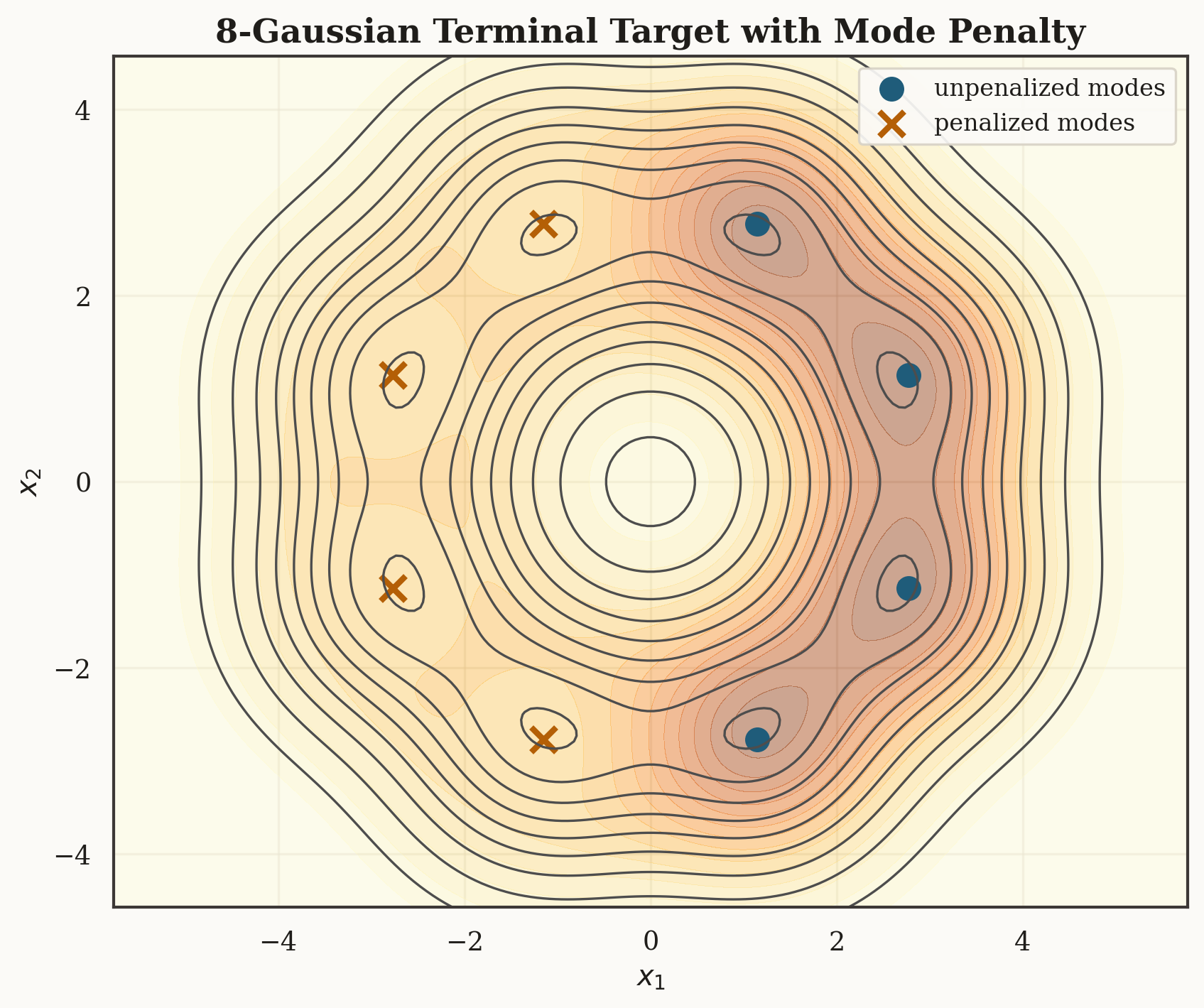}
    \caption{\textbf{Eight-Gaussian guided terminal target.}  Blue markers denote
    unpenalized modes, and orange crosses denote penalized modes.  Contours show
    the base and guided terminal densities.}
    \label{fig:app-eight-setup}
\end{figure}

\begin{figure}[h]
    \centering
    \includegraphics[width=0.92\textwidth]{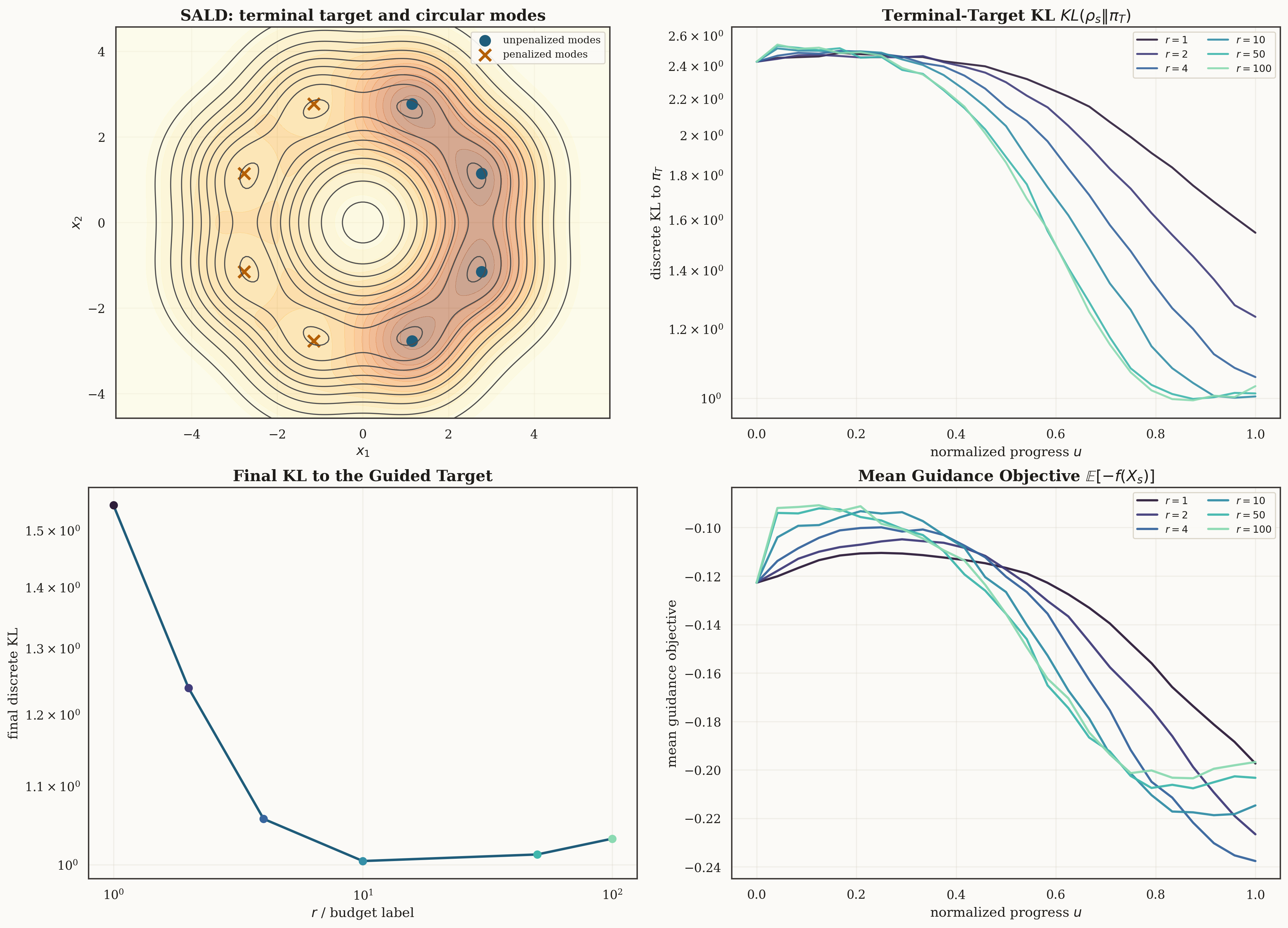}
    \caption{\textbf{SALD on the eight-Gaussian mode-penalty task.}  SALD reduces
    the terminal KL as $r$ increases and then approaches a plateau.}
    \label{fig:app-eight-sald}
\end{figure}

\begin{figure}[h]
    \centering
    \includegraphics[width=0.92\textwidth]{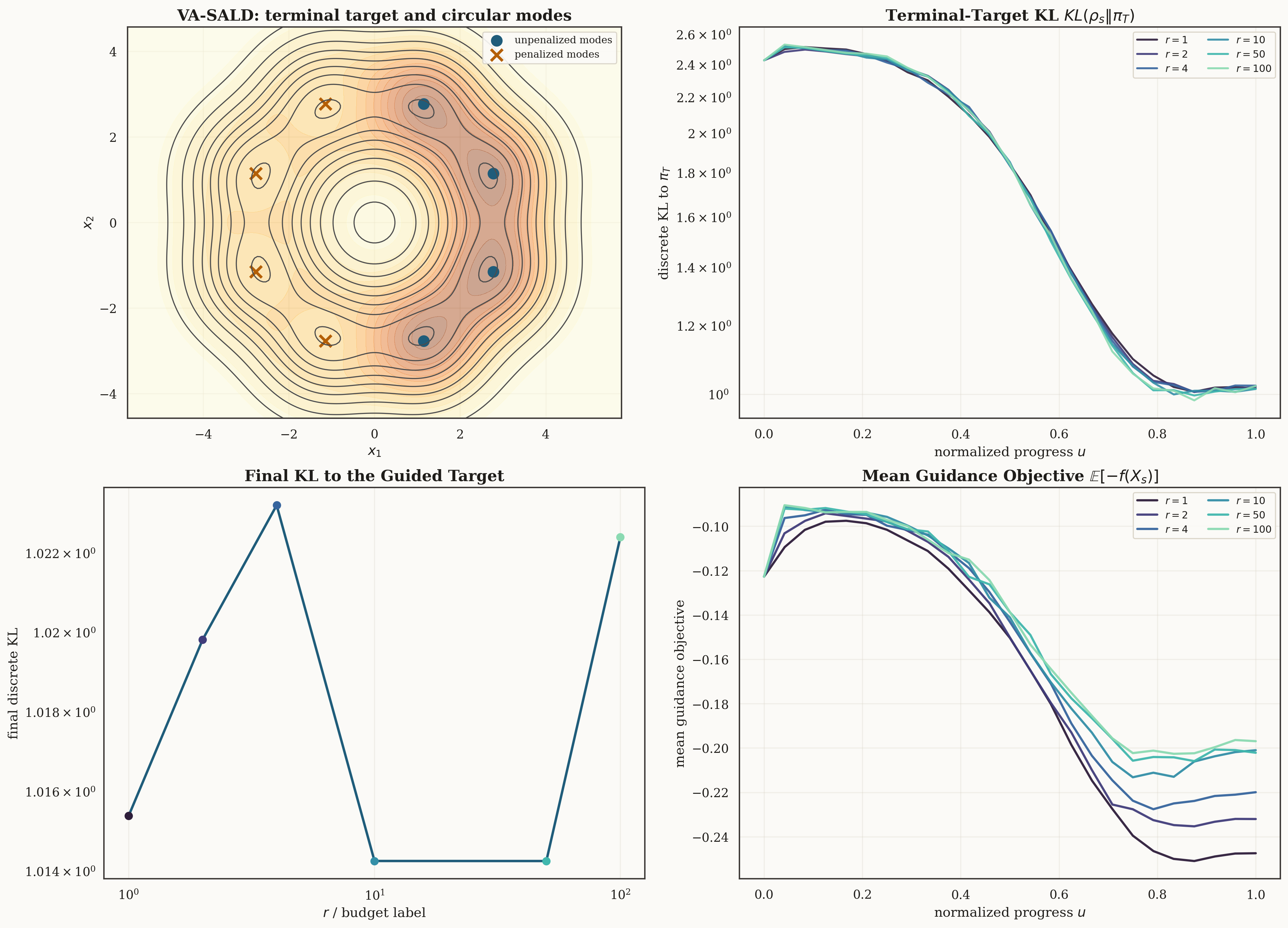}
    \caption{\textbf{VA-SALD on the eight-Gaussian mode-penalty task.}  The
    terminal KL is low and nearly budget-insensitive, indicating that the
    velocity-aware reverse dynamics already align well with this guided target.}
    \label{fig:app-eight-va}
\end{figure}

\begin{figure}[h]
    \centering
    \includegraphics[width=0.92\textwidth]{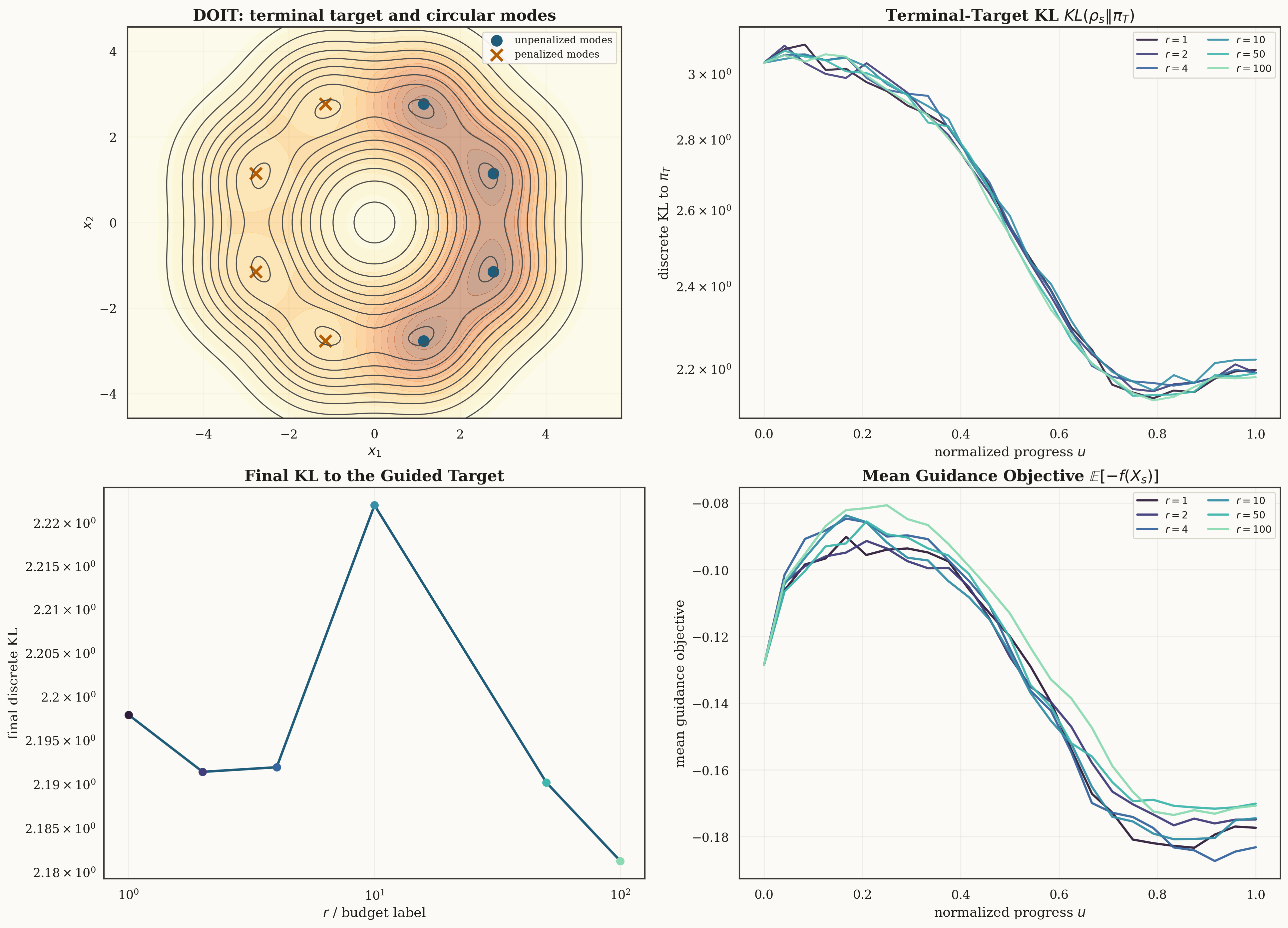}
    \caption{\textbf{DOIT on the eight-Gaussian mode-penalty task.}  The local
    proposal-based Doob correction improves the guide-related statistic but
    leaves a larger terminal KL than SALD and VA-SALD at the same computational
    budgets.}
    \label{fig:app-eight-doit}
\end{figure}

\paragraph{Summaries.}
Tables~\ref{tab:synthetic-kl} and~\ref{tab:synthetic-budget} summarize the
terminal metrics.  The effective DOIT particle
budget is reported as the number of particles times the number of local
proposals, so that it is directly comparable to the SALD and VA-SALD particle
counts.

\begin{table}[h]
    \centering
    \scriptsize
    \caption{\textbf{Terminal KL.}  Budgets are
    $r\in\{1,2,4,10,50,100\}$.  Smaller values are better.}
    \label{tab:synthetic-kl}
    \begin{tabular}{llrrrrrr}
        \hline
        Task & Method & $r=1$ & $r=2$ & $r=4$ & $r=10$ & $r=50$ & $r=100$ \\
        \hline
        Two moons & SALD & 0.715 & 0.592 & 0.525 & 0.516 & 0.523 & 0.511 \\
        Two moons & VA-SALD & 0.556 & 0.549 & 0.546 & 0.528 & 0.527 & 0.518 \\
        Two moons & DOIT & 1.267 & 1.272 & 1.265 & 1.257 & 1.268 & 1.270 \\
        Eight Gaussian & SALD & 1.547 & 1.239 & 1.057 & 1.005 & 1.013 & 1.032 \\
        Eight Gaussian & VA-SALD & 1.015 & 1.020 & 1.023 & 1.014 & 1.014 & 1.022 \\
        Eight Gaussian & DOIT & 2.198 & 2.191 & 2.192 & 2.222 & 2.190 & 2.181 \\
        \hline
    \end{tabular}
\end{table}

\begin{table}[h]
    \centering
    \scriptsize
    \caption{\textbf{Budget accounting and representative terminal guide
    statistics.}  The table reports the final budget $r=100$, the number of
    Euler steps, the particle count used by the sampler, and the effective
    particle count used for budget matching.  The guide statistic is the
    terminal mean penalty.}
    \label{tab:synthetic-budget}
    \begin{tabular}{llrrrr}
        \hline
        Task & Method & Steps & Particles & Effective particles & Mean penalty \\
        \hline
        Two moons & SALD & 100000 & 10000 & 10000 & 2.661 \\
        Two moons & VA-SALD & 100000 & 10000 & 10000 & 2.662 \\
        Two moons & DOIT & 100000 & 2500 & 10000 & 2.591 \\
        Eight Gaussian & SALD & 100000 & 10000 & 10000 & 0.197 \\
        Eight Gaussian & VA-SALD & 100000 & 10000 & 10000 & 0.197 \\
        Eight Gaussian & DOIT & 100000 & 2500 & 10000 & 0.171 \\
        \hline
    \end{tabular}
\end{table}

\clearpage
\subsection{Additional Details of Guided Image Generation}\label{app:additional_stablediff}

\subsubsection{Apply VA-SALD to Flow-Matching Setting}

We employ the pre-trained Stable-Diffusion 3.5 Medium (SD-3.5M)~\cite{esser2024scaling} trained by the Flow Matching objective as the backbone to access the velocity and the base score function for VA-SALD.

Specifically,  given the forward  SDE in Eq.~(\ref{SDE_FlowMatching}) that generates the same marginal $p_\tau$ w.r.t the deterministic ODE in Eq~(\ref{ODE_FlowMatching}) of flow-matching~\cite{liu2025flow},
\begin{equation}
\label{ODE_FlowMatching}
    \dd \boldsymbol{x}_\tau = \boldsymbol{v}_\tau \dd \tau
\end{equation}
\begin{equation}
\label{SDE_FlowMatching}
    \dd \boldsymbol{x}_\tau = ( \boldsymbol{v}_\tau \big(\boldsymbol{x}_\tau) + \frac{\sigma_\tau^2}{2} \nabla \log p_\tau(\boldsymbol{x}_\tau) \big)\dd  \tau + \sigma_\tau  \dd  W_\tau
\end{equation}
the score function is 
\begin{equation}
    \nabla \log p_\tau(\boldsymbol{x}_\tau) = - \frac{\boldsymbol{x}_\tau}{\tau} - \frac{1-\tau}{\tau} \boldsymbol{v}_\tau (\boldsymbol{x}_\tau)
\end{equation}

Note that $t= T-\tau $ (with $T=1$ in flow-matching), 
the SDE of our VA-SALD for guidance is  
\begin{align}
    \label{eq:guided_sald_FlowMatching}
   \dd X_s
& =
\left(
-\dot t(s)B_{t(s)}(X_s)
+
\dot t(s)\frac{\sigma_{t(s)}^2}{2} 
\nabla\log p_{t(s)}(X_s)  + \frac{\sigma_{t(s)}^2}{2} \nabla\log \pi_{t(s)}(X_s)
\right)\dd s
+
\sigma_{t(s)} \dd W_s.  \\
& = \left( -\dot t(s) \boldsymbol{v}_{t(s)} (X_s) + \frac{\sigma_{t(s)}^2}{2} \nabla\log \pi_{t(s)}(X_s) \right) \dd s +
\sigma_{t(s)} \dd W_s
\end{align}
where $\pi_{t(s)}(X_s) \propto p_t(X_s) \exp{(- c \; f_{t(s)}(X_s))}$ and 
$B_{t(s)}(X_s) = \boldsymbol{v}_{t(s)} (X_s)  + \frac{\sigma_{t(s)}^2}{2} \nabla \log p_{t(s)}(X_s) $.

In our experiments, we employ $t= \frac{s}{r}$ and then $\dot t(s) =\frac{1}{r}$, and we set $\sigma_t = (1-t) \sigma_0$. We set $\sigma_0 = 0.7$ same as the noise level used in \cite{liu2025flow}.

Finally,  the SDE of our VA-SALD for guidance used in our experiments  is 
\begin{align}\label{VA_SALD_FlowMatching_SDE}
    \dd X_s & = \left ( - \frac{\sigma_t^2}{2(1-t)} X_{s} - \big( \frac{1}{r} + \frac{t \sigma_t^2}{2(1-t)} \big) \boldsymbol{v}_{t(s)} (X_s)  - \frac{ c  \sigma_t^2 }{2} \nabla f_{t(s)} (X_s)  \right ) \dd s + \sigma_{t(s)} \dd W_s \\ &  = \left ( - \frac{(1-t)\sigma_0^2}{2} X_s -\big( \frac{1}{r} + \frac{t(1-t)\sigma_0^2}{2}\big) \boldsymbol{v}_{t(s)} (X_s) - \frac{ c (1-t)^2 \sigma_0^2 }{2} \nabla f_{t(s)} (X_s) \right) \dd s + (1-t)\sigma_0 \dd W_s
\end{align}


\subsubsection{Construction of $\pi_t$ for Black-box Guidance}
To apply our VA-SALD on the appealing inference-time black-box guided generation task,  we construct the moving distribution sequence $\pi_t$  as  $\pi_t \propto  p_t  \exp{(-c \; f_t)}  $   with  Gaussian-smooth reward functions $f_t(\boldsymbol{x}) = \mathbb{E}_{\mathcal{N}(\boldsymbol{0},\boldsymbol{I})} [ f(\boldsymbol{x} + \bar{\sigma} _t \boldsymbol{\epsilon}) ]  $. And $\bar{\sigma} _T  = 0$ leads to $f_T = f$ to achieve the true reward function $f$. 

The gradient $\nabla f_t(\boldsymbol{x})$ can be approximated by the zeroth-order gradient estimator Eq.~(\ref{BO})~\cite{nesterov2017random}:  
\begin{align}
\label{BO}
    \nabla f_t(\boldsymbol{x}) & = \frac{1}{\bar{\sigma} _t } \mathbb{E}_{\mathcal{N}(\boldsymbol{0},\boldsymbol{I})} [ f(\boldsymbol{x} + \bar{\sigma} _t \boldsymbol{\epsilon}) \boldsymbol{\epsilon} ]  \approx \frac{1}{N \bar{\sigma} _t } \sum_{i=1}^N f(\boldsymbol{x} + \bar{\sigma} _t  \boldsymbol{\epsilon}^i) \boldsymbol{\epsilon}^i ,
\end{align}
where $\boldsymbol{\epsilon}^i$ for all $i \in \{1,\cdots,N\} $ denote i.i.d. samples from $\mathcal{N}(\boldsymbol{0},\boldsymbol{I})$.
 We employ the group reward normalization technique introduced in \cite{lyu2019black} to normalize the reward as 
 $\hat{f}(\boldsymbol{x}^i) = \frac{f(\boldsymbol{x}^i) - \hat{\mu}}{\hat{\sigma}}  $, where $\boldsymbol{x}^i= \boldsymbol{x} + \bar{\sigma} _t \boldsymbol{\epsilon}^i$, $\hat{\mu}$ and $\hat{\sigma}$ denote the mean and std of the rewards among the group samples.

\subsubsection{Discretization Update} 

Apply the Euler-Maruyama discretization to the SDE in Eq.(\ref{VA_SALD_FlowMatching_SDE}) with stepsize $\eta$, we have
\begin{align}
\label{VA_SALD_FlowMatching_SDE_discrete}
    X_{(k+1)\eta} = \left(1 -  \frac{\sigma_t^2}{2(1-t)} \eta \right) X_{k\eta}  -  \left( \frac{1}{r} + \frac{t \sigma_t^2}{2(1-t)} \right)  \eta \; \boldsymbol{v}_{t} (X_{k\eta})  - \frac{ c  \sigma_t^2 \eta}{2} \nabla f_{t} (X_{k\eta})   + \sigma_t \sqrt{\eta} \; \boldsymbol{w}
\end{align}
with $t = \frac{k\eta}{r}$ and $\boldsymbol{w}$ sampling from $\mathcal{N}(\boldsymbol{0},\boldsymbol{I})$.

In our experiments , we set the $\bar{\sigma} _t$ in zeroth-order gradient estimator in Eq.(\ref{BO}) as $\bar{\sigma} _t  = \sigma_t \sqrt{\eta}$. 
Plug the zeroth-order gradient estimator in Eq.(\ref{BO}) into Eq.~(\ref{VA_SALD_FlowMatching_SDE_discrete}), we achieve the discrete update rule. 
\begin{align}
    X_{(k+1)\eta} 
    &= \left(1 -  \frac{\sigma_t^2}{2(1-t)} \eta \right) X_{k\eta}  -  \left( \frac{1}{r} + \frac{t \sigma_t^2}{2(1-t)} \right)  \eta \; \boldsymbol{v}_{t} (X_{k\eta}) \notag\\ 
    &- \frac{ c  \sigma_t \sqrt{\eta}  }{2N} \sum_{i=1}^N f (X_{k\eta} + \sigma_t \sqrt{\eta}\boldsymbol{\epsilon}^i) \boldsymbol{\epsilon}^i  + \sigma_t \sqrt{\eta} \; \boldsymbol{w}
\end{align}
with $t = \frac{k\eta}{r}$ and $\boldsymbol{w}$ sampling from $\mathcal{N}(\boldsymbol{0},\boldsymbol{I})$.

We employ the group normalized reward and fix stepsize $\eta = 0.025$ and batch query size $N= 32$ in our experiments.  Given stepsize $\eta$, the number of steps of VA-SALD is $K= \frac{r}{\eta}$. 

\subsubsection{Baselines}

We compare our VA-SALD with two closely-related baselines that employ zeroth-order gradient: FM-ZG and FM-Evolv~\cite{wei2025evolvable}
\begin{itemize}
    \item \textbf{FM-ZG}:  Apply the default (Flow-Matching) SDE sampler in \cite{liu2025flow} with the plug-in zeroth-order gradient estimator (the same one used in VA-SALD) in the drift term. The std parameter $\bar{\sigma} _t$ of the zeroth-order gradient estimator is set to the same value as the coefficient of the Gaussian variable $\boldsymbol{w}$, which is the same setting scheme as VA-SALD. 

    \item \textbf{FM-Evolv}: Apply the default (Flow-Matching) SDE sampler in \cite{liu2025flow} with a posterior  zeroth-order gradient update at each inference step~\cite{wei2025evolvable}.  The std parameter $\bar{\sigma} _t$ of the zeroth-order gradient estimator is set to the same value as the coefficient of the Gaussian variable $\boldsymbol{w}$, which is the same setting scheme as VA-SALD. 
\end{itemize}

In all the experiments,  we keep the number of steps and batch query size of all baselines the same as VA-SALD.

\subsubsection{Guided Function}

We employ Aesthetic Score, PickScore, and CLIP Score as the guided function.  In addition, we compute the mean pair-wise distance of  CLIP embedding as the Diversity score, same as in \cite{liu2025flow}.

The Hugging Face link of the pre-trained SD3.5-M model and the Guided Function used in our experiments are listed below:
\begin{table}[h]
\centering
\begin{tabular}{|l|l|}
\hline
SD3.5-M         & https://huggingface.co/stabilityai/stable-diffusion-3.5-medium \\ \hline
Aesthetic Score & https://huggingface.co/trl-lib/ddpo-aesthetic-predictor        \\ \hline 
PickScore  &    https://huggingface.co/yuvalkirstain/PickScore\_v1              \\  \hline
\end{tabular}
\end{table}

\subsubsection{Experiments Setup}

We evaluate our VA-SALD on four generation tasks with complex prompts. Details of the prompt used are listed in the Table~\ref{FourPromptTask}.
\begin{table}[h]
\centering
\caption{Four generation tasks with complex prompts.}
\resizebox{\textwidth}{!}{
\begin{tabular}{|l|l|}
\hline
Task\#1  &    "The art deco music festival poster has a fox made of polished brass in the center of the picture, with smooth lines and elegant posture"              \\  \hline
Task\#2        & "A glowing dragon soaring through floating islands, leaving behind a trail of shimmering stardust" \\ \hline
Task\#3 & "quick doodle of a guy, medium hair with long bangs, hd detailed detailed"       \\ \hline 
Task\#4  &    "In the ink painting style, a naive giant panda is sitting on the majestic Great Wall, leisure lychewing bamboo"              \\  \hline
\end{tabular}
}\label{FourPromptTask}
\end{table}

For VA-SALD, and all baselines FM-ZG, FM-Evolv, we perform $10$ independent runs with seed in $\{0,1,\cdots,9\}$ for each guided function.  In all experiments, we set the guidance scale $c=8$.  For VA-SALD, we fix $r=4$, stepsize $\eta = 0.025$, which leads to the number of steps $K=160$. For all baselines, we use the same number of steps $K=160$.

The experimental results are shown in Table~\ref{ResultsT1} to Table~\ref{ResultsT4},  respectively. We can observe that VA-SALD achieves higher Aesthetic Scores and PickScores in this large guidance strength setting ($c=8$).  The generated images with  Aesthetic Score guidance are shown in Figure~\ref{fig:dragon} to Figure~\ref{fig:Panda}. We can see that the generated images of the baselines begin to degenerate, indicating that the heuristic guidance update is more likely to lead to over-optimization and to generation far from the image data manifold.  

 VA-SALD still generates good quality images. This is because VA-SALD theoretically converges to the distribution $\pi_T  \propto p_{0} \cdot \exp{(-c f)}$ for a given pre-trained model distribution $p_0$.  The  distribution $\pi_T$  is the optimal solution of the following regularized optimization problem
\begin{align}
\label{OBj_KL}
  \min_{p \in \mathcal{P}} \left \{ \mathbb{E}_p [f(X)] + \frac{1}{c}\textbf{KL}(p||p_0)   \right \},
\end{align}
which well-balanced the optimization for guidance $f$ and the KL-divergence to the pre-trained distribution $p_0$.

In contrast, baselines employ a heuristic guidance update, which leads to an unknown distribution after guidance. The resulting distribution of baselines is inferior compared with VA-SALD  for  the above optimization problem~(in Eq.~(\ref{OBj_KL})) 


\begin{table}[t]
\caption{On prompt "The art deco music festival poster has a fox made of polished brass in the center of the picture, with smooth lines and elegant posture" with guidance $c=8$}
\resizebox{\textwidth}{!}{
\begin{tabular}{|c|cc|cc|cc|}
\hline
         & \multicolumn{1}{c|}{Aesthetic} & DiversityScore & \multicolumn{1}{c|}{pickscore} & DiversityScore & \multicolumn{1}{c|}{CLIPscore} & DiversityScore \\ \hline
VA-SALD  & 6.061 $\pm$ 0.479                  & 13.941         & 24.207 $\pm$ 0.764                 & 13.636         & 35.264 $\pm$ 3.441                 & 11.461         \\ \hline
FM-ZG    & 5.977 $\pm$ 0.258                  & 12.405         & 23.079 $\pm$ 0.467                      & 13.339               & 33.733 $\pm$ 2.941                 & 14.058         \\ \hline
FM-Evolv & 5.547 $\pm$ 0.323                  & 10.693         & 20.795 $\pm$ 0.688                 & 12.136         & 32.089 $\pm$ 3.086                 & 11.847         \\ \hline
\end{tabular}
}\label{ResultsT1}
\end{table}

\begin{table}[t]
\caption{On prompt "A glowing dragon soaring through floating islands, leaving behind a trail of shimmering stardust" with guidance $c=8$}
\resizebox{\textwidth}{!}{
\begin{tabular}{|c|cc|cc|cc|}
\hline
         & \multicolumn{1}{c|}{Aesthetic} & DiversityScore & \multicolumn{1}{c|}{pickscore} & DiversityScore & \multicolumn{1}{c|}{CLIPscore} & DiversityScore \\ \hline
VA-SALD  & 6.586 $\pm$ 0.140                  & 7.393          & 24.526 $\pm$ 0.433                 & 7.466          & 29.619 $\pm$ 1.187                 & 7.718          \\ \hline
FM-ZG    & 6.291 $\pm$ 0.325                  & 9.929          & 22.201 $\pm$ 0.537                 & 10.220         & 29.729 $\pm$ 0.714                 & 8.889          \\ \hline
FM-Evolv & 5.633 $\pm$ 0.590                  & 13.498         & 21.420 $\pm$ 0.528                 & 10.456         & 31.794 $\pm$ 1.916                 & 10.367         \\ \hline
\end{tabular}
}\label{ResultsT2}
\end{table}

\begin{table}[!t]
\caption{On prompt "quick doodle of a guy, medium hair with long bangs, hd detailed detailed" with guidance $c=8$}
\resizebox{\textwidth}{!}{
\begin{tabular}{|c|cc|cc|cc|}
\hline
         & \multicolumn{1}{c|}{Aesthetic} & DiversityScore & \multicolumn{1}{c|}{pickscore} & DiversityScore & \multicolumn{1}{c|}{CLIPscore} & DiversityScore \\ \hline
VA-SALD  & 6.309 $\pm$ 0.109                  & 10.078         & 21.728 $\pm$ 0.327                 & 10.901         & 24.858 $\pm$ 1.353                 & 9.767          \\ \hline
FM-ZG    & 6.091 $\pm$ 0.130                  & 11.205         & 21.606 $\pm$ 0.299                 & 10.609         & 24.127 $\pm$ 1.047                 & 10.740         \\ \hline
FM-Evolv & 6.048$\pm$ 0.249                   & 11.635         & 21.158 $\pm$ 0.688                 & 11.916         & 25.569 $\pm$ 1.535                 & 10.735         \\ \hline
\end{tabular}
}\label{ResultsT3}
\end{table}

\begin{table}[!t]
\caption{On prompt "In the ink painting style, a naive giant panda is sitting on the majestic Great Wall, leisure lychewing bamboo" with guidance $c=8$}
\resizebox{\textwidth}{!}{
\begin{tabular}{|c|cc|cc|cc|}
\hline
         & \multicolumn{1}{c|}{Aesthetic} & DiversityScore & \multicolumn{1}{c|}{pickscore} & DiversityScore & \multicolumn{1}{c|}{CLIPscore} & DiversityScore \\ \hline
VA-SALD  & 6.697 $\pm$ 0.153                  & 6.977          & 22.821 $\pm$ 0.318                 & 7.378          & 36.541 $\pm$ 0.902                 & 7.411          \\ \hline
FM-ZG    & 6.337 $\pm$ 0.148                  & 8.485          & 21.529 $\pm$ 0.271                 & 7.967          & 34.660 $\pm$ 1.019                 & 8.423          \\ \hline
FM-Evolv & 5.598$\pm$ 0.240                   & 9.135          & 20.543 $\pm$ 0.460                 & 9.821          & 35.333 $\pm$ 1.905                 & 9.701          \\ \hline
\end{tabular}
}\label{ResultsT4}
\end{table}

\clearpage

\begin{figure}
     \centering
     \subfloat[VA-SALD]{\includegraphics[width=0.85\textwidth]{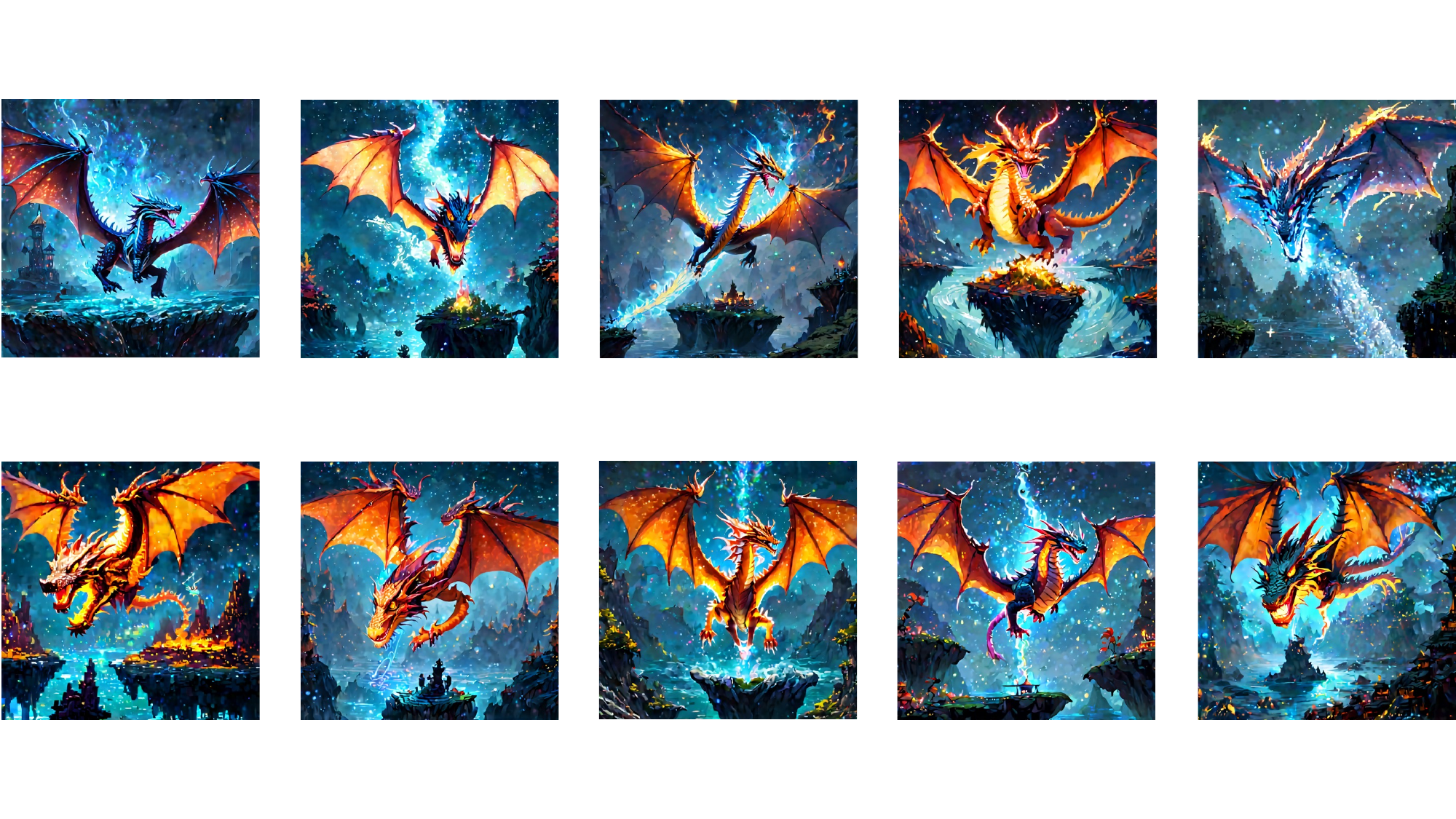}}    
      \hfill
      \subfloat[FM-ZG]{\includegraphics[width=0.85\textwidth]{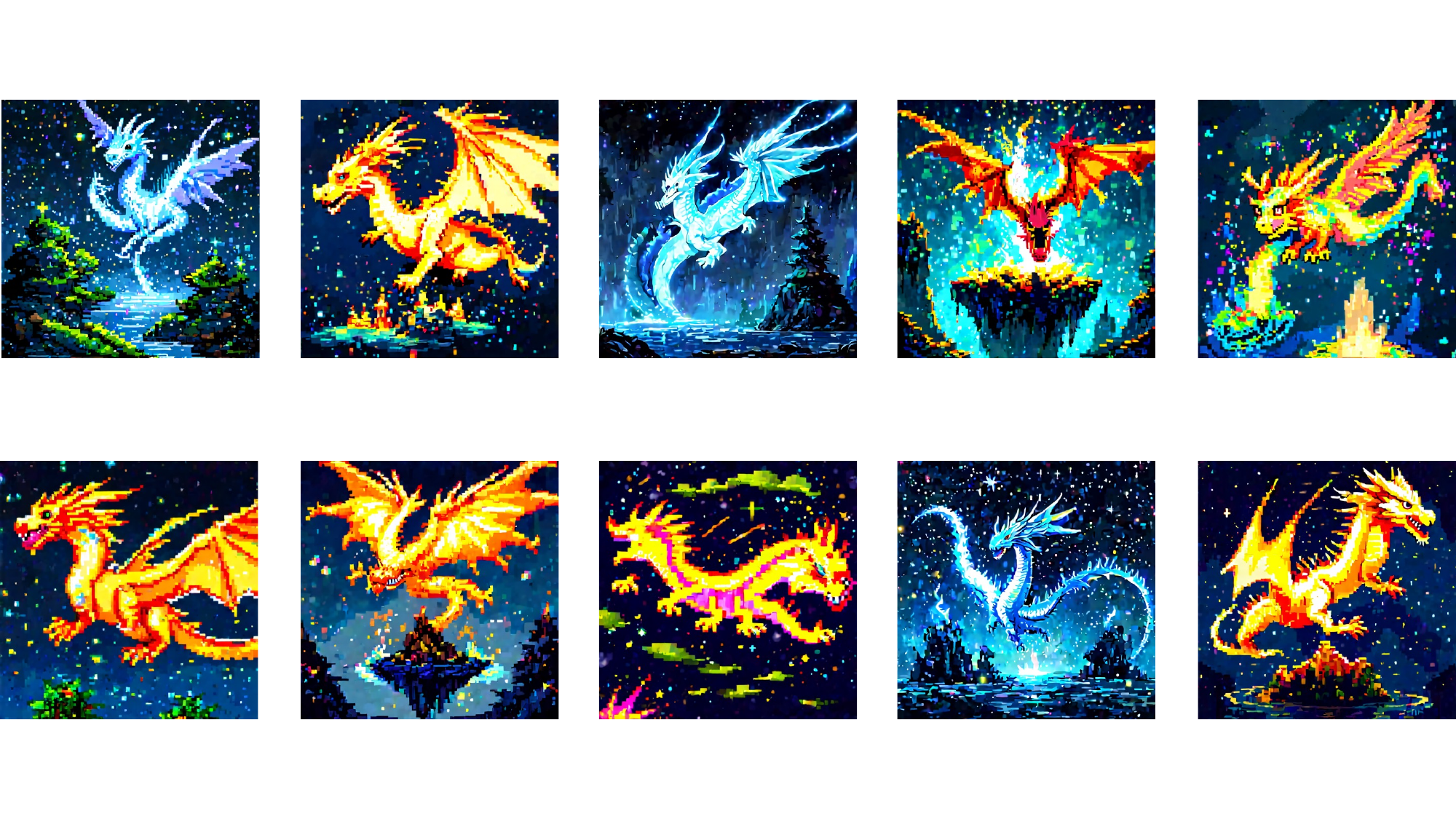}}  
      \hfill
       \subfloat[FM-Evolv]{\includegraphics[width=0.85\textwidth]{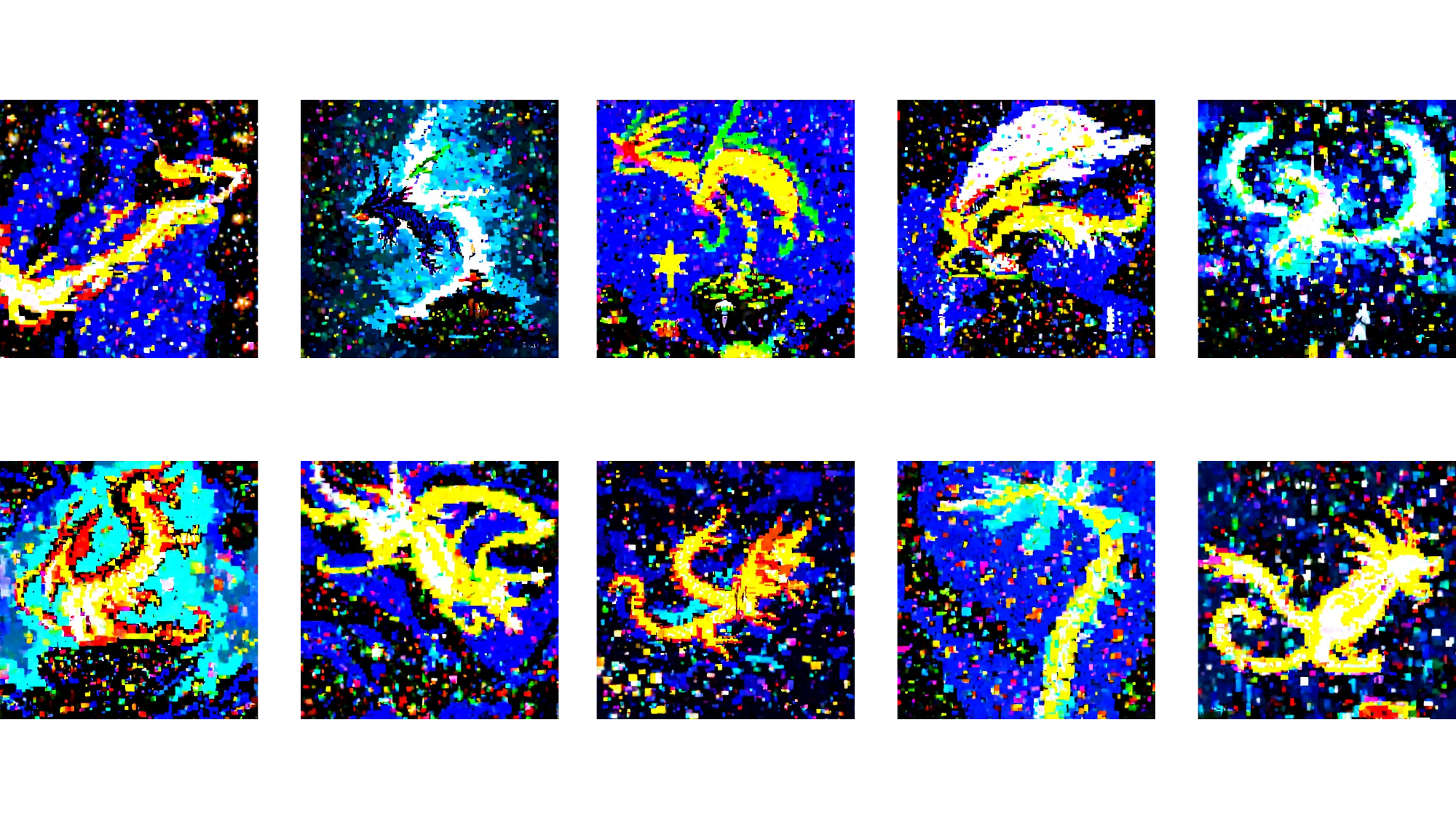}} 
        \caption{Generated Images on prompt "A glowing dragon soaring through floating islands, leaving behind a trail of shimmering stardust" with guidance $c=8$}
        \label{fig:dragon}
\end{figure}

\begin{figure}
     \centering
     \subfloat[VA-SALD]{\includegraphics[width=0.85\textwidth]{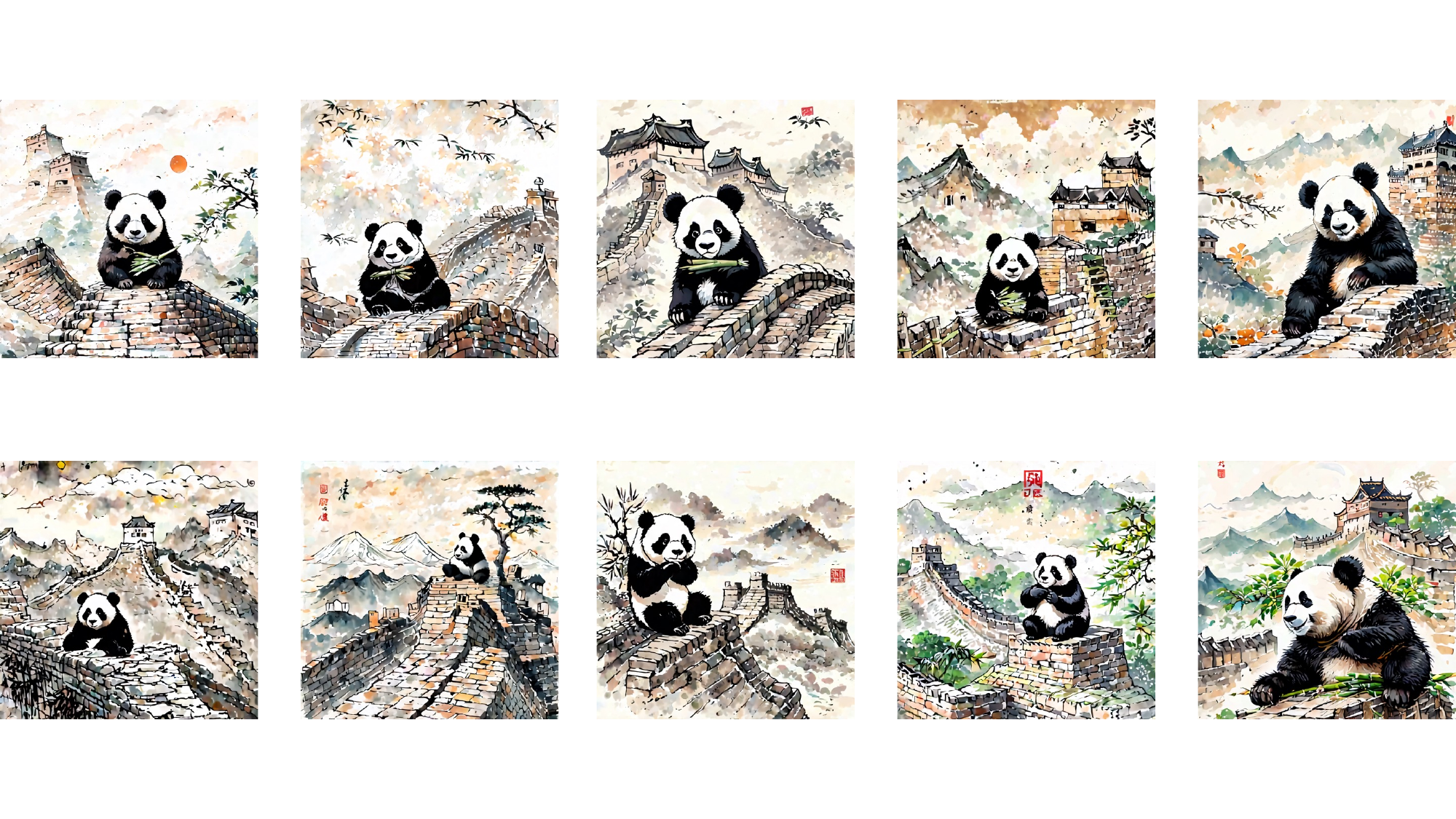}}    
      \hfill
      \subfloat[FM-ZG]{\includegraphics[width=0.85\textwidth]{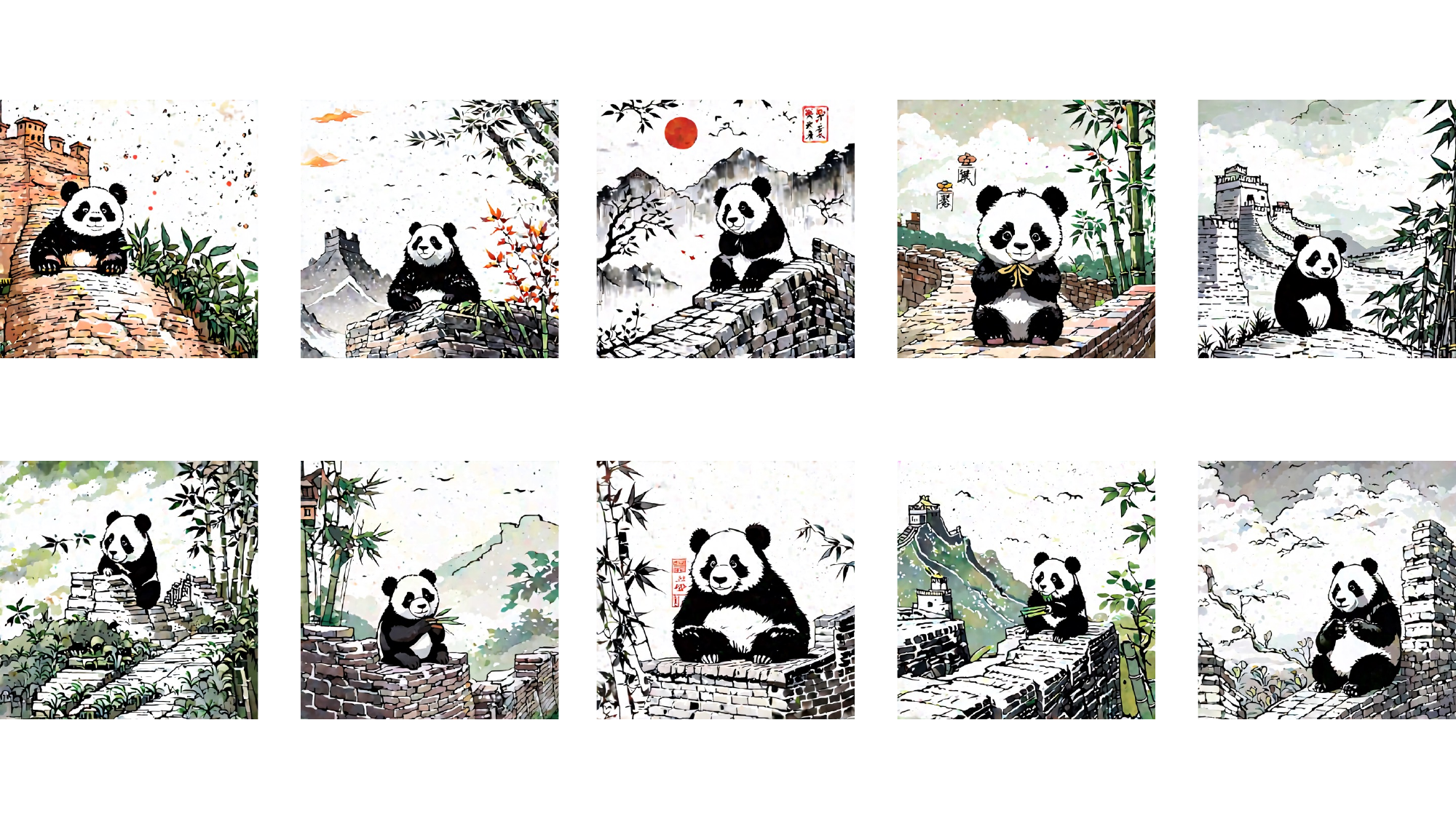}}  
      \hfill
       \subfloat[FM-Evolv]{\includegraphics[width=0.85\textwidth]{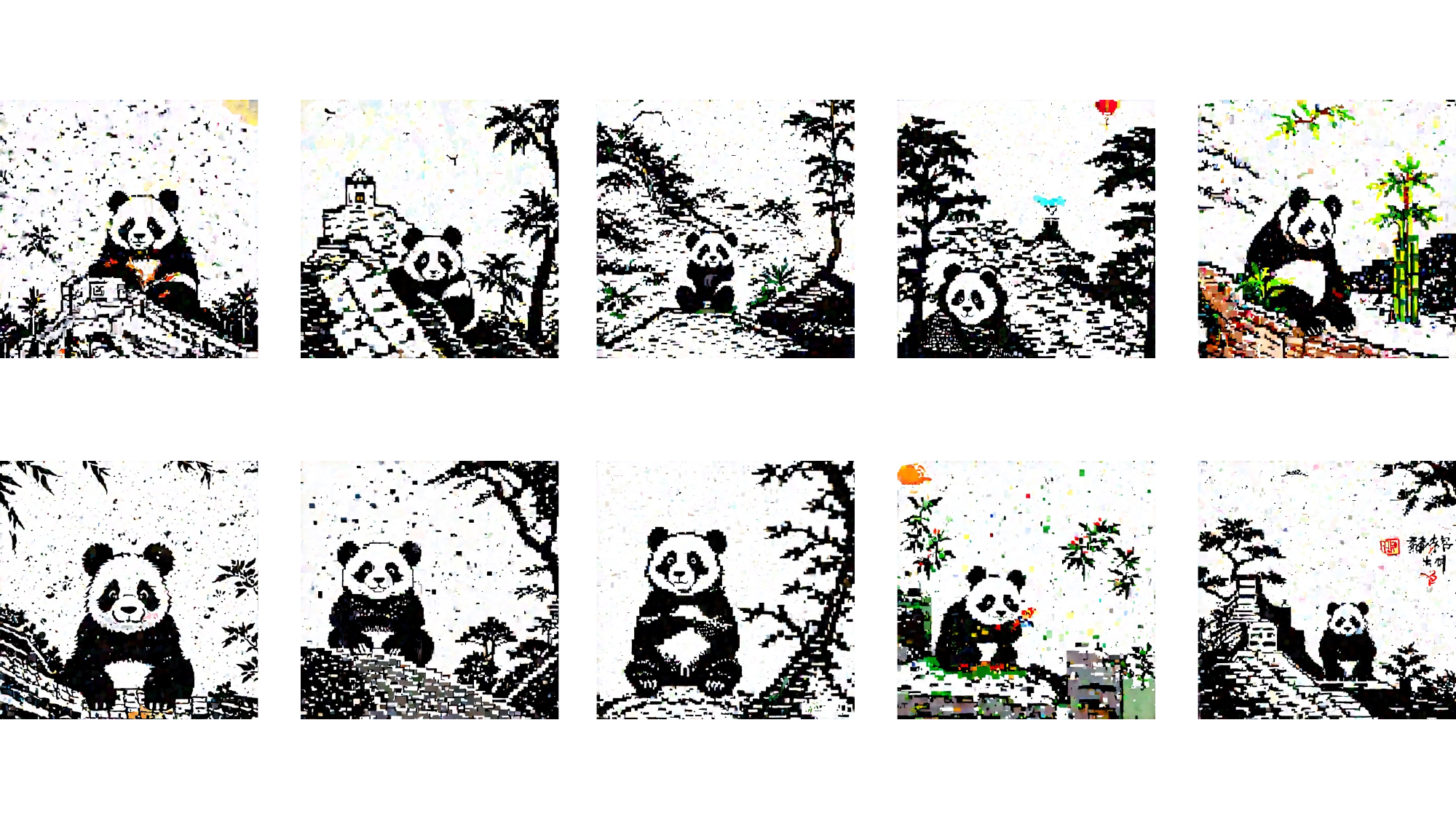}} 
        \caption{Generated Images on prompt "In the ink painting style, a naive giant panda is sitting on the majestic Great Wall, leisure lychewing bamboo" with guidance $c=8$}
        \label{fig:Panda}
\end{figure}

\begin{figure}
     \centering
     \subfloat[VA-SALD]{\includegraphics[width=0.85\textwidth]{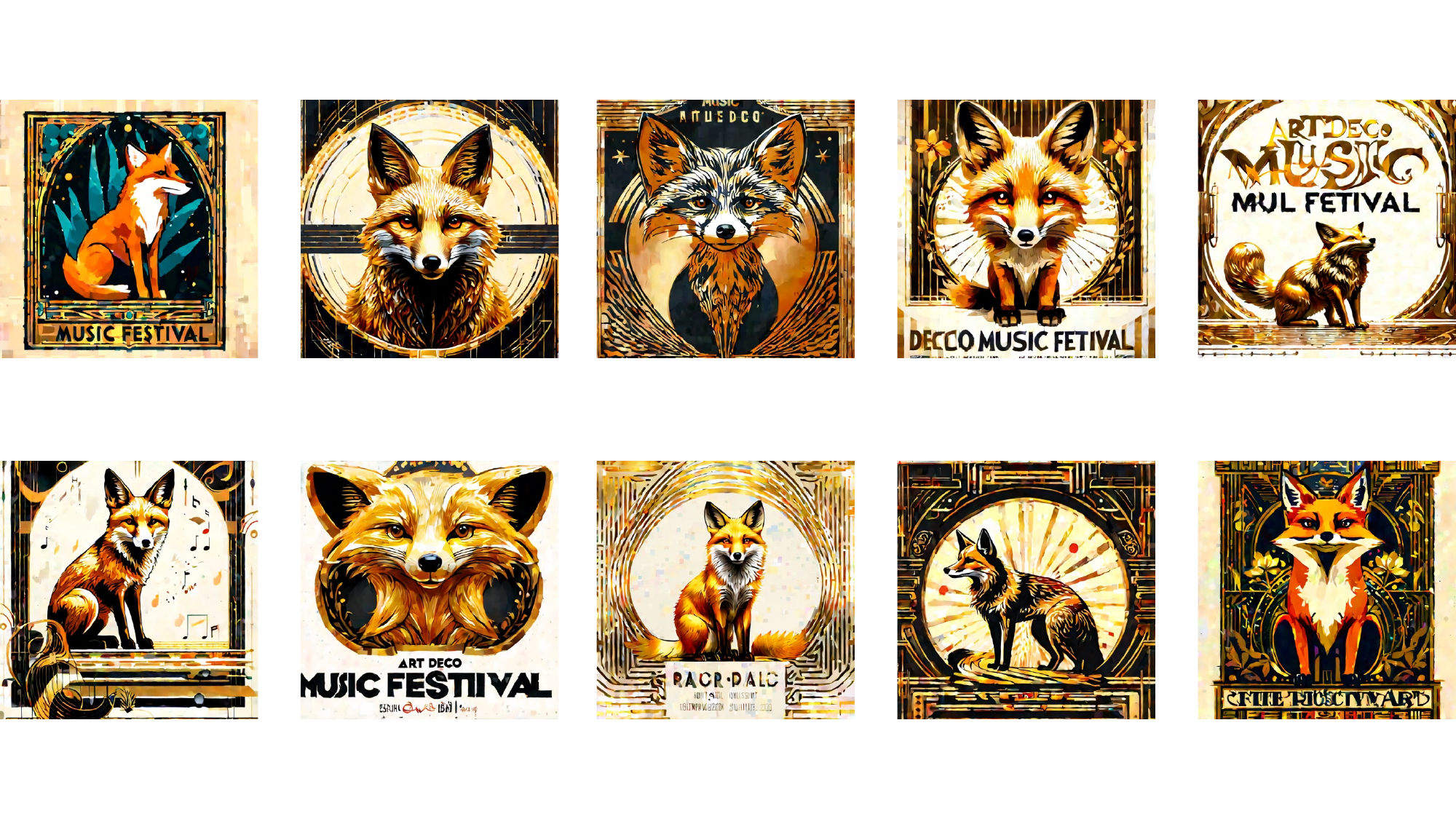}}    
      \hfill
      \subfloat[FM-ZG]{\includegraphics[width=0.85\textwidth]{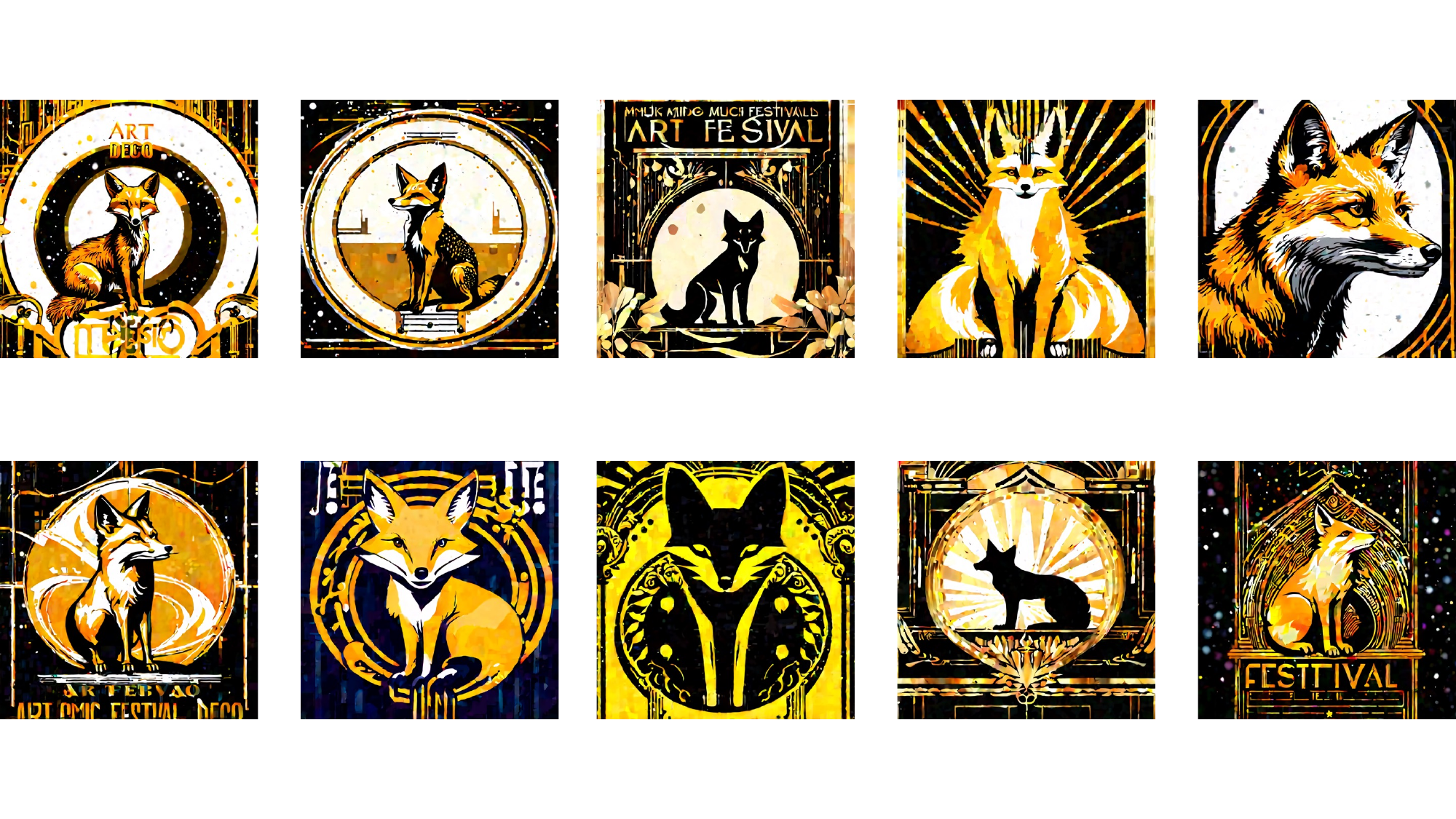}}  
      \hfill
       \subfloat[FM-Evolv]{\includegraphics[width=0.85\textwidth]{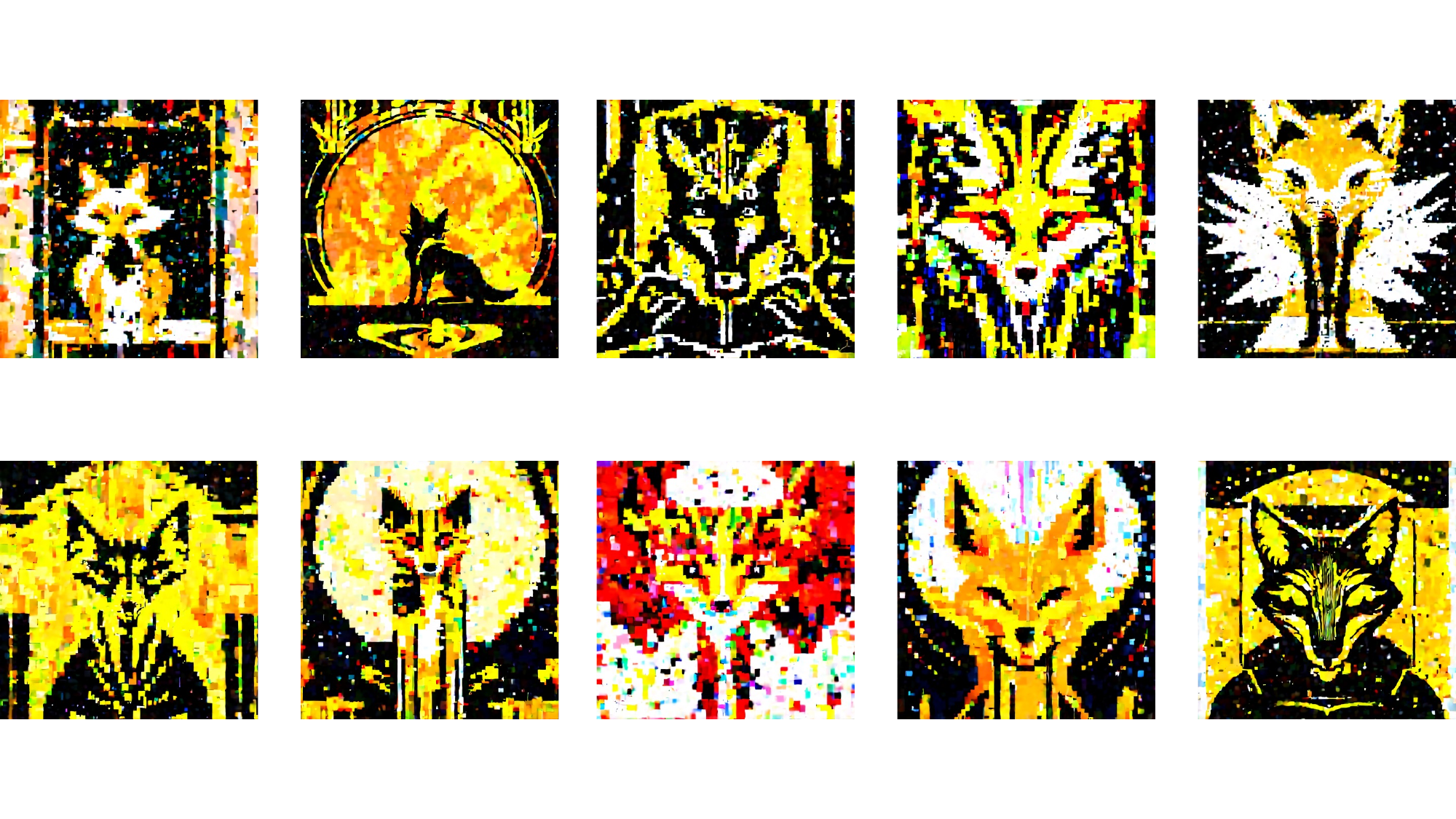}} 
        \caption{Generated Images on prompt "The art deco music festival poster has a fox made of polished brass in the center of the picture, with smooth lines and elegant posture" with guidance $c=8$}
        \label{fig:Fox}
\end{figure}

\begin{figure}
     \centering
     \subfloat[VA-SALD]{\includegraphics[width=0.85\textwidth]{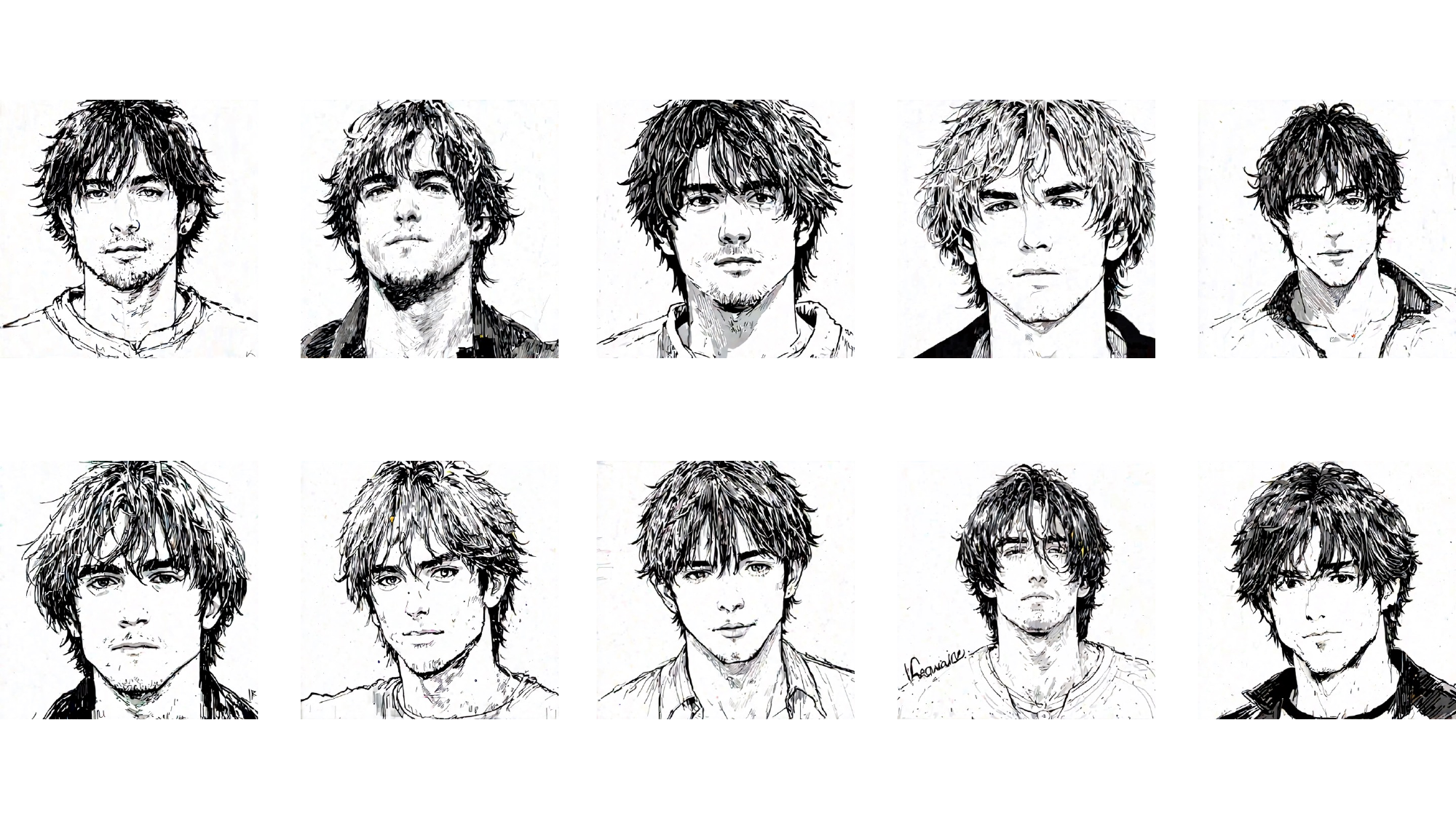}}    
      \hfill
      \subfloat[FM-ZG]{\includegraphics[width=0.85\textwidth]{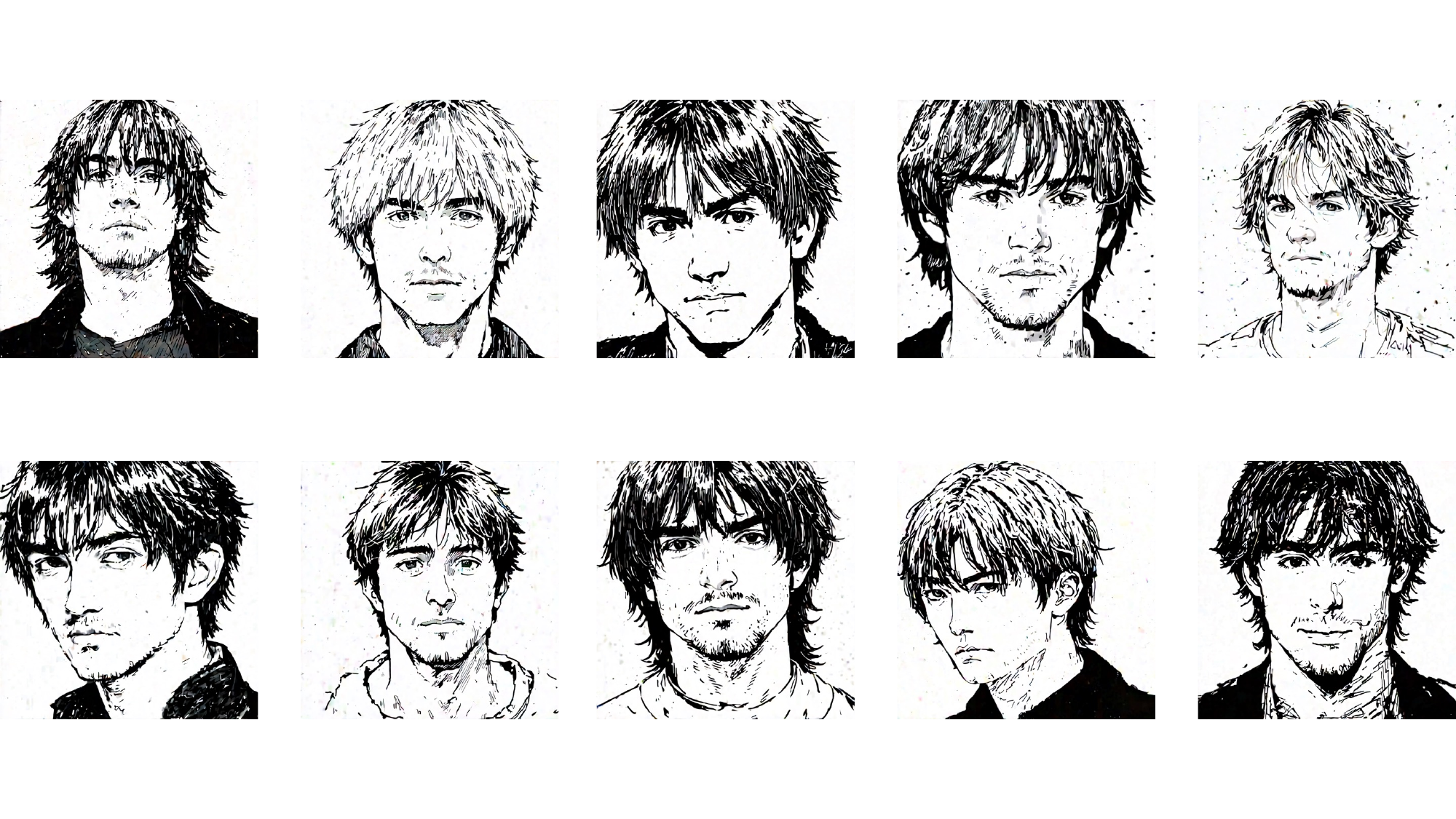}}  
      \hfill
       \subfloat[FM-Evolv]{\includegraphics[width=0.85\textwidth]{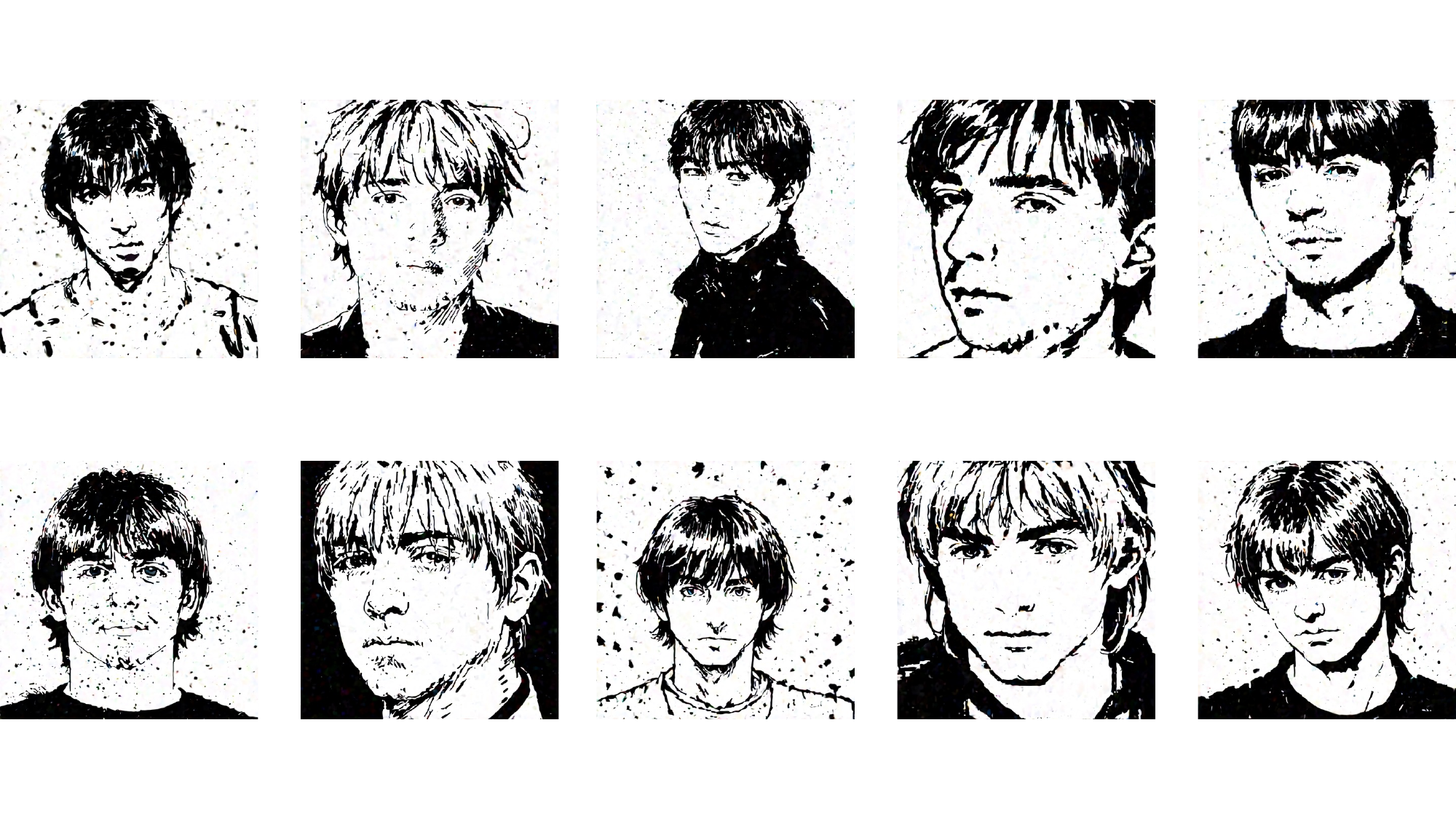}} 
        \caption{Generated Images on prompt "quick doodle of a guy, medium hair with long bangs, hd detailed detailed" with guidance $c=8$}
        \label{fig:Man}
\end{figure}

\clearpage
\subsubsection{Visualization of the Generated Images with different $r$  (number of steps) }

We further provide the visualization of the generated images with different $r$  ( and number of steps).  Four simple prompts, "bear", "wolf",  "hippo", and "lion" , are employed.  For all the methods, we fix  $seed = 0$. The generated images are shown in Figure~\ref{bear_gc8} to Figure~\ref{lion_gc5}.  We can observe that as the $r$ increases, VA-SALD changes gradually significantly in a stable manner for both small guidance $c=5$ and large guidance $c=8$.  In contrast, baseline: FM-ZG and FM-Evolv tend to degenearte at large guidance $c=8$. 

\begin{figure}[t] 
    \centering
    \includegraphics[width=1\textwidth]{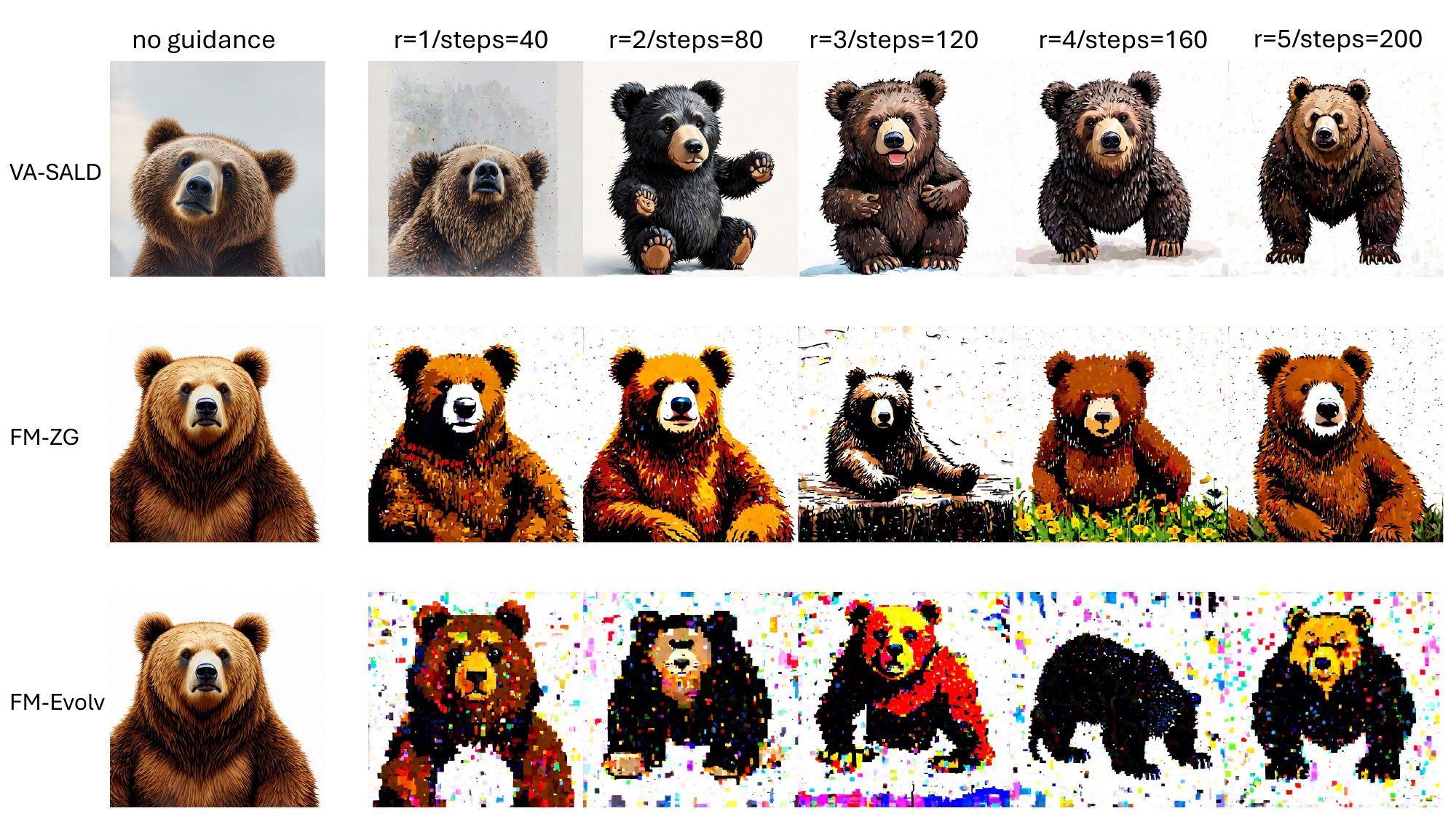} 
    \caption{Guided generation on "bear"   with guidance parameter $c=8$ }
    \label{bear_gc8}
\end{figure}

\begin{figure}[t] 
    \centering
    \includegraphics[width=1\textwidth]{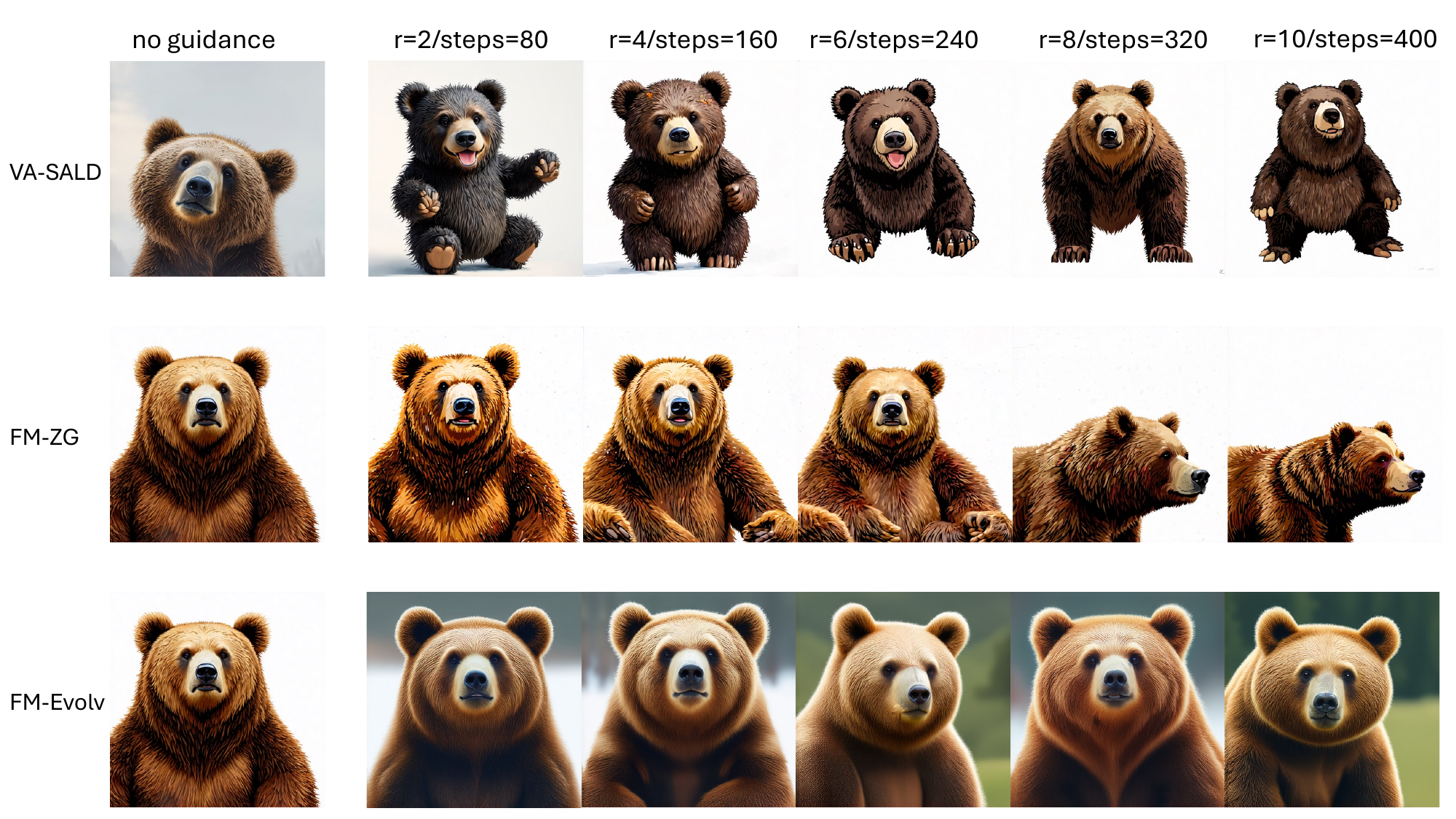} 
    \caption{Guided generation on "bear"   with guidance parameter $c=5$ }
    \label{bear_gc5}
\end{figure}

 \begin{figure}[h] 
    \centering
    \includegraphics[width=1\textwidth]{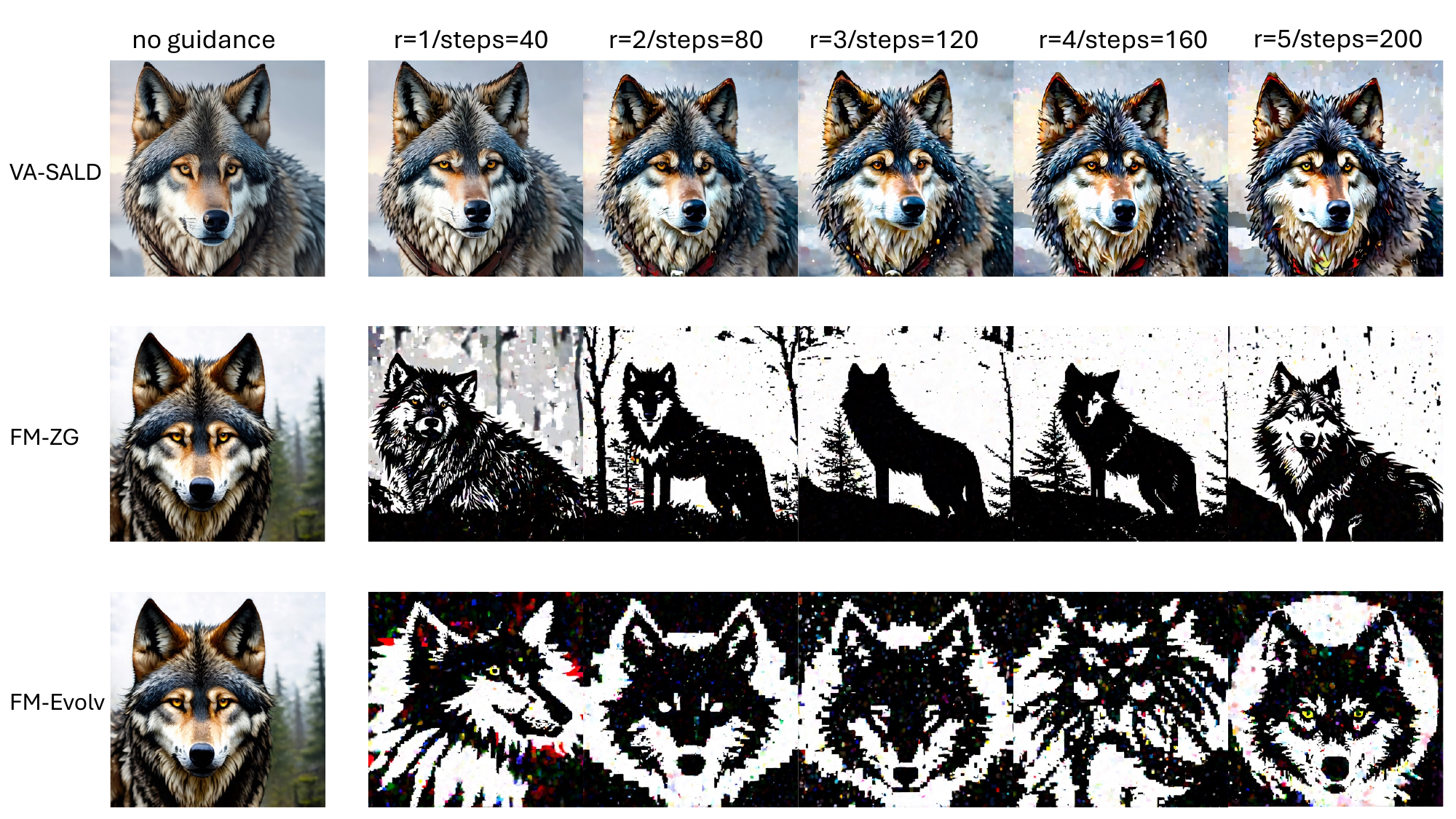} 
    \caption{Guided generation on "wolf"   with guidance parameter $c=8$ }
    \label{wolf_gc8}
\end{figure}

\begin{figure}[h] 
    \centering
    \includegraphics[width=1\textwidth]{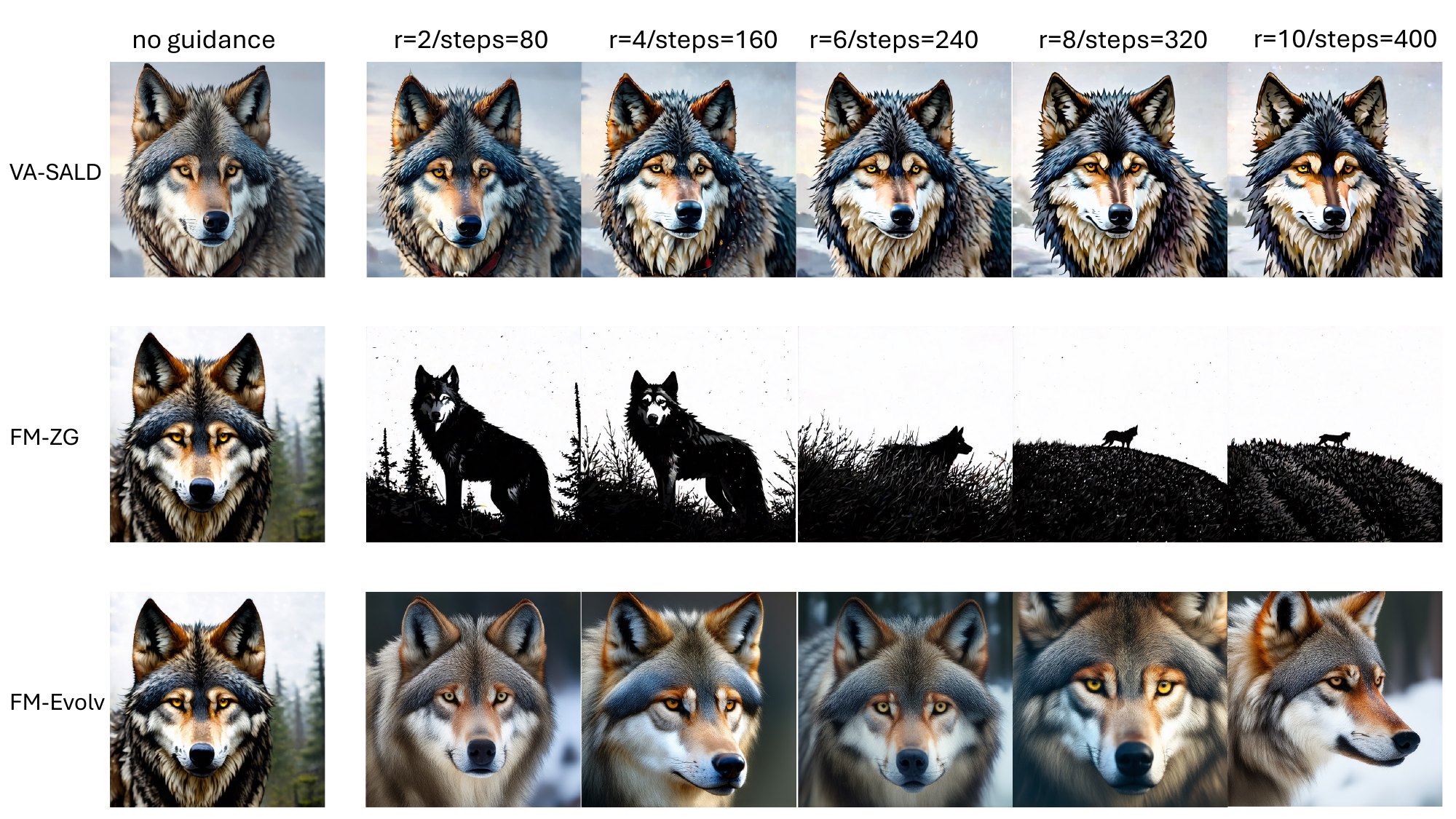} 
    \caption{Guided generation on "wolf"   with guidance parameter $c=5$ }
    \label{wolf_gc5}
\end{figure}

\begin{figure}[h] 
    \centering
    \includegraphics[width=1\textwidth]{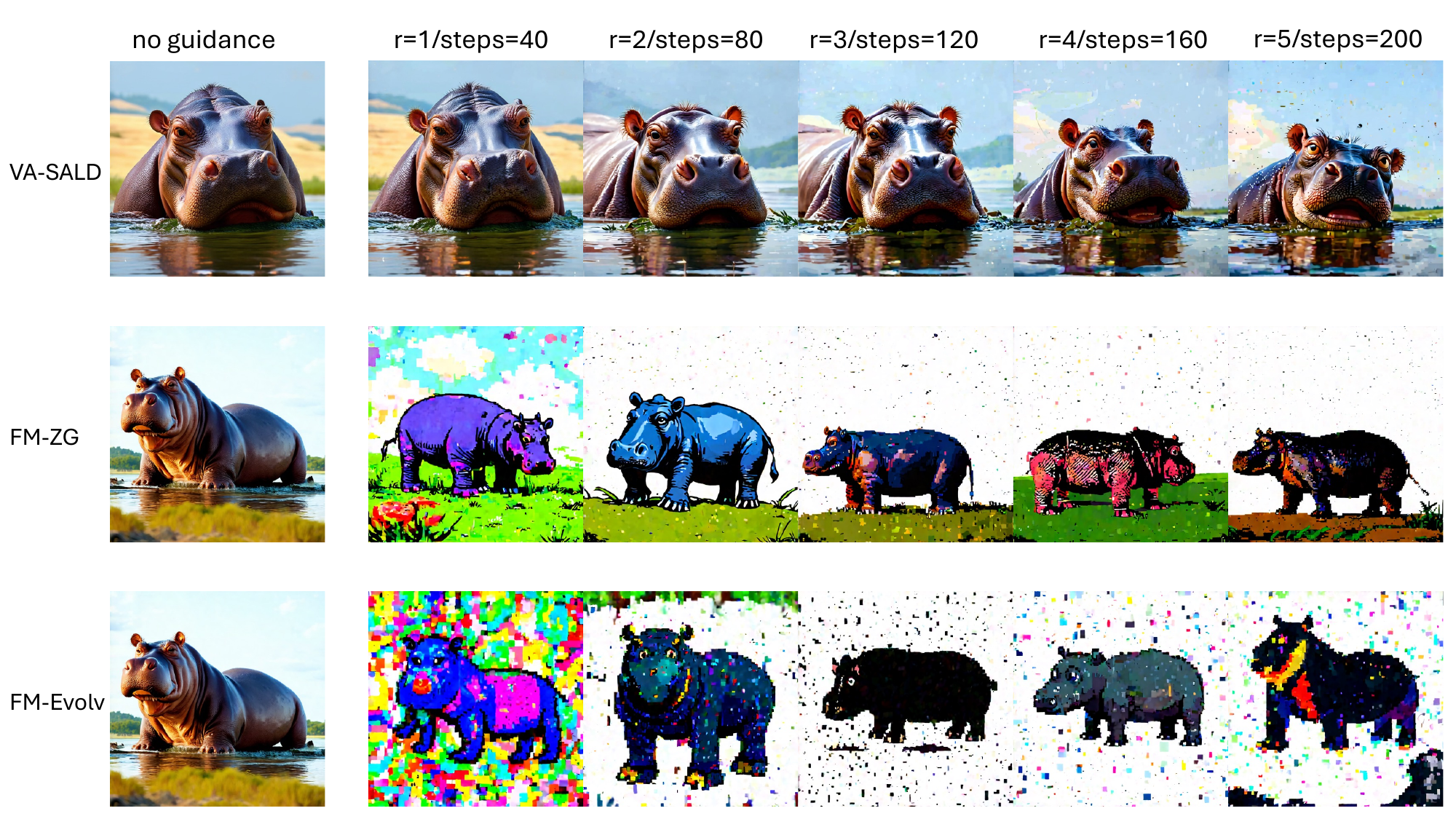} 
    \caption{Guided generation on "hippo"   with guidance parameter $c=8$ }
    \label{hippo_gc8}
\end{figure}

\begin{figure}[h] 
    \centering
    \includegraphics[width=1\textwidth]{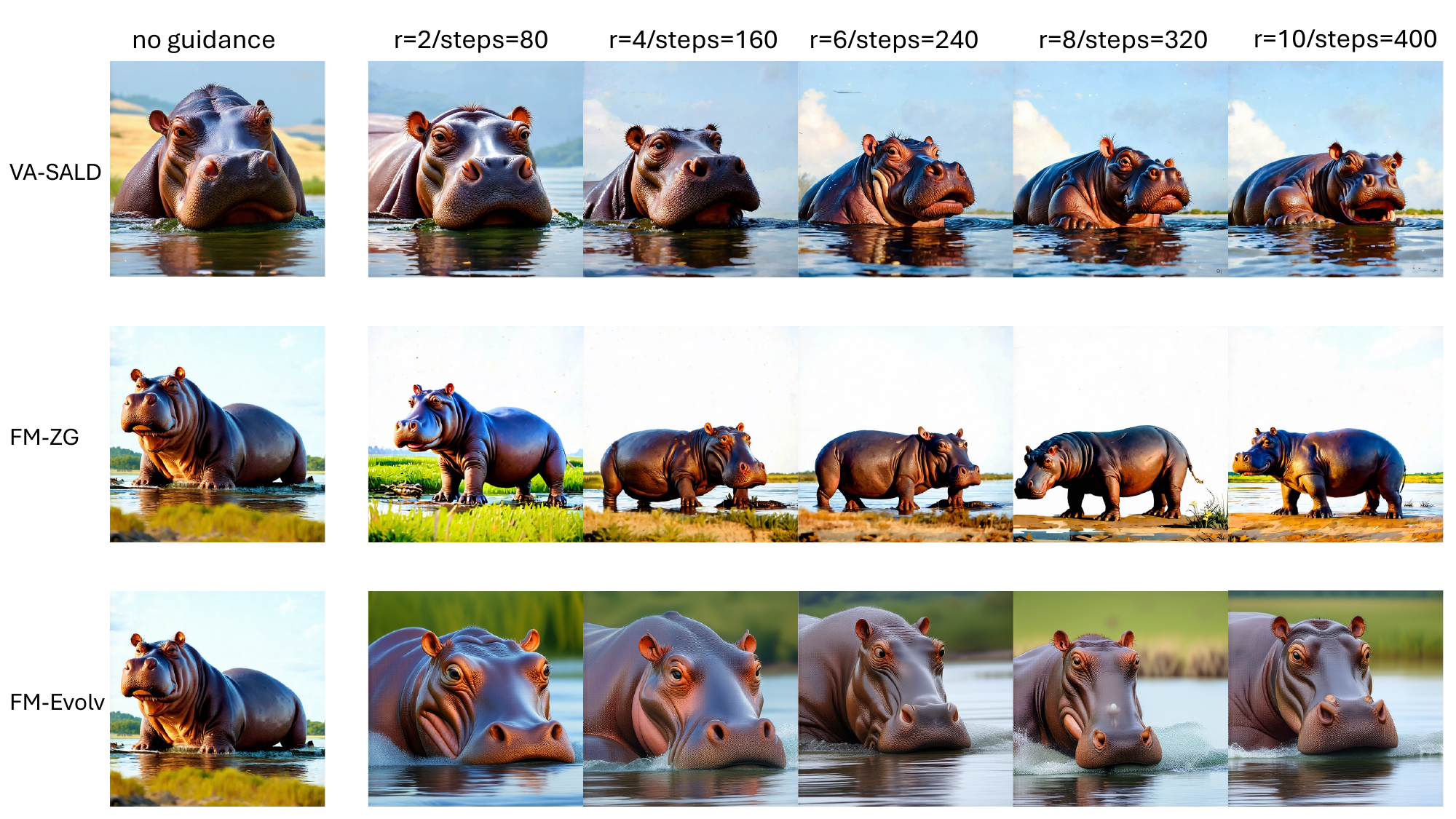} 
    \caption{Guided generation on "hippo"   with guidance parameter $c=5$ }
    \label{hippo_gc5}
\end{figure}

\begin{figure}[h] 
    \centering
    \includegraphics[width=1\textwidth]{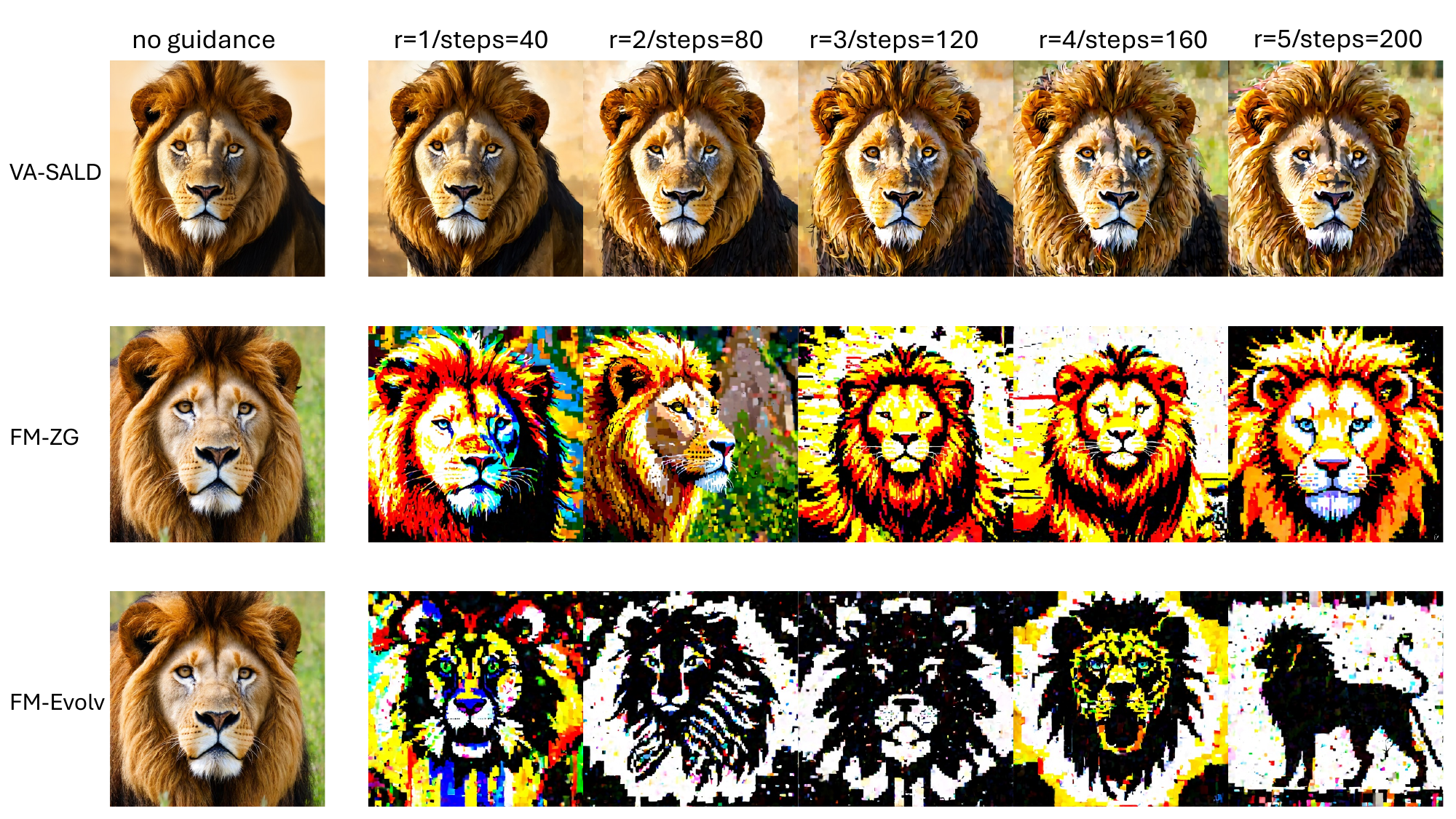} 
    \caption{Guided generation on "lion"   with guidance parameter $c=8$ }
    \label{lion_gc8_2}
\end{figure}

\begin{figure}[h] 
    \centering
    \includegraphics[width=1\textwidth]{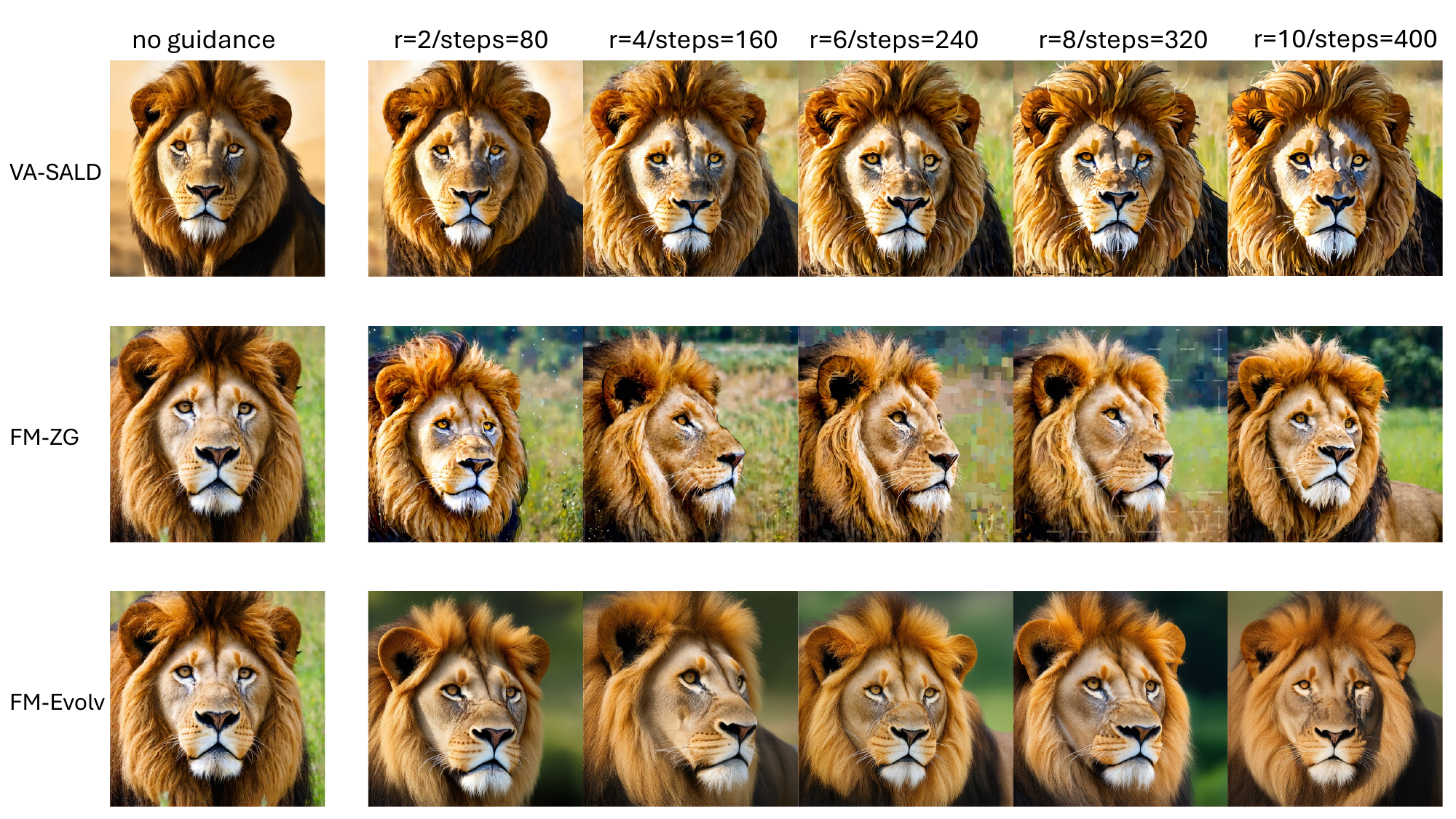} 
    \caption{Guided generation on "lion"   with guidance parameter $c=5$ }
    \label{lion_gc5}
\end{figure}

\clearpage

\clearpage

\end{document}